\documentclass{article}

     \PassOptionsToPackage{numbers,compress}{natbib}

     \usepackage[final]{neurips_2020}

\usepackage[utf8]{inputenc} 
\usepackage[T1]{fontenc}    
\usepackage{hyperref}       
\usepackage{url}            
\usepackage{booktabs}       
\usepackage{amsfonts}       
\usepackage{nicefrac}       
\usepackage{microtype}      
\usepackage{amsthm}
\usepackage{amssymb}
\usepackage{graphicx}
\usepackage{hhline}
\usepackage{xcolor}
\usepackage{multirow}
\usepackage{subcaption}
\newcommand{\comment}[1]{}
\usepackage{floatrow}
\usepackage{wrapfig,lipsum,booktabs}

\newfloatcommand{capbtabbox}{table}[][\FBwidth]
\newcommand{\cellcenter}[1]{\multicolumn{1}{|c|}{#1}}
\usepackage[framemethod=tikz]{mdframed}
\usepackage{cancel}

\definecolor{boxcolor}{rgb}{0.122, 0.435, 0.698}

\newmdenv[innerlinewidth=0.5pt, roundcorner=4pt,linecolor=boxcolor,innerleftmargin=6pt,
innerrightmargin=6pt,innertopmargin=6pt,innerbottommargin=6pt]{mybox}

\newtheorem{theorem}{Theorem}
\newtheorem{observation}{Observation}
\newtheorem{definition}{Definition}
\usepackage{commath}

\usepackage{amsmath}
\usepackage{ulem}

\usepackage[utf8]{inputenc} 
\usepackage[T1]{fontenc}    
\usepackage{hyperref}       
\usepackage{url}            
\usepackage{booktabs}       
\usepackage{amsfonts}       
\usepackage{nicefrac}       
\usepackage{microtype}      
\usepackage{amsthm}
\usepackage{graphicx}
\usepackage{hhline}
\usepackage{xcolor}
\usepackage{multirow}
\usepackage{algpseudocode}
\usepackage{algorithm}
\usepackage{subcaption}
\usepackage{graphicx}
\usepackage{adjustbox}
\usepackage{rotating}
\usepackage{hhline}
\newtheorem{lemma}{Lemma}
\usepackage{commath}
\algnewcommand\algorithmicforeach{\textbf{for each}}
\algdef{S}[FOR]{ForEach}[1]{\algorithmicforeach\ #1\ \algorithmicdo}
\usepackage[toc,page]{appendix}
\usepackage{amsmath}
\newtheorem{innercustomthm}{Theorem}
\newenvironment{customthm}[1]
  {\renewcommand\theinnercustomthm{#1}\innercustomthm}
  {\endinnercustomthm}

\title{A Ranking-based, Balanced Loss Function Unifying Classification and Localisation in Object Detection 
}

\author{%
  Kemal Oksuz, Baris Can Cam, Emre Akbas$^*$, Sinan Kalkan\thanks{Equal contribution for senior authorship.}\\
  Dept. of Computer Engineering, Middle East Technical University\\
  Ankara, Turkey \\
  \texttt{\{kemal.oksuz, can.cam, eakbas, skalkan\}@metu.edu.tr} 
}

\begin{document}

\maketitle

\begin{abstract}
We propose \textit{average Localisation-Recall-Precision} (aLRP), a unified, bounded, balanced and ranking-based loss function for both classification and localisation tasks in object detection. aLRP extends the Localisation-Recall-Precision (LRP) performance metric (Oksuz et al., 2018) inspired from how Average Precision (AP) Loss extends precision to a ranking-based loss function for classification (Chen et al., 2020). aLRP has the following distinct advantages: (i) aLRP is the first ranking-based loss function for both classification and localisation tasks. (ii) Thanks to using ranking for both tasks, aLRP naturally enforces high-quality localisation for high-precision classification. (iii) aLRP provides provable balance between positives and negatives. (iv) Compared to on average $\sim$6 hyperparameters in the loss functions of state-of-the-art detectors, aLRP Loss has only one hyperparameter, which we did not tune in practice. On the COCO dataset, aLRP Loss improves its ranking-based predecessor, AP Loss, up to around $5$ AP points, achieves $48.9$ AP without test time augmentation and outperforms all one-stage detectors. Code available at: \url{https://github.com/kemaloksuz/aLRPLoss}.






\end{abstract}

\section{Introduction}
\label{sec:intro}
Object detection requires jointly optimizing a classification objective ($\mathcal{L}_c$) and a localisation objective ($\mathcal{L}_r$) combined conventionally with a balancing hyperparameter ($w_r$) as follows: 
\begin{align}
    \mathcal{L} = \mathcal{L}_c + w_r \mathcal{L}_r.
    \label{eq:classical_loss}
\end{align}
Optimizing $\mathcal{L}$ in this manner has three critical drawbacks: (D1) It does not correlate the two tasks, and hence, does not guarantee high-quality localisation for high-precision examples (Fig. \ref{fig:Teaser}). (D2) It requires a careful tuning of $w_r$ \cite{CenterNet,GIoULoss,FreeAnchor}, which is prohibitive since a single training may last on the order of days, and ends up with a sub-optimal constant $w_r$ \cite{LapNet,WrLearningviaUncertainty}. (D3) It is adversely impeded by the positive-negative imbalance in $\mathcal{L}_c$ and inlier-outlier imbalance in $\mathcal{L}_r$, thus it requires sampling strategies \cite{gradientharmonizing, FocalLoss} or specialized loss functions \cite{FastRCNN, LibraRCNN}, introducing more hyperparameters (Table \ref{tab:Hyperparameters}). 


A recent solution for D3 is to directly maximize Average Precision (AP) with a loss function called AP Loss \cite{APLoss}. AP Loss is a ranking-based loss function to optimize the ranking of the classification outputs and provides balanced training between positives and negatives.

In this paper, we extend AP Loss to address all three drawbacks (D1-D3) with one, unified loss function called average Localisation Recall Precision (aLRP) Loss. In analogy with the link between precision and AP Loss, we formulate aLRP Loss as the average of LRP values \cite{LRP} over the positive examples on the Recall-Precision (RP) curve. aLRP has the following benefits: (i) It exploits ranking for both classification and localisation, enforcing high-precision detections to have high-quality localisation (Fig. \ref{fig:Teaser}). (ii) aLRP has a single hyperparameter (which we did not need to tune) as opposed to $\sim$6 in state-of-the-art loss functions (Table \ref{tab:Hyperparameters}). (iii) The network is trained by a single loss function that provides provable balance between positives and negatives.

Our contributions are: 
\textbf{(1)} We develop a generalized framework to  optimize non-differentiable ranking-based functions by extending the error-driven optimization of AP Loss.
\textbf{(2)} We prove that ranking-based loss functions conforming to this generalized form provide a natural balance between positive and negative samples.
\textbf{(3)} We introduce aLRP Loss (and its gradients) as a special case of this generalized formulation. Replacing AP and SmoothL1 losses by aLRP Loss for training RetinaNet improves the performance by up to $5.4$AP, and our best model reaches $48.9$AP without test time augmentation, outperforming all existing one-stage detectors with significant margin.
    

\begin{figure}[t]
    \centerline{
        \includegraphics[width=0.98\textwidth]{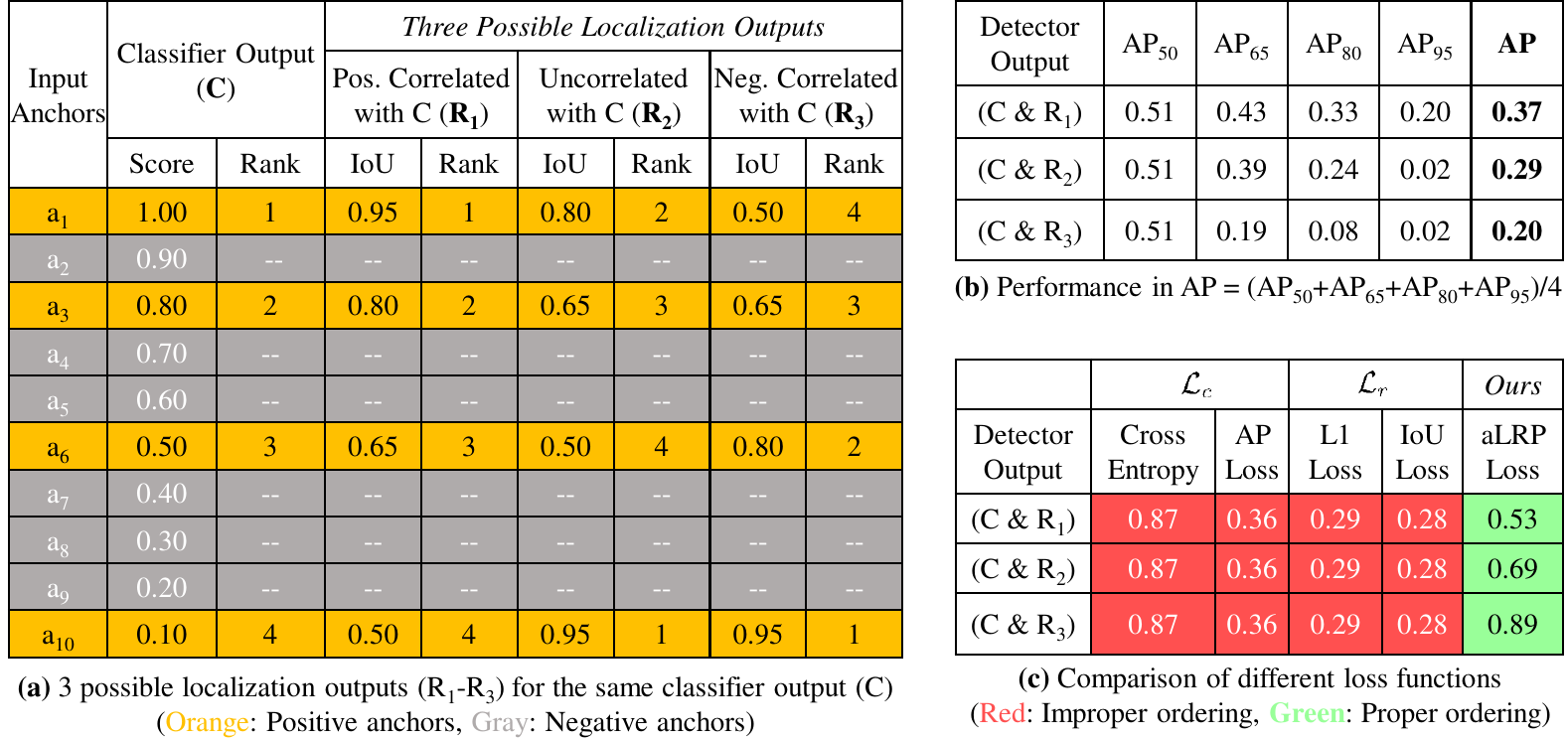}
    }
    \caption{\textbf{aLRP Loss enforces high-precision detections to have high-IoUs, while others do not.} \textbf{(a)} Classification and three possible localisation outputs for $10$ anchors and the rankings of the positive anchors with respect to (wrt) the scores (for $C$) and IoUs (for $R_1$, $R_2$ and $R_3$). Since the regressor is only trained by positive anchors, ``--'' is assigned for negative anchors. \textbf{(b,c)} Performance and loss assignment comparison of $R_1$, $R_2$ and $R_3$ when combined with $C$. When correlation between the rankings of classifier and regressor outputs decreases, performance degrades up to $17$ AP (b). While any combination of $\mathcal{L}_c$ and $\mathcal{L}_r$ cannot distinguish  them, aLRP Loss penalizes the outputs accordingly (c). The details of the calculations are presented in Appendix A. 
} 
\label{fig:Teaser}
\end{figure}

\begin{table}[t]
    \centering
    \caption{State-of-the-art loss functions have several hyperparameters ($6.4$ on avg.). aLRP Loss has only one for step-function approximation (Sec. \ref{subsec:APLoss}). See Appendix B for descriptions of the required hyperparameters. FL: Focal Loss, CE: Cross Entropy, SL1: Smooth L1, H: Hinge Loss. }
    \label{tab:Hyperparameters}
    \begin{tabular}{|l|l|c|}\hline
         \cellcenter{Method}& \cellcenter{$\mathcal{L}$} & Number of hyperparameters \\ \hline \hline
        AP Loss \cite{APLoss}&AP Loss+$\alpha$ SL1&3 \\ \hline
        Focal Loss \cite{FocalLoss}&FL+ $\alpha$ SL1&4\\ \hline
        FCOS \cite{FCOS}&FL+$\alpha$ IoU+$\beta$ CE&4\\  \hline
        DR Loss \cite{DRLoss}&DR Loss+$\alpha$ SL1&5\\ \hline
        FreeAnchor \cite{FreeAnchor}& $\alpha \log(\max( e^\text{CE} \times e^{\beta \text{SL1}}))$+$\gamma$ FL&8\\\hline
        Faster R-CNN \cite{FasterRCNN}&CE+$\alpha$ SL1+$\beta$CE+$\gamma$ SL1&9\\ \hline
        Center Net \cite{CenterNet}&FL+FL+$\alpha$ L2+$\beta$ H+$\gamma$ (SL1+SL1)&10\\ \hline \hline
        Ours&aLRP Loss&1 
        \\ \hline
    \end{tabular}\\
\end{table}

\subsection{Related Work}
\label{sec:RelatedWork}
\textbf{Balancing $\mathcal{L}_c$ and $\mathcal{L}_r$} in Eq. \eqref{eq:classical_loss}, an open problem in object detection (OD) \cite{Review}, bears important challenges: Disposing $w_r$, and correlating $\mathcal{L}_c$ and $\mathcal{L}_r$. \textit{Classification-aware regression loss} \cite{PrimeSample} links the branches by weighing $\mathcal{L}_r$ of an anchor using its classification score. Following Kendall et al. \cite{WrLearningviaUncertainty}, \textit{LapNet} \cite{LapNet} tackled the challenge by making $w_r$ a learnable parameter based on homoscedastic uncertainty of the tasks. Other approaches \cite{IoUNet,BoundedIoU} combine the outputs of two branches during non-maximum suppression (NMS) at inference. Unlike these methods, aLRP Loss considers the ranking wrt scores for both branches and  addresses the imbalance problem naturally.


\textbf{Ranking-based objectives in OD:} An inspiring solution for balancing classes is to optimize a ranking-based objective. However, such objectives are discrete wrt the scores, rendering their direct incorporation challenging. A solution is to use black-box solvers for an interpolated AP loss surface \cite{BlackboxCombinatorialSolvers}, which, however, provided only little gain in performance. AP Loss \cite{APLoss} takes a different approach by using an error-driven update mechanism to calculate gradients (Sec. \ref{eq:APLoss}). An alternative, DR Loss \cite{DRLoss}, employs Hinge Loss to enforce a margin between the scores of the positives and negatives. Despite promising results, these methods are limited to classification and leave localisation as it is. In contrast, we propose a single, balanced, ranking-based loss  to train both branches.


\section{Background}
\label{subsec:background}


\subsection{AP Loss and Error-Driven Optimization}
\label{subsec:APLoss}

AP Loss \cite{APLoss} directly optimizes the following loss for AP with intersection-over-union (IoU) thresholded at 0.50:
\begin{align}
    \label{eq:APLoss}
    \mathcal{L}^{\mathrm{AP}} = 1- \mathrm{\mathrm{AP}}_{50} = 1 - \frac{1}{|\mathcal{P}|} \sum \limits_{i \in \mathcal{P}} \mathrm{precision}(i) 
    = 
    1 - \frac{1}{|\mathcal{P}|} \sum \limits_{i \in \mathcal{P}} \frac{\mathrm{rank}^+(i)}{\mathrm{rank}(i)},
\end{align}
where $\mathcal{P}$ is the set of positives; $\mathrm{rank}^+(i)$ and $\mathrm{rank}(i)$ are respectively the ranking positions of the $i$th sample among positives and all samples. $\mathrm{rank}(i)$ can be easily defined using a step function $H(\cdot)$ applied on the  difference between the score of $i$ ($s_i$) and the score of each other sample:
\begin{align}
\mathrm{rank}(i) =  
1+\sum \limits_{j \in \mathcal{P}, j \neq i} H(x_{ij})+\sum \limits_{j \in  \mathcal{N}} H(x_{ij}),
\end{align}
where $x_{ij} = -(s_i-s_j)$ is positive if $s_i<s_j$; $\mathcal{N}$ is the set of negatives; and $H(x)=1$ if $x \geq 0$ and $H(x)=0$ otherwise. In practice, $H(\cdot)$ is replaced by $x / 2 \delta+0.5$ in the interval $[-\delta,\delta]$ (in aLRP, we use $\delta=1$ as set by AP Loss \cite{APLoss} empirically; this is the only hyperparameter of aLRP -- Table \ref{tab:Hyperparameters}). $\mathrm{rank}^+(i)$ can be defined similarly over $j\in\mathcal{P}$. With this notation, $\mathcal{L}^{\mathrm{AP}}$ can  be rewritten as follows:
\begin{align}
    \label{eq:APPrimaryTerms}
    \mathcal{L}^{\mathrm{AP}} = \frac{1}{|\mathcal{P}|} \sum \limits_{i \in \mathcal{P}} \sum \limits_{j \in \mathcal{N}}  \frac{H(x_{ij})} {\mathrm{rank}(i)}
    =
    \frac{1}{|\mathcal{P}|} \sum \limits_{i \in \mathcal{P}} \sum \limits_{j \in \mathcal{N}}  L_{ij}^{\mathrm{AP}} ,
\end{align}
where  $L_{ij}^{\mathrm{AP}}$ is called a \textit{primary term} which is zero if $i \notin \mathcal{P}$ or $j \notin \mathcal{N}$
\footnote{By setting $L_{ij}^{\mathrm{AP}}=0$ when $i \notin \mathcal{P}$ or $j \notin \mathcal{N}$, we do not require the $y_{ij}$ term used by Chen et al. \cite{APLoss}.}.

Note that this system is composed of two parts: (i) The differentiable part up to $x_{ij}$, and (ii) the non-differentiable part that follows $x_{ij}$. Chen et al. proposed that an error-driven update of $x_{ij}$ (inspired from perceptron learning \cite{Rosenblatt}) can be combined with derivatives of the differentiable part. Consider the update in $x_{ij}$ that minimizes $L^{\mathrm{AP}}_{ij}$ (and hence $\mathcal{L}^{\mathrm{AP}}$): $\Delta x_{ij}={L^{{\mathrm{AP}}*}_{ij}}-L^{\mathrm{AP}}_{ij}=0-L^{\mathrm{AP}}_{ij}=-L^{\mathrm{AP}}_{ij}$, with the target, $L^{{\mathrm{AP}}*}_{ij}$, being zero for perfect ranking. Chen et al. showed that the gradient of $L^{\mathrm{AP}}_{ij}$ wrt $x_{ij}$ can be taken as $- \Delta x_{ij}$. With this, the gradient of $\mathcal{L}^{\mathrm{AP}}$ wrt scores can be calculated as follows:
\begin{align}
    \label{eq:APGradients}
    \frac{\partial \mathcal{L}^{\mathrm{AP}}}{\partial s_i} 
    = \sum \limits_{j,k} \frac{\partial \mathcal{L}^{\mathrm{AP}}}{\partial x_{jk}} \frac{\partial x_{jk}}{\partial s_i} 
    = 
    -\frac{1}{|\mathcal{P}|} \sum \limits_{j,k} \Delta x_{jk} \frac{\partial x_{jk}}{\partial s_i} 
    = \frac{1}{|\mathcal{P}|} \left( \sum \limits_{j} \Delta x_{ij} - \sum \limits_{j} \Delta x_{ji} \right). 
\end{align}

\subsection{Localisation-Recall-Precision (LRP) Performance Metric}
\label{subsec:LRP}
LRP \cite{LRP,LRParXiv} is a metric that  quantifies  classification and localisation performances jointly. Given a detection set thresholded at a score ($s$) and their matchings with the ground truths,  LRP  aims to assign an error value within $[0,1]$ by considering localisation, recall and precision:
\begin{align}
\label{eq:LRPdefcompact}
\mathrm{LRP}(s) = \frac{1}{N_{FP} +N_{FN}+{N_{TP}}}\left( N_{FP} +N_{FN} + \sum \limits_{k \in {TP}} \mathcal{E}_{loc}(k) \right),
\end{align}
where $N_{FP}, N_{FN}$ and $N_{TP}$ are the number of false positives (FP), false negatives (FN) and true positives (TP);  A detection is a TP if $\mathrm{IoU}(k) \geq \tau$ where $\tau=0.50$ is the conventional TP labeling threshold, and a TP has a localisation error of $\mathcal{E}_{loc}(k) = (1-\mathrm{IoU}(k))/(1-\tau)$. The detection performance is, then,  $\min \limits_s (\mathrm{LRP}(s))$ on the precision-recall (PR) curve, called optimal LRP (oLRP).

\section{A Generalisation of Error-Driven Optimization for Ranking-Based Losses}
\label{sec:generalization}

Generalizing the error-driven optimization technique of AP Loss \cite{APLoss} to other ranking-based loss functions is not trivial. In particular, identifying the primary terms is a challenge especially when the loss has components that involve only  positive examples, such as the localisation error in aLRP Loss. 

Given a ranking-based loss function, $\mathcal{L}=\frac{1}{Z}\sum_{i \in \mathcal{P}} \ell(i)$, defined as a sum over individual losses, $\ell(i)$, at positive examples (e.g., Eq. \eqref{eq:APLoss}), with $Z$ as a problem specific normalization constant, our goal is to express $\mathcal{L}$ as a sum of \textit{primary terms} in a more general form than Eq. \eqref{eq:APPrimaryTerms}: 

\begin{definition} The \textbf{primary term} $L_{ij}$ concerning examples $i \in \mathcal{P}$ and $j \in \mathcal{N}$  is the loss originating from $i$ and distributed over $j$ via a probability mass function $p (j | i)$. Formally, 
\begin{align}
    \label{eq:GeneralPrimaryTermDefinition}
    L_{ij} = \begin{cases} \ell (i) p(j|i), & \mathrm{for}\;i \in \mathcal{P}, j \in \mathcal{N} \\
    0, & \mathrm{otherwise}.
    \end{cases}
\end{align}
\end{definition}
Then, as desired, we can express $\mathcal{L}=\frac{1}{Z}\sum_{i \in \mathcal{P}} \ell(i)$ in terms of $L_{ij}$:
\begin{theorem}
\label{theorem:PrimaryTerms}
$\mathcal{L}= \frac{1}{Z}\sum \limits_{i \in \mathcal{P} } \ell(i) =
\frac{1}{Z}\sum \limits_{i \in \mathcal{P} }\sum \limits_{j \in \mathcal{N} }  L_{ij}$. See Appendix C for the proof.
\end{theorem}{}
%
Eq. \eqref{eq:GeneralPrimaryTermDefinition} makes it easier to define primary terms and adds more flexibility on the error distribution: e.g., AP Loss takes $p(j|i) = H(x_{ij})/N_{FP}(i)$, which distributes error uniformly (since it is reduced to $1/N_{FP}(i)$) over $j \in \mathcal{N}$ with $s_j \geq s_i$; though, a skewed $p(j|i)$  can be used to promote harder examples (i.e. larger $x_{ij}$). Here, $N_{FP}(i) = \sum_{j \in \mathcal{N}} H(x_{ij})$ is the number of false positives for $i \in \mathcal{P}$.

Now we can identify the gradients of this generalized definition following Chen et al. (Sec. \ref{subsec:APLoss}): The error-driven update in $x_{ij}$ that would minimize $\mathcal{L}$ is $\Delta x_{ij} = {L_{ij}}^* - L_{ij}$, where ${L_{ij}}^*$ denotes ``the primary term when $i$ is ranked properly''. Note that ${L_{ij}}^*$, which is set to zero in AP Loss, needs to be carefully defined (see Appendix G for a bad example). With $\Delta x_{ij}$ defined, the gradients can be derived  similar to Eq. \eqref{eq:APGradients}. The steps for obtaining the gradients of $\mathcal{L}$ are summarized in Algorithm \ref{alg:Errordriven}.

\begin{algorithm}
\caption{Obtaining the gradients of a ranking-based function with error-driven update. \label{alg:Errordriven}}
\begin{flushleft}
\textbf{Input:} A ranking-based function $\mathcal{L} = (\ell(i), Z)$, and a probability mass function $p(j|i)$ \\ 
 \textbf{Output:} The gradient of $\mathcal{L}$ with respect to model output $\mathbf{s}$
 \end{flushleft}
\begin{algorithmic}[1]
\State $\forall i,j$ find primary term: $L_{ij} = \ell (i) p(j|i)$ if $i \in \mathcal{P}, j \in \mathcal{N}$; otherwise $L_{ij}=0$ (c.f. Eq. \eqref{eq:GeneralPrimaryTermDefinition}).
\State $\forall i,j$ find target primary term: ${L_{ij}}^*= \ell(i)^* p(j|i)$ ($\ell(i)^*$: the error on $i$ when $i$ is ranked properly.)
\State $\forall i,j$ find error-driven update: $\Delta x_{ij} = {L_{ij}}^* - L_{ij}=\left(\ell(i)^* - \ell(i)\right) p(j|i)$.
\State  \textbf{return} $\frac{1}{Z} ( \sum \limits_{j} \Delta x_{ij} - \sum \limits_{j} \Delta x_{ji} )$ for each $s_i \in \mathbf{s}$ (c.f. Eq. \eqref{eq:APGradients}).
\end{algorithmic}
\end{algorithm}

This optimization provides balanced training for  ranking-based losses conforming to Theorem \ref{theorem:PrimaryTerms}:
\begin{theorem} 
\label{theorem:BalancedTraining}
Training is balanced between positive and negative examples at each iteration; i.e. the summed gradient magnitudes of positives and negatives are equal (see Appendix C for the proof):
\begin{align}
\label{eq:Theorem1}
\sum \limits_{i \in \mathcal{P}} \abs{\frac{\partial \mathcal{L}}{\partial s_i}} = \sum \limits_{i \in \mathcal{N}} \abs{\frac{\partial \mathcal{L}}{\partial s_i}}.
\end{align}
\end{theorem}
\textcolor{blue}{\rule{\textwidth}{0.1mm}}\\
\textbf{Deriving AP Loss.} Let us derive AP Loss as a case example for this generalized framework: $\ell^{\mathrm{AP}}(i)$ is simply $1-\mathrm{precision(i)}=N_{FP}(i) /\mathrm{rank}(i)$, and $Z=|\mathcal{P}|$. $p(j|i)$ is assumed to be uniform, i.e. $p(j|i)=H(x_{ij})/N_{FP}(i)$. These give us $L_{ij}^{\mathrm{AP}}=\frac{N_{FP}(i) }{\mathrm{rank}(i)} \frac{H(x_{ij})}{N_{FP}(i)}= \frac{H(x_{ij})}{\mathrm{rank}(i)}$ (c.f. $L_{ij}^{\mathrm{AP}}$ in Eq. \eqref{eq:APPrimaryTerms}). Then, since ${L_{ij}^{\mathrm{AP}}}^*=0$, $\Delta x_{ij}=0-L^{\mathrm{AP}}_{ij}=-L^{\mathrm{AP}}_{ij}$ in Eq. \eqref{eq:APGradients}.

\textbf{Deriving Normalized Discounted Cumulative Gain Loss \cite{OptimizingUpperBound}}: See Appendix D.\\
\textcolor{blue}{\rule{\textwidth}{0.1mm}}


\comment{
\section{[OLD] Error-Driven Optimization of the Ranking-Based Functions}
\label{sec:generalization}
While the formulation in Section \ref{subsec:APLoss} reaches state-of-the-art results with a ranking-based loss for OD problem; it is not easily applicable to a larger set of ranking-based functions. After defining the set of targeted ranking-based functions, this section first introduces a formal definition for primary terms and then derives a general error-driven update rule to optimize them.

Given a set of examples, $\mathcal{D}$, the ranking task aims to promote a subset $\mathcal{A} \subset \mathcal{D}$, while demoting the remaining examples, $\mathcal{B} = \mathcal{D} - \mathcal{A}$. Here, we consider the ranking functions, $ \mathcal{R}(\mathcal{D})$, that can be decomposed on either $i \in \mathcal{A}$ or $i \in \mathcal{B}$.
that is,
    $\mathcal{R}(\mathcal{D}) = \frac{1}{Z} \sum \limits_{i \in \mathcal{C}}  \mathcal{E}^\mathcal{R} (i)$.
Here, $\mathcal{E}^\mathcal{R} (i)$ is the local error of the $i \in \mathcal{C}$ with $\mathcal{C} = \mathcal{A}$ or $\mathcal{C}= \mathcal{B}$ and $Z$ is a normalization constant. WLOG setting $\mathcal{C} = \mathcal{A}$, we follow a two-step framework to generalize error-driven update rule for $\mathcal{R}(\mathcal{D})$:


\textbf{Step 1: Identifying Primary Terms:} Here, we provide a more intuitive and general definition of the primary terms of $\mathcal{R}(\mathcal{D})$:

\begin{definition}
For $\mathcal{R}(\mathcal{D})$, the primary term between $i \in \mathcal{A}$ and $j \in \mathcal{B}$ (else the primary term is set to $0$), $\Psi^\mathcal{R}_{ij}$, is the reciprocal-error originating from $i$ and distributed over $j$ via the probability density function $p^ \mathcal{R} (i, j)$ (i.e. $\forall i \sum \limits_{j \in \mathcal{B}} p^\mathcal{R}(i, j) = 1$). Formally, 
\begin{align}
    \label{eq:GeneralPrimaryTermDefinition}
    \Psi^\mathcal{R}_{ij} = \mathcal{E}^\mathcal{R} (i) \times p^\mathcal{R} (i, j).
\end{align}
\end{definition}
Then, as desired, we can compute  the value of $\mathcal{R}(\mathcal{D})$ by summing over this reciprocal errors:
\begin{theorem}
\label{theorem:PrimaryTerms}
$R(\mathcal{D})=  
\frac{1}{Z} \sum \limits_{i \in \mathcal{A} }\sum \limits_{j \in \mathcal{B} }  \Psi^ \mathcal{R}_{ij}$.
\end{theorem}{}
\begin{proof}
\begin{align}
    \mathcal{R}(\mathcal{D}) =\frac{1}{Z} \sum \limits_{i \in \mathcal{A} }  \mathcal{E}^\mathcal{R} (i)  
    =\frac{1}{Z} \sum \limits_{i \in \mathcal{A}}  \mathcal{E}^\mathcal{R} (i) \left( \sum \limits_{j \in  \mathcal{B} }p^\mathcal{R} (i, j)    \right)
    = \frac{1}{Z} \sum \limits_{i \in \mathcal{A} }\sum \limits_{j \in \mathcal{B} }  \Psi^ \mathcal{R}_{ij}
\end{align}
\end{proof}

Defining primary terms in this fashion provides following advantages over the top-down approach: (1) $\mathcal{E}^\mathcal{R} (i)$ is trivial (e.g. $\mathcal{E}^{AP} (i) = (1-\mathrm{precision}(i))$). (2) Determining a $p^\mathcal{R}(i, j)$ becomes more intuitive and adds more flexibility to the error distribution. To illustrate, while $p^{\mathcal{R}}(i, j) = \frac{H(x_{ij})}{\sum \limits_{k \in {\mathcal{N}}}H(x_{ik})}$ distributes the error uniformly over $j \in \mathcal{B}$ with $s_j \geq s_i$, a more left-skewed one will promote harder examples (i.e. larger $x_{ij}$) more. 

\textbf{Step 2: Identifying Gradients:} Here, we define the ``error'' of a difference transformation to use it for the update following perceptron learning algorithm. In particular, $x_{ij}$ is updated by its deviation from ``the primary term when $i$ is ranked properly'' denoted by ${\Psi^\mathcal{R}_{ij}}^*$ , as follows:
%
\begin{align}
    \Delta x_{ij} = {\Psi^R_{ij}}^* - \Psi^R_{ij}
\end{align}
Then, the gradients of the $R$ wrt the input of the difference transform is determined as: 
\begin{align}
    \label{eq:GeneralizedGradients}
    \frac{\partial R(\mathcal{D})}{\partial s_i} 
    = -\sum \limits_{j,k} \Delta x_{jk} \frac{\partial x_{jk}}{\partial s_i} 
    = \sum \limits_{j} \Delta x_{ij} - \sum \limits_{j} \Delta x_{ji}  
    = \sum \limits_{j \in \mathcal{A}} \left( {\Psi^R_{ji} - \Psi^R_{ji}}^*  \right) - \sum \limits_{j \in \mathcal{B}}  \left( { \Psi^R_{ij} - \Psi^R_{ij}}^*\right)
\end{align}

\begin{theorem} 
\label{theorem:BalancedTraining}
The magnitudes of the total gradients of the examples to be promoted (i.e. $\mathcal{A}$) and demoted (i.e. $\mathcal{B}$) are equal at every training iteration, which implies balanced training. Formally,
\begin{align}
\label{eq:Theorem1}
\abs{\sum \limits_{i \in \mathcal{A}} \frac{\partial \mathcal{R}(\mathcal{D})}{\partial s_i}} = \abs{\sum \limits_{i \in \mathcal{B}} \frac{\partial \mathcal{R}(\mathcal{D})}{\partial s_i}}.    
\end{align}
\end{theorem}
\begin{proof}
$\frac{\partial R(\mathcal{D})}{\partial s_i}=- \sum \limits_{j \in \mathcal{B}}  \left( { \Psi^R_{ij} - \Psi^R_{ij}}^*\right)$ for $i \in \mathcal{A}$ since $\Psi^R_{ij}>0$ only if $j \in \mathcal{B}$. Similarly, $\frac{\partial R(\mathcal{D})}{\partial s_i}$ for $j \in \mathcal{B}$ is simply the summation of the individual contributions from each $i \in \mathcal{A}$ following $p(i, j)^\mathcal{R}$. Since $\forall i \sum \limits_{j \in \mathcal{B}} p^\mathcal{R}(i, j) = 1$, theorem holds. See the Supp.Mat. for the formal proof.
\end{proof}
\begin{observation}
\label{observation:GradMatrix}
Equations \ref{eq:GeneralizedGradients} and Theorem \ref{theorem:BalancedTraining} together allow us to model the gradients from a different perspective: Let us define a gradient matrix by $G ^ \mathcal{R}$ for $\mathcal{R}(\mathcal{D})$ such that $G_{ij}$ represents an element at $i$th row and $j$th column. Then, assuming $i \in \mathcal{A}$ and $j \in \mathcal{B}$ correspond to the rows and columns, one can define the elements of the gradient matrix for the ranking function $\mathcal{R}(\mathcal{D})$ as follows:
\begin{align}
    G_{ij} ^ \mathcal{R} =\Psi^\mathcal{R}_{ij} - {\Psi^\mathcal{R}_{ij}}^* 
\end{align}
Then simply, summing over the rows and columns yields the magnitude of the gradients for each $i \in \mathcal{A}$ and for each $j \in \mathcal{B}$ respectively. For exact gradients with their directions, it is sufficient to multiply summed rows by $-1$ for $i \in \mathcal{A}$. Hence, using this observation, we reduce the gradient formulation of a ranking-based function $\mathcal{R}$ to define $G_{ij} ^ \mathcal{R}$.
\end{observation}
}
\section{Average Localisation-Recall-Precision (aLRP) Loss}
\label{sec:aLRPLoss}

\begin{figure}[t]
    \centerline{
        \includegraphics[width=0.99\textwidth]{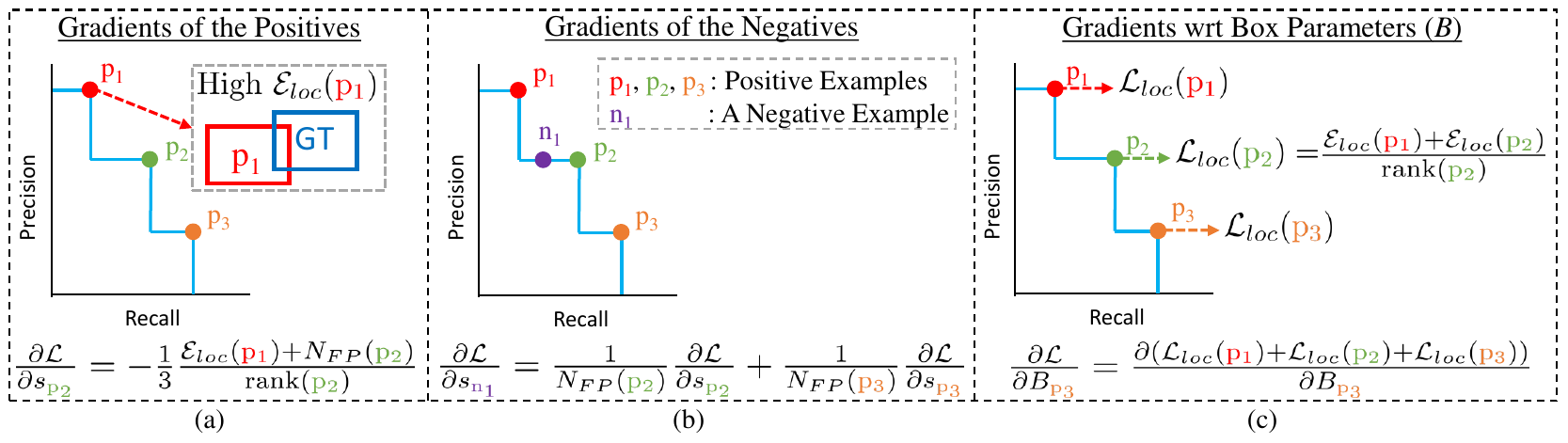}
    }
    \caption{\textbf{aLRP Loss assigns gradients to each branch based on the outputs of both branches.} Examples on the PR curve are in sorted order wrt scores ($s$). $\mathcal{L}$ refers to $\mathcal{L}^\mathrm{aLRP}$. \textbf{(a)} A $\mathrm{p}_i$'s gradient wrt its score considers (i) localisation errors of  examples with larger $s$ (e.g. high $\mathcal{E}_{loc}(\mathrm{p_1})$ increases the gradient of $s_{\mathrm{p_2}}$ to suppress  $\mathrm{p_1}$), (ii) number of negatives with larger $s$. \textbf{(b)} Gradients wrt $s$ of the negatives: The gradient of a $\mathrm{p}_i$ is uniformly distributed over the negatives with larger $s$. Summed  contributions from all positives determine the gradient of a negative. \textbf{(c)} Gradients of the box parameters: While $\mathrm{p_1}$ (with highest $s$) is included in total localisation error on each positive, i.e. $\mathcal{L}_{loc}(i)=\frac{1}{\mathrm{rank}(i)} (  \mathcal{E}_{loc}(i) +\sum \limits_{k \in \mathcal{P}, k \neq i}  \mathcal{E}_{loc}(k) H(x_{ik}))$,  $\mathrm{p_3}$ is included once with the largest $\mathrm{rank}(\mathrm{p}_i)$.
    \label{fig:ToyExample}
} 
\end{figure}

Similar to the relation between precision and AP Loss, aLRP Loss is defined as the average of LRP values ($\ell^{\mathrm{LRP}}(i)$) of positive examples:
\begin{align}
\label{eq:aLRPLoss}
    \mathcal{L}^\mathrm{aLRP}:=\frac{1}{|\mathcal{P}|}\sum \limits_{i \in \mathcal{P}} \ell^{\mathrm{LRP}}(i) 
    .
\end{align}
For LRP, we assume that anchors are dense enough to cover all ground-truths, i.e. $N_{FN}=0$. Also, since a detection is enforced to follow the label of its anchor during training, TP and FP sets are replaced by the thresholded subsets of $\mathcal{P}$ and $\mathcal{N}$,  respectively. This is applied by $H(\cdot)$, and $\mathrm{rank}(i)=N_{TP}+N_{FP}$ from Eq. \eqref{eq:LRPdefcompact}. Then, following the definitions in Sec. \ref{subsec:APLoss}, $\ell^{\mathrm{LRP}}(i)$ is:
\begin{align}
\label{eq:LRPReformulation}
    \ell^{\mathrm{LRP}}(i) 
    = \frac{1}{\mathrm{rank(i)}}
    \left(N_{FP}(i) + \mathcal{E}_{loc}(i) +  \sum \limits_{k \in \mathcal{P}, k \neq i}  \mathcal{E}_{loc}(k) H(x_{ik}) \right).
\end{align}
Note that Eq. \eqref{eq:LRPReformulation}  allows using robust forms of IoU-based  losses (e.g. generalized IoU (GIoU) \cite{GIoULoss}) only by replacing IoU Loss (i.e. $1- \mathrm{IoU}(i)$) in $\mathcal{E}_{loc}(i)$ and normalizing the range to $[0,1]$.

In order to provide more insight and facilitate gradient derivation, we split Eq. \eqref{eq:aLRPLoss} into two as localisation and classification components such that $\mathcal{L}^\mathrm{aLRP}=\mathcal{L}^\mathrm{aLRP}_{cls}+\mathcal{L}^\mathrm{aLRP}_{loc}$, where
\begin{align}
    \label{eq:aLRPComponents}
    \mathcal{L}^\mathrm{aLRP}_{cls} &= \frac{1}{|\mathcal{P}|}\sum \limits_{i \in \mathcal{P}} \frac{N_{FP}(i)}{\mathrm{rank}(i)} \text{, and } \mathcal{L}^\mathrm{aLRP}_{loc} = \frac{1}{|\mathcal{P}|}\sum \limits_{i \in \mathcal{P}} \frac{1}{\mathrm{rank}(i)} \left(  \mathcal{E}_{loc}(i) +\sum \limits_{k \in \mathcal{P}, k \neq i}  \mathcal{E}_{loc}(k) H(x_{ik}) \right).
\end{align}

\subsection{Optimization of the aLRP Loss}
$\mathcal{L}^\mathrm{aLRP}$ is differentiable wrt the estimated box parameters, $B$, since $\mathcal{E}_{loc}$ is differentiable \cite{GIoULoss,UnitBox} (i.e. the derivatives of $\mathcal{L}^\mathrm{aLRP}_{cls}$ and $\mathrm{rank}(\cdot)$ wrt $B$ are $0$). However, $\mathcal{L}^\mathrm{aLRP}_{cls}$ and $\mathcal{L}^\mathrm{aLRP}_{loc}$ are not differentiable wrt the classification scores, and therefore, we need the generalized framework from Sec. \ref{sec:generalization}. 

Using the same error distribution from AP Loss, the primary terms of aLRP Loss can be defined as $L^\mathrm{aLRP}_{ij} = \ell^\mathrm{LRP} (i) p(j|i)$. As for the target primary terms, we use the following desired LRP Error:

%
\begin{align}
    {\ell^\mathrm{LRP}(i)}^* = \frac{1}{\mathrm{rank(i)}}
    \left(\textcolor{red}{\cancelto{0}{N_{FP}(i)}} + \mathcal{E}_{loc}(i) +  \textcolor{red}{\cancelto{0}{\sum \limits_{k \in \mathcal{P}, k \neq i}  \mathcal{E}_{loc}(k) H(x_{ik})}}\right)=\frac{ \mathcal{E}_{loc}(i)}{\mathrm{rank}(i)},
\end{align}
yielding a target primary term, ${L^\mathrm{aLRP}_{ij}}^*={\ell^\mathrm{LRP}(i)}^* p(j|i)$, which includes localisation error and can be non-zero when $s_i<s_j$, unlike AP Loss. Then, the resulting error-driven update for $x_{ij}$ is (line 3 of Algorithm \ref{alg:Errordriven}):
\begin{align}
      \Delta x_{ij} = \left({\ell^\mathrm{LRP}(i)}^*- \ell^\mathrm{LRP}(i) \right) p(j|i) = -\frac{ 1}{\mathrm{rank}(i)} \left( N_{FP}(i) + \sum \limits_{k \in \mathcal{P}, k \neq i}  \mathcal{E}_{loc}(k) H(x_{ik}) \right)  \frac{H(x_{ij})}{N_{FP}(i)}.
\end{align}
Finally, ${\partial \mathcal{L}^\mathrm{aLRP}}/{\partial s_i}$ can be obtained with Eq. \eqref{eq:APGradients}. Our algorithm to compute the loss and gradients is presented in Appendix E in detail and has the same time\&space complexity with AP Loss. 

\begin{wrapfigure}{r}{0.49\textwidth}
  \begin{center}
    \includegraphics[width=1\textwidth]{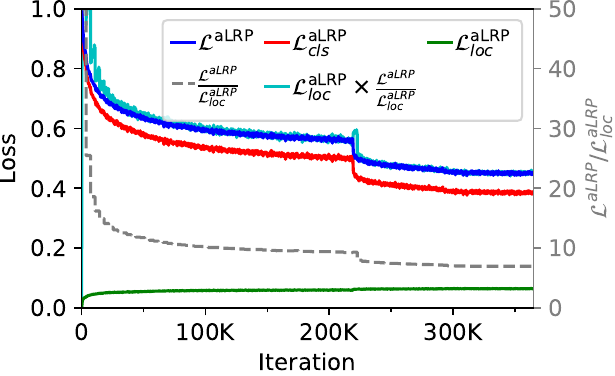}
  \end{center}
    \caption{aLRP Loss and its components. The localisation component is self-balanced.\label{fig:LossPlot}}
\end{wrapfigure}

\textbf{Interpretation of the Components:} A distinctive property of aLRP Loss is that classification  and  localisation errors are handled in a unified manner: i.e. with aLRP, both  classification and localisation branches use the entire output of the detector, instead of working in their separate domains as conventionally done. As shown in Fig. \ref{fig:ToyExample}(a,b),  $\mathcal{L}^\mathrm{aLRP}_{cls}$ takes into account localisation errors of detections with larger scores ($s$) and promotes the detections with larger IoUs to have higher $s$, or suppresses the detections with high-$s$\&low-IoU.  Similarly, $\mathcal{L}^\mathrm{aLRP}_{loc}$ inherently weighs each positive based on its classification rank (see Appendix F for the weights): the contribution of a positive increases if it has a larger $s$. To illustrate, in Fig. \ref{fig:ToyExample}(c), while $\mathcal{E}_{loc}(p_1)$ (i.e. with largest $s$) contributes to each $\mathcal{L}_{loc}(i)$; $\mathcal{E}_{loc}(p_3)$ (i.e. with the smallest $s$) only contributes once with a very low weight due to its rank normalizing $\mathcal{L}_{loc}(\mathrm{p}_3)$. Hence,  the localisation branch effectively focuses on detections ranked higher wrt $s$.  

\subsection{A Self-Balancing Extension for the Localisation Task}\label{sect:self_balance}
LRP metric yields localisation error only if a detection is classified correctly (Sec. \ref{subsec:LRP}). Hence,  when the classification performance is poor (e.g. especially at the beginning of training), the aLRP Loss is dominated by the classification error ($N_{FP}(i)/\mathrm{rank}(i) \approx 1$ and $\ell ^{\mathrm{LRP}} (i) \in [0,1]$ in Eq. \eqref{eq:LRPReformulation}). As a result,  the localisation head is hardly trained at the beginning (Fig. \ref{fig:LossPlot}). Moreover, Fig. \ref{fig:LossPlot} also shows that $\mathcal{L}^{\mathrm{aLRP}}_{cls}/\mathcal{L}^{\mathrm{aLRP}}_{loc}$ varies significantly throughout training. To alleviate this, we propose a simple and dynamic \textit{self-balancing} (SB) strategy using the gradient magnitudes: note that $ \sum_{i \in \mathcal{P}} \abs{{\partial \mathcal{L}}^{\mathrm{aLRP}}/{\partial s_i}} = \sum_{i \in \mathcal{N}} \abs{{\partial \mathcal{L}}^{\mathrm{aLRP}}/{\partial s_i}} \approx \mathcal{L}^{\mathrm{aLRP}}$ (see Theorem \ref{theorem:BalancedTraining} and Appendix F). Then, assuming that the gradients wrt scores and boxes are proportional to  their contributions to the aLRP Loss, we multiply ${\partial \mathcal{L}}^{\mathrm{aLRP}}/{\partial B}$ by the average $\mathcal{L}^{\mathrm{aLRP}}/\mathcal{L}^{\mathrm{aLRP}}_{loc}$ of the previous epoch.

\comment{
\section{[OLD] Average Localization-Recall-Precision (aLRP) Loss}
\label{sec:aLRPLoss}

\begin{figure}[t]
    \centerline{
        \includegraphics[width=0.98\textwidth]{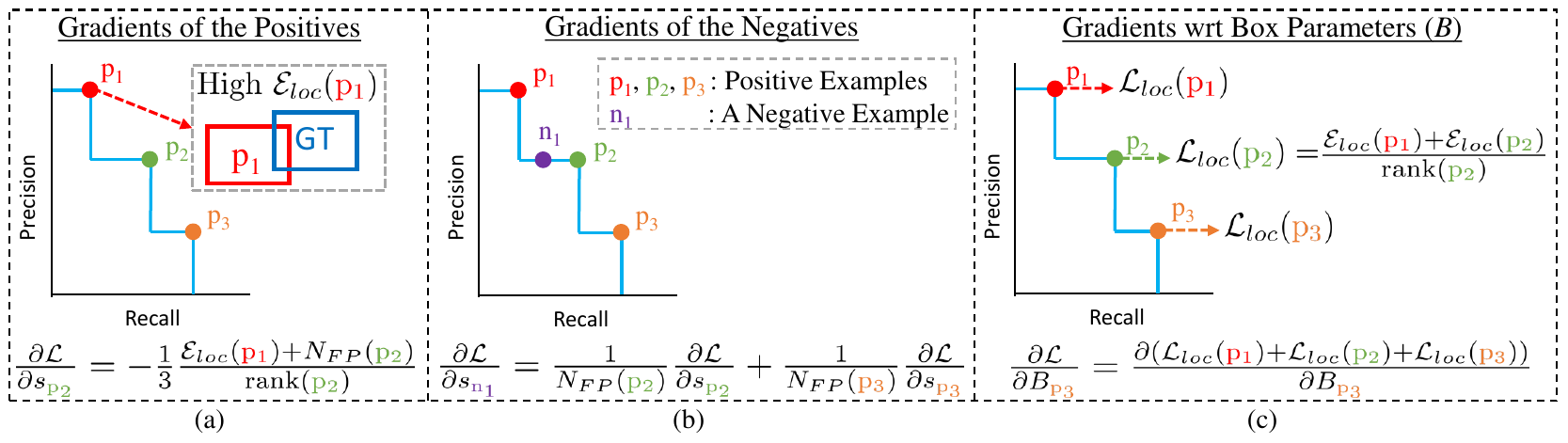}
    }
    \caption{\textbf{An example presenting that aLRP Loss assigns gradients to each branch by considering the outputs of both branches.} The examples on the RP curve are in sorted order wrt their scores. \textbf{(a)} Gradients of the foreground examples: For any $x$, the gradients of the classifier consider (i) the localization errors of the detections with larger scores than itself and (ii) number of background objects with larger scores. For example, the large low IoU of $\textcolor{red}{x}$ increases the gradient of $s_{\textcolor{green}{x}}$ to suppress high-precision $\textcolor{red}{x}$. \textbf{(b)} Gradients of the foreground examples: The gradient of a foreground detection is uniformly distributed over the background detections with larger scores. Summing over the contributions from all foreground detections yield the gradient of a background detection. \textbf{(c)} Gradients of the box parameters, $B$: While $\textcolor{red}{x}$, with highest score, is included by each local regression error, $\mathrm{LRP}_{loc}(x)$, $\textcolor{orange}{x}$ contributes only once with the largest $rank(\textcolor{orange}{x})$. (Best viewed in color.) 
    \label{fig:ToyExample}
} 
\end{figure}

This section defines aLRP Loss, derives its gradients and provides details for our training scheme.

\subsection{Definition}
While incorporating LRP into our loss function, we assume that: (1) Since the object hypotheses (i.e. anchors) are dense, $N_{FN}=0$. (2) Since the anchors are split as positives and negatives, $\mathcal{TP}$ and $\mathcal{FP}$ in Equation \ref{eq:LRPdefcompact} are replaced by $\mathcal{P}$ and $\mathcal{N}$ respectively. With these assumptions and sticking to the notation introduced in Section \ref{subsec:APLoss}, $\mathrm{LRP}(s_i)$ is:
\begin{align}
\label{eq:LRPReformulation}
    \mathrm{LRP}(s_i) 
    = \frac{1}{\mathrm{rank(i)}}
    \left( \sum \limits_{k \in \mathcal{P}, k \neq i}  \mathcal{E}_{loc}(k) H(x_{ik})+ \mathcal{E}_{loc}(i) + N_{FP}(i) \right)
\end{align}
such that $\mathcal{E}_{loc}(i) = \frac{1-IoU(i)}{1-\tau}$ and $N_{FP}(i)=\sum \limits_{k \in \mathcal{N}} H(x_{ik})$. Then, we define aLRP Loss, $\mathcal{L}^{aLRP}$, as the average of the $\mathrm{LRP}(s_i)$ errors over the $i \in \mathcal{P}$. Formally,

\begin{align}
\label{eq:aLRPLoss}
    \mathcal{L}^{aLRP}:=\frac{1}{|\mathcal{P}|}\sum \limits_{i \in \mathcal{P}} \mathrm{LRP}(s_i) 
    = \frac{1}{|\mathcal{P}|}\sum \limits_{i \in \mathcal{P}} \frac{1}{\mathrm{rank(i)}}
    \left( \sum \limits_{k \in \mathcal{P}, k \neq i}  \mathcal{E}_{loc}(k) H(x_{ik})+ \mathcal{E}_{loc}(i) + N_{FP}(i) \right).
\end{align}

Noting that $\tau$ is the positive labelling threshold of the sampler (i.e. $0.50$), Eq. \ref{eq:aLRPLoss} does not contain any hyperparameter.

\subsection{Optimization and Interaction of the Branches}
In order to provide more insight and facilitate gradient derivation, we split Equation \ref{eq:aLRPLoss} into two as regression and classification components such that $\mathcal{L}^{aLRP}=\mathcal{L}^{aLRP}_{cls}+\mathcal{L}^{aLRP}_{loc}$, where
\begin{align}
    \label{eq:aLRPComponents}
    \mathcal{L}^{aLRP}_{cls} &= \frac{1}{|\mathcal{P}|}\sum \limits_{i \in \mathcal{P}} \frac{N_{FP}(i)}{\mathrm{rank}(i)} \text{, and }\;\; \mathcal{L}^{aLRP}_{loc} = \frac{1}{|\mathcal{P}|}\sum \limits_{i \in \mathcal{P}} \frac{\sum \limits_{k \in \mathcal{P}, k \neq i}  \mathcal{E}_{loc}(k) H(x_{ik})+ \mathcal{E}_{loc}(i) }{\mathrm{rank}(i)}.
\end{align}

To start with the gradients of the classifier, we follow our formulation from Section \ref{sec:generalization}:

\textbf{Step 1: Identifying Primary Terms:} The local error of a positive example is $\mathrm{LRP}(i)$, hence $\mathcal{E}^{aLRP}(i) = \mathrm{LRP}(i)$ and adopt a uniform distribution in Eq. \ref{eq:GeneralPrimaryTermDefinition} to obtain the primary terms of aLRP:
\begin{align}
    \label{eq:aLRPGeneralPrimaryTermDefinition}
    \Psi^{aLRP}_{ij} = \mathrm{LRP}(i)  \times \frac{H(x_{ij})}{N_{FP}(i)} = \frac{ \sum \limits_{k \in \mathcal{P}, k \neq i}  \mathcal{E}_{loc}(k) H(x_{ik})+ \mathcal{E}_{loc}(i) + N_{FP}(i) }{\mathrm{rank}(i)} \times \frac{H(x_{ij})}{N_{FP}(i)}.
\end{align}

\textbf{Step 2: Identifying Gradients:} Following Observation \ref{observation:GradMatrix}, here we define $G_{ij} ^ {aLRP}$. Note that when $i$th positive is ranked accordingly, there is still a localization loss, hence, ${\Psi^{aLRP}_{ij}}^*$ is:
\begin{align}
    {\Psi^{aLRP}_{ij}}^* = \frac{ \mathcal{E}_{loc}(i)}{rank(i)} \times \frac{H(x_{ij})}{N_{FP}(i)}
\end{align}
Having determined ${\Psi^{aLRP}_{ij}}^*$, $G_{ij} ^{aLRP}$ is simply:
\begin{align}
    \label{eq:aLRPClassificationGrads}
    G_{ij} ^{aLRP} = \Psi^{aLRP}_{ij} - {\Psi^{aLRP}_{ij}}^* 
    = \frac{ \sum \limits_{k \in \mathcal{P}, k \neq i}  \mathcal{E}_{loc}(k) H(x_{ik}) + N_{FP}(i) }{\mathrm{rank}(i)} \times \frac{H(x_{ij})}{N_{FP}(i)}.
\end{align}
We note that while $\mathcal{L}^{aLRP}_{cls}=\mathcal{L}^{AP}$ in quantity their gradients are different since $\mathcal{L}^{aLRP}_{loc}$ requires ranking, hence we consider $\mathcal{L}^{aLRP}$ in total to derive the gradients of the aLRP Loss with respect to the scores. Fig. \ref{fig:ToyExample}(a,b) provide more insight on the gradient signals obtained using $G_{ij} ^{aLRP}$ values. By taking into account the localization errors of the detections with larger scores, the classifier aims to promote the detections with larger IoUs to have higher scores, or suppress the detections with high-score and low-IoU. Hence, we conclude that classification branch considers the entire output of the detector appropriately.

Secondly, $\mathcal{L}^{aLRP}$ is differentiable wrt to the regressor output. To be more specific, $\mathcal{L}^{aLRP}_{cls}$ does not include any regressor output, hence the gradients can be determined from $\mathcal{L}^{aLRP}_{loc}$ in which $rank(i)$ is constant and $\mathcal{E}_{loc}(i)$ is differentiable \cite{UnitBox,GIoULoss,DIoULoss}. Note that $\mathcal{L}^{aLRP}_{loc}$ inherently weighs each positive considering its rank. In other words, the contribution of the positive anchor increases if it has a larger score. To illustrate, in Fig. \ref{fig:ToyExample}(c), while $\mathcal{E}_{loc}(x)$ of the example with the largest score contributes at all levels to the individual $\mathcal{L}^{aLRP}_{loc}$ terms, the one with the smallest score only contributes once with a very low weight due to its rank normalizing the local regression error. Hence, in such a way, the regressor focuses on the detections ranked higher with respect to the the scores.
 


\subsection{Training Details}
\label{subsec:TrainingDetails}
Here, we note two critical ideas to improve the performance of our design. Firstly, faster converging forms of IoU-based regression losses (e.g. generalized IoU (GIoU) \cite{GIoULoss} and distance IoU (DIoU) \cite{DIoULoss}) can easily be incorporated into our design by replacing IoU Loss (i.e. $1- \mathrm{IoU}(i)$) in $\mathcal{E}_{loc}(i)$ and normalizing the range to $[0,1]$. Secondly, in the early epochs of the training the classifier is not able to classify the objects and the range of aLRP (i.e. $[0,1]$) is occupied by the classifier. In otherwords, since there is very few true positive with large $\mathrm{rank}(i)$ in $\mathcal{L}^{aLRP}_{loc}$ (Eq. \ref{eq:aLRPComponents}), the regression loss is low as well. To alleviate that, similar to Chen et al.\cite{GradNorm}, we add a hyperparameter-free adaptive self-balancing (ASB) by manipulating directly the gradients using the information from the last epoch: We note that normalizing the total magnitude of the gradients wrt background and foreground scores (Theorem \ref{theorem:BalancedTraining}) by $|\mathcal{P}|$  makes them to approximately $\mathcal{L}^{aLRP}$ (see. Eq. \ref{eq:aLRPGeneralPrimaryTermDefinition}). Assuming that the gradients wrt scores and boxes are proportional to the their contributions to the $\mathrm{aLRP}$ loss value, we multiply the gradients of $\mathrm{aLRP}$ wrt to the boxes by the average $ \frac{\mathcal{L}^{aLRP}}{\mathcal{L}^{aLRP}_{loc}}$ in the last epoch, which makes the normalized gradients with respect to foreground scores, background scores and boxes (three competing tasks) are approximately equal during training. Finally, similar to Chen et al. \cite{APLoss}, we replace the step function by piecewise step function (see Supp.Mat. or \cite{APLoss} for the details). Overall algorithm is presented in the Supp.Mat. and has the same time (i.e. $O(|\mathcal{P}| \times|\mathcal{N}|)$) and space (i.e. $O(|\mathcal{N}|)$) complexity with the AP Loss.

}
\section{Experiments}
\label{sec:Experiments}

\textbf{Dataset:} We train all our models on  COCO \textit{trainval35K} set \cite{COCO} (115K images), test on \textit{minival} set (5k images) and compare with the state-of-the-art (SOTA) on \textit{test-dev} set (20K images). 

\textbf{Performance Measures:} COCO-style AP \cite{COCO} and when possible optimal LRP \cite{LRP} (Sec. \ref{subsec:LRP}) are used for comparison. For more insight into aLRP Loss, we use Pearson correlation coefficient ($\rho$) to measure correlation between the rankings of classification and localisation, averaged over classes.

\textbf{Implementation Details:} For training, we use $4$ v100 GPUs. The batch size is 32 for training with $512 \times 512$ images (aLRPLoss500), whereas it is 16 for  $800 \times 800$ images (aLRPLoss800). Following AP Loss, our models are trained for 100 epochs using stochastic gradient descent with a momentum factor of $0.9$. We use a learning rate of $0.008$ for aLRPLoss500 and $0.004$ for aLRPLoss800, each decreased by factor $0.1$ at epochs 60 and 80. Similar to previous work \cite{APLoss,CenterNet}, standard data augmentation methods from SSD \cite{SSD} are used. At test time, we rescale  shorter sides of images  to $500$ (aLRPLoss500) or $800$ (aLRPLoss800) pixels by ensuring that the longer side does not exceed $1.66 \times$ of the shorter side. NMS is applied to $1000$ top-scoring detections using $0.50$ as IoU threshold.

\subsection{Ablation Study}
\label{sect:ablation}
In this section, in order to provide a fair comparison, we build upon the official implementation of our baseline, AP Loss \cite{APOfficialRepo}. Keeping all design choices fixed, otherwise stated, we just replace AP \& Smooth L1 losses by aLRP Loss to optimize RetinaNet \cite{FocalLoss}. We conduct ablation analysis using aLRPLoss500 on ResNet-50 backbone (more ablation experiments are presented in the Appendix G).

\textbf{Effect of using ranking for localisation:} Table \ref{tab:minival} shows that using a ranking loss for localisation improves AP (from $35.5$ to $36.9$). For better insight, $\mathrm{AP_{90}}$ is also included in Table \ref{tab:minival}, which shows $\sim$5 points increase  despite similar $\mathrm{AP_{50}}$ values. This confirms that aLRP Loss does produce high-quality outputs for both branches, and boosts the performance for larger IoUs.

\textbf{Effect of Self-Balancing (SB):} Section \ref{sect:self_balance} and Fig. \ref{fig:LossPlot} discussed how $\mathcal{L}^\mathrm{aLRP}_{cls}$ and $\mathcal{L}^\mathrm{aLRP}_{loc}$ behave during training and introduced self-balancing to improve training of the localisation branch. Table \ref{tab:minival} shows that SB provides +1.8AP gain, similar $\mathrm{AP}_{50}$ and +8.4 points in $\mathrm{AP}_{90}$ against AP Loss. Comparing SB with constant weighting in Table \ref{tab:scalerweight}, our SB approach provides slightly better performance than constant weighting, which requires extensive tuning and end up with different $w_r$ constants for IoU and GIoU. Finally, Table \ref{tab:initialization} presents that initialization of SB (i.e. its value for the first epoch) has a negligible effect on the performance even with very large values. We use 50 for initialization.


\textbf{Using GIoU:} Table \ref{tab:minival} suggests robust IoU-based regression (GIoU) improves performance slightly. 

\textbf{Using ATSS:} Finally, we replace the standard IoU-based assignment by ATSS \cite{ATSS}, which uses less anchors and decreases training time notably for aLRP Loss: One iteration drops from 0.80s to 0.53s with ATSS (34\% more efficient with ATSS) -- this time is 0.71s and 0.28s for AP Loss and Focal Loss respectively. With ATSS, we also observe +1.3AP improvement (Table \ref{tab:minival}). See App. G for details.

Hence, we use GIoU \cite{GIoULoss} as part of aLRP Loss, and employ ATSS \cite{ATSS} when training RetinaNet. 
%

\begin{table}[]
    \centering
    \setlength{\tabcolsep}{0.3em}
    \footnotesize
    \caption{Ablation analysis on COCO \textit{minival}. For optimal LRP (oLRP), lower is better.}
    \label{tab:minival}
    \begin{tabular}{|c|c|c|c|c|c|c|c|c|c||c|} \hline
        Method& Rank-Based $\mathcal{L}_c$ & Rank-Based $\mathcal{L}_r$&SB&ATSS&$\mathrm{AP}$&$\mathrm{AP_{50}}$&$\mathrm{AP_{75}}$&$\mathrm{AP_{90}}$&$\mathrm{oLRP}$&$\rho$\\ \hline \hline
        AP Loss \cite{APLoss}&\checkmark& & & & $35.5$&$58.0$&$37.0$&$9.0$&$71.0$ &$0.45$
        \\
        \hline 
         \multirow{4}{*}{aLRP Loss} &\checkmark&\checkmark (w IoU)& & &$36.9$&$57.7$&$38.4$&$13.9$&$69.9$ &$0.49$
         \\ 
         &\checkmark&\checkmark (w IoU)&\checkmark& &$38.7$&$58.1$&$40.6$&$17.4$&$68.5$&$0.48$
         \\ 
         &\checkmark&\checkmark (w GIoU)&\checkmark& &$38.9$&$58.5$&$40.5$&$17.4$&$68.4$&$0.48$
         \\ 
         &\checkmark&\checkmark (w GIoU)&\checkmark&\checkmark&$40.2$&$60.3$&$42.3$&$18.1$&$67.3$&$0.48$ 
         \\ 

        \hline
    \end{tabular}
\end{table}

\begin{table}[]
\RawFloats
\parbox{.58\linewidth}{
    \centering
    \setlength{\tabcolsep}{0.4em}
    \footnotesize
    \caption{SB does not require tuning and slightly outperforms constant weighting for both IoU types.}
    \label{tab:scalerweight}
    \begin{tabular}{|c|c c c c c c c|c|} \hline
        $w_r$&$1$&$2$&$5$&$10$&$15$&$20$&$25$&SB \\ \hline 
        w IoU&$36.9$&$37.8$&$38.5$&$38.6$&$38.3$&$37.1$&$36.0$&$\mathbf{38.7}$ \\ \hline
        w GIoU&$36.0$&$37.0$&$37.9$&$38.7$&$38.8$&$38.7$&$38.8$&$\mathbf{38.9}$ \\ \hline
    \end{tabular}
}
\hfill
\parbox{.375\linewidth}{
    \centering
    \setlength{\tabcolsep}{0.5em}
    \footnotesize
    \caption{SB is not affected significantly by the initial weight in the first epoch ($w_r$) even for large values.}
    \label{tab:initialization}
    \begin{tabular}{|c|c c c c|} \hline
        $w_r$&$1$&$50$&$100$&$500$ \\ \hline 
        AP&$38.8$&$\mathbf{38.9}$&$38.7$&$38.5$\\ \hline
    \end{tabular}
}
\end{table}
\begin{wraptable}{r}{6.5cm}
    \setlength{\tabcolsep}{0.1em}
    \footnotesize
    \caption{Effect of correlating rankings.}
    \label{tab:correlationEffect}
    \begin{tabular}{|c|c|c|c|c|c|} \hline
        $\mathcal{L}$&$\rho$&$\mathrm{AP}$&$\mathrm{AP_{50}}$&$\mathrm{AP_{75}}$&$\mathrm{AP_{90}}$\\ \hline \hline
        aLRP Loss&$0.48$&$38.7$&$58.1$&$40.6$&$17.4$\\ 
        \hline
        Lower Bound&$-1.00$&$28.6$&$58.1$&$23.6$&$5.6$\\ 
        Upper Bound&$1.00$&$48.1$&$58.1$&$51.9$&$33.9$\\ 
        \hline
    \end{tabular}
\end{wraptable}
\subsection{More insight on aLRP Loss}
\textbf{Potential of Correlating Classification and Localisation.}
We analyze two bounds: (i) A \textit{Lower Bound} where localisation provides an inverse ranking compared to classification. (ii) An \textit{Upper Bound} where localisation provides exactly the same ranking as classification. Table \ref{tab:correlationEffect} shows that correlating ranking can have a significant effect (up to $20$ AP) on the performance especially for  larger IoUs. Therefore, correlating rankings promises significant  improvement (up to $\sim10$AP). Moreover, while $\rho$ is $0.44$ and $0.45$ for Focal Loss (results not provided in the table) and AP Loss (Table \ref{tab:minival}), respectively, aLRP Loss yields higher correlation ($0.48, 0.49$).
%


\begin{figure*}[t!]
    \begin{subfigure}[t]{0.45\textwidth}
        \centering
        \includegraphics[width=1\textwidth]{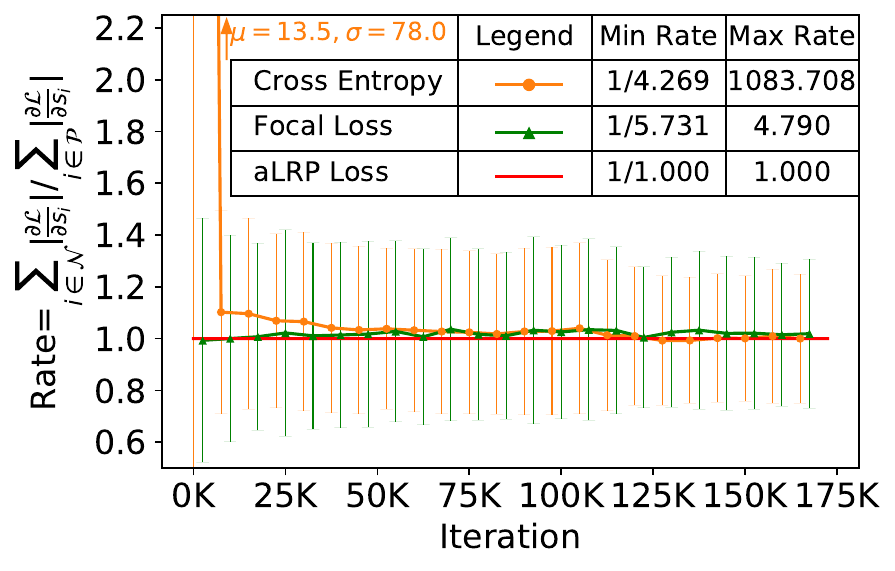}
    \end{subfigure}
    ~
    \begin{subfigure}[t]{0.42\textwidth}
        \centering
        \includegraphics[width=1\textwidth]{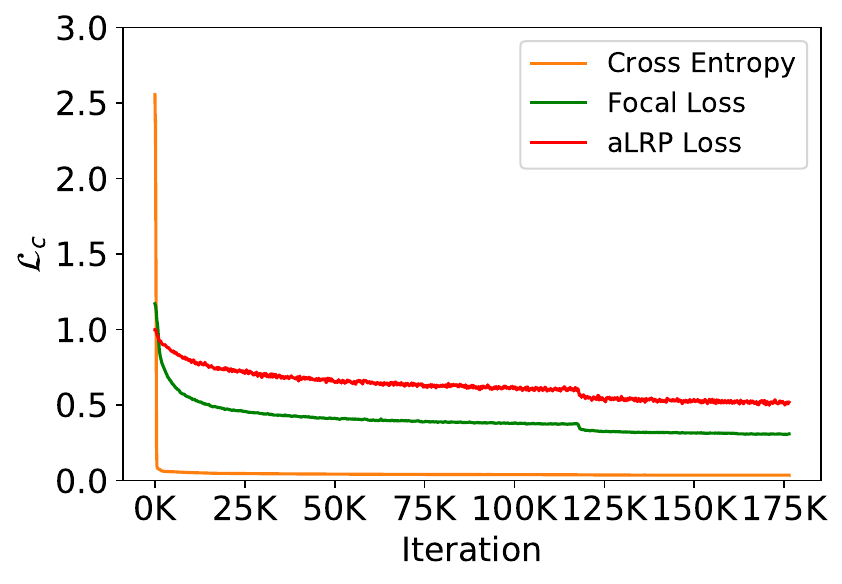}
    \end{subfigure}    
    \caption{\textbf{(left)} The rate of the total gradient magnitudes of negatives to positives. \textbf{(right)} Loss values. \label{fig:GradComp}}
\end{figure*}
\textbf{Analysing Balance Between Positives and Negatives.}
For this analysis, we compare Cross Entropy Loss (CE), Focal Loss (FL) and aLRP Loss on RetinaNet trained for 12 epochs and average results over 10 runs. Fig. \ref{fig:GradComp} experimentally confirms Theorem \ref{theorem:BalancedTraining} for aLRP Loss ($\mathcal{L}^{\mathrm{aLRP}}_{cls}$), as it exhibits perfect balance between the gradients throughout training. However, we see large fluctuations in derivatives of CE and FL (left), which biases training towards positives or negatives alternately across iterations. As expected, imbalance impacts CE more as it quickly drops (right), overfitting in favor of negatives since it is dominated by the error and gradients of these large amount of negatives.

\comment{
\textbf{It is crucial to set ${L_{ij}}^*$ as the target error when an example is ranked properly.} Fig. \ref{fig:GradComp} presents a case where ${L_{ij}}^*$ is set to 0 (i.e. minimum value of aLRP). For this case, the training continues properly similar to aLRP Loss up to a point and then diverges. Note that this occurs when the positives start to be ranked properly but are still assigned gradients since ${L_{ij}}^*-{L_{ij}} \neq 0$ due to the nonzero localization error. This causes $\sum \limits_{i \in \mathcal{P}} \abs{\frac{\partial \mathcal{L}}{\partial s_i}} > \sum \limits_{i \in \mathcal{N}} \abs{\frac{\partial \mathcal{L}}{\partial s_i}}$, violating Theorem \ref{theorem:BalancedTraining} (compare min-rate and max-rate in Fig. \ref{fig:GradComp}). Therefore, assigning proper targets as indicated in Section \ref{sec:generalization} is crucial for balanced training.
}


\subsection{Comparison with State of the Art (SOTA)}
\label{sect:comparison}
Different from the ablation analysis, we find it useful to decrease the learning rate of aLRPLoss500 at epochs 75 and 95. For SOTA comparison, we use the mmdetection framework \cite{mmdetection} for efficiency (we reproduced Table \ref{tab:minival} using our mmdetection implementation, yielding similar results - see our repository). Table \ref{tab:testdev} presents the results, which are discussed below: 

\textbf{Ranking-based Losses.} aLRP Loss yields significant gains over other ranking-based solutions: e.g., compared with AP Loss, aLRP Loss provides +5.4AP for scale 500 and +5.1AP for scale 800. Similarly, for scale 800, aLRP Loss performs  4.7AP better than DR Loss with ResNeXt-101.

\textbf{Methods combining branches.} Although a direct comparison is not fair since different conditions are used, we observe a significant margin (around 3-5AP in scale 800) compared to other approaches that combine localisation and classification.

\textbf{Comparison on scale 500.} We see that, even with ResNet-101, aLRPLoss500  outperforms all other methods with 500 test scale. With ResNext-101, aLRP Loss outperforms its closest counterpart (HSD) by 2.7AP and also in all sizes ($\mathrm{AP_S}$-$\mathrm{AP_L}$).

\textbf{Comparison on scale 800.} For 800 scale, aLRP Loss achieves 45.9 and 47.8AP on ResNet-101 and ResNeXt-101 backbones respectively. Also in this scale, aLRP Loss consistently outperforms its closest counterparts (i.e. FreeAnchor and CenterNet) by 2.9AP and reaches the highest results wrt all performance measures. With DCN \cite{DCNv2}, aLRP Loss reaches 48.9AP, outperforming ATSS by 1.2AP.



\begin{table}[t]
    \centering
    \footnotesize 
    \caption{Comparison with the SOTA detectors on COCO \textit{test-dev}. $S, \times 1.66$ implies that the image is rescaled such that its longer side cannot exceed $1.66 \times S$ where $S$ is the size of the shorter side.  R:ResNet, X:ResNeXt, H:HourglassNet, D:DarkNet, De:DeNet. We use ResNeXt101 64x4d.}
    \label{tab:testdev}
    \setlength{\tabcolsep}{0.25em}
    \begin{tabular}{|l|c|c|c|c|c|c|c|c|c|} \hline
        \cellcenter{Method}&Backbone&Training Size&Test Size&$\mathrm{AP}$&$\mathrm{AP_{50}}$&$\mathrm{AP_{75}}$&$\mathrm{AP_{S}}$ &$\mathrm{AP_{M}}$&$\mathrm{AP_{L}}$ \\ 
        \hhline{==========}
        \textit{One-Stage Methods}& & & & & & & & &\\
        RefineDet \cite{AnchorRefine}$^{\ddagger}$&R-101&$512\times512$&$512\times512$&$36.4$&$57.5$&$39.5$&$16.6$&$39.9$&$51.4$\\
        EFGRNet \cite{EnrichedFeatureGuided}$^{\ddagger}$&R-101&$512\times512$&$512\times512$&$39.0$&$58.8$&$42.3$&$17.8$&$43.6$&$54.5$\\ 
        ExtremeNet \cite{ExtremeNet}$^{* \ddagger}$&H-104&$511\times511$&original&$40.2$&$55.5$&$43.2$&$20.4$&$43.2$&$53.1$\\
        RetinaNet \cite{FocalLoss}&X-101&$800, \times 1.66$&$800, \times 1.66$&$40.8$&$61.1$&$44.1$&$24.1$&$44.2$&$51.2$\\
        HSD \cite{HierarchicalShotDet} $^{\ddagger}$&X-101&$512\times512$&$512\times512$&$41.9$&$61.1$&$46.2$&$21.8$&$46.6$&$57.0$\\ 
        FCOS \cite{FCOS}$^\dagger$&X-101&$(640, 800), \times 1.66$&$800, \times 1.66$&$44.7$&$64.1$&$48.4$&$27.6$&$47.5$&$55.6$\\         
        CenterNet \cite{CenterNet}$^{* \ddagger}$&H-104&$511\times511$&original&$44.9$&$62.4$&$48.1$&$25.6$&$47.4$&$57.4$\\          
        ATSS \cite{ATSS}$^\dagger$&X-101-DCN&$(640, 800), \times 1.66$&$800, \times 1.66$&$47.7$&$66.5$&$51.9$&$29.7$&$50.8$&$59.4$\\        
        \hhline{==========}
        \textit{Ranking Losses}& & & & & & & & &\\  
        AP Loss500 \cite{APLoss}$^\ddagger$&R-101&$512 \times 512$&$500, \times 1.66$&$37.4$&$58.6$&$40.5$&$17.3$&$40.8$&$51.9$\\
        AP Loss800 \cite{APLoss}$^\ddagger$&R-101&$800 \times 800$&$800, \times 1.66$&$40.8$&$63.7$&$43.7$&$25.4$&$43.9$&$50.6$\\
        DR Loss \cite{DRLoss}$^\dagger$&X-101&$(640, 800), \times 1.66$&$800, \times 1.66$&$43.1$&$62.8$&$46.4$&$25.6$&$46.2$&$54.0$\\        
        \hhline{==========}
        \textit{Combining Branches}& & & & & & & & &\\
        LapNet \cite{LapNet}&D-53&$512 \times 512$&$512 \times 512$&$37.6$&$55.5$&$40.4$&$17.6$&$40.5$&$49.9$\\ 
        Fitness NMS \cite{BoundedIoU}&De-101&$512, \times 1.66$&$768, \times 1.66$&$39.5$&$58.0$&$42.6$&$18.9$&$43.5$&$54.1$\\ 
        Retina+PISA \cite{PrimeSample}&R-101&$800, \times 1.66$&$800, \times 1.66$&$40.8$&$60.5$&$44.2$&$23.0$&$44.2$&$51.4$\\
        FreeAnchor \cite{FreeAnchor}$^\dagger$&X-101&$(640, 800), \times 1.66$&$800, \times 1.66$&$44.9$&$64.3$&$48.5$&$26.8$&$48.3$&$55.9$\\
        \hhline{==========}
        \textit{Ours}& & & & & & & & &\\
        aLRP Loss500$^\ddagger$&R-50&$512 \times 512$&$500, \times 1.66$&$41.3$&$61.5$&$43.7$&$21.9$&$44.2$&$54.0$\\  
        aLRP Loss500$^\ddagger$&R-101&$512 \times 512$&$500, \times 1.66$&$42.8$&$62.9$&$45.5$&$22.4$&$46.2$&$56.8$\\  
        aLRP Loss500$^\ddagger$&X-101&$512 \times 512$&$500, \times 1.66$&$44.6$&$65.0$&$47.5$&$24.6$&$48.1$&$58.3$  \\
        aLRP Loss800$^\ddagger$&R-101&$800 \times 800$&$800, \times 1.66$&$45.9$&$66.4$&$49.1$&$28.5$&$48.9$&$56.7$\\ 
        aLRP Loss800$^\ddagger$&X-101&$800 \times 800$&$800, \times 1.66$&$47.8$&$68.4$&$51.1$&$30.2$&$50.8$&$59.1$ \\ 
        aLRP Loss800$^\ddagger$&X-101-DCN&$800 \times 800$&$800, \times 1.66$&$\mathbf{48.9}$&$\mathbf{69.3}$&$\mathbf{52.5}$&$\mathbf{30.8}$&$\mathbf{51.5}$&$\mathbf{62.1}$ \\ 
        \hline
        \hhline{==========}
        \textit{Multi-Scale Test}& & & & & & & & &\\
        aLRP Loss800$^\ddagger$&X-101-DCN&$800 \times 800$&$800, \times 1.66$&$50.2$&$70.3$&$53.9$&$32.0$&$53.1$&$63.0$\\ 
        \hline        
    \end{tabular}
    {\\ \footnotesize $^\dagger$: multiscale training, $^\ddagger$: SSD-like augmentation, $^*$: Soft NMS \cite{SoftNMS} and flip augmentation at test time}
\end{table}

\subsection{Using aLRP Loss with Different Object Detectors}
Here, we use aLRP Loss to train FoveaBox \cite{FoveaBox} as an anchor-free detector, and Faster R-CNN \cite{FasterRCNN} as a two-stage detector. All models use $500$ scale setting, have a ResNet-50 backbone and follow our mmdetection implementation \cite{mmdetection}. Further implementation details are presented in Appendix G.


\textbf{Results on FoveaBox:} To train FoveaBox, we keep the learning rate same with RetinaNet (i.e. $0.008$) and only replace the loss function by aLRP Loss. Table \ref{tab:anchorfree} shows that aLRP Loss outperforms Focal Loss and AP Loss, each combined by Smooth L1 (SL1 in Table \ref{tab:anchorfree}), by $1.4$ and $3.2$ AP points (and similar oLRP points) respectively. Note that aLRP Loss also simplifies tuning hyperparameters of Focal Loss, which are set in FoveaBox to different values from RetinaNet. One training iteration of Focal Loss, AP Loss and aLRP Loss take $0.34$, $0.47$ and $0.54$ sec respectively.

\textbf{Results on Faster R-CNN:} To train Faster R-CNN, we remove sampling, use aLRP Loss to train both stages (i.e. RPN and Fast R-CNN) and reweigh aLRP Loss of RPN by $0.20$. Thus, the number of hyperparameters is reduced from nine (Table \ref{tab:Hyperparameters}) to three (two $\delta$s for step function, and a weight for RPN). We validated the learning rate of aLRP Loss as $0.012$, and train baseline Faster R-CNN by both L1 Loss and GIoU Loss for fair comparison. aLRP Loss outperforms these baselines by more than 2.5AP and 2oLRP points while simplifying the training pipeline (Table \ref{tab:twostage}). One training iteration of Cross Entropy Loss (with L1) and aLRP Loss take $0.38$ and $0.85$ sec respectively.

\begin{table}[]
\RawFloats
\parbox{.48\linewidth}{
    \centering
    \setlength{\tabcolsep}{0.15em}
    \footnotesize
    \caption{Comparison on FoveaBox \cite{FoveaBox}.}
    \label{tab:anchorfree}
    \begin{tabular}{|c|c|c|c|c|c|} \hline
    $\mathcal{L}$&$\mathrm{AP}$&$\mathrm{AP_{50}}$&$\mathrm{AP_{75}}$&$\mathrm{AP_{90}}$&$\mathrm{oLRP}$ \\ \hline
    Focal Loss+SL1&$38.3$&$57.8$&$40.7$&$15.7$&$68.8$\\ \hline
    AP Loss+SL1&$36.5$&$58.3$&$38.2$&$11.3$&$69.8$\\ \hline
    aLRP Loss (Ours)&$\mathbf{39.7}$&$\mathbf{58.8}$&$\mathbf{41.5}$&$\mathbf{18.2}$&$\mathbf{67.2}$\\ \hline
    \end{tabular}
}
\hfill
\parbox{.52\linewidth}{
    \centering
    \setlength{\tabcolsep}{0.15em}
    \footnotesize
    \caption{Comparison on Faster R-CNN \cite{FasterRCNN}}
    \label{tab:twostage}
    \begin{tabular}{|c|c|c|c|c|c|} \hline
    $\mathcal{L}$&$\mathrm{AP}$&$\mathrm{AP_{50}}$&$\mathrm{AP_{75}}$&$\mathrm{AP_{90}}$&$\mathrm{oLRP}$ \\ \hline
    Cross Entropy+L1&$37.8$&$58.1$&$41.0$&$12.2$&$69.3$\\ \hline
    Cross Entropy+GIoU&$38.2$&$58.2$&$41.3$&$13.7$&$69.0$\\ \hline
    aLRP Loss (Ours)&$\mathbf{40.7}$&$\mathbf{60.7}$&$\mathbf{43.3}$&$\mathbf{18.0}$&$\mathbf{66.7}$\\ \hline
    \end{tabular}
}
\end{table}

\comment{
\begin{table}[t]
    \centering
    \caption{Comparison on COCO \textit{test-dev}. RetinaNet is optimized using aLRP Loss. }
    \label{tab:testdev}
    \begin{tabular}{|l|c|c|c|c|c|c|c|} \hline
        \cellcenter{Method}&Backbone&$\mathrm{AP}$&$\mathrm{AP_{50}}$&$\mathrm{AP_{75}}$&$\mathrm{AP_{S}}$ &$\mathrm{AP_{M}}$&$\mathrm{AP_{L}}$ \\ 
        \hhline{========}
        \textit{One-Stage Methods}& & & & & & &\\
        AP Loss &ResNet-101&$40.8$&$63.7$&$43.7$&$25.4$&$43.9$&$50.6$\\
        FCOS &ResNeXt-101&$44.7$&$64.1$&$48.4$&$27.6$&$47.5$&$55.6$\\         
        CenterNet &Hourglass-104&$44.9$&$62.4$&$48.1$&$25.6$&$47.4$&$57.4$\\
        FreeAnchor &ResNeXt-101&$44.9$&$64.3$&$48.5$&$26.8$&$48.3$&$55.9$\\ 
        ATSS &ResNeXt-101-DCN&$47.7$&$66.5$&$51.9$&$29.7$&$50.8$&$59.4$\\        
        \hhline{========}


        \textit{Ours}& & & & & & &\\
        aLRP Loss800&ResNet-101&$45.9$&$66.4$&$49.1$&$28.5$&$48.9$&$56.7$\\ 
        aLRP Loss800&ResNeXt-101&$\mathbf{47.8}$&$\mathbf{68.4}$&$\mathbf{51.1}$&$\mathbf{30.2}$&$\mathbf{50.8}$&$59.1$ \\ 
        aLRP Loss800&ResNeXt-101-DCN&$\mathbf{?}$&$\mathbf{?}$&$\mathbf{?}$&$\mathbf{?}$&$\mathbf{?}$&$?$ \\ 
        \hline
        \comment{
        \hhline{==========}
        \textit{Multi-Scale Test}& & & & & & & & &\\
        DR Loss \cite{DRLoss}$^\dagger$&X-101&$(640, 800), \times 1.66$&N/A&$44.7$&$63.8$&$48.7$&$28.2$&$47.4$&$56.2$\\
        CenterNet \cite{CenterNet}$^*$&H-104&$511\times511$&N/A&$47.0$&$64.5$&$50.7$&$28.9$&$49.9$&$58.9$\\        
        FreeAnchor \cite{FreeAnchor}$^\dagger$&X-101&$(640, 800), \times 1.66$&N/A&$47.3$&$66.3$&$51.5$&$30.6$&$50.4$&$59.0$\\
        aLRP Loss800-78e&X-101&$800 \times 800$&N/A&$\mathbf{47.1}$&$\mathbf{68.1}$&$\mathbf{50.4}$&$\mathbf{30.3}$&$\mathbf{49.4}$&$\mathbf{59.1}$ \\}
        \hline        
    \end{tabular}
\end{table}

\begin{table}[]
    \centering
    \caption{Comparison on FoveaBox using COCO \textit{minival}.}
    \label{tab:anchorfree}
    \begin{tabular}{|c|c|c|c|c|c|c|} \hline
    $\mathcal{L}$&Backbone&$\mathrm{AP}$&$\mathrm{AP_{50}}$&$\mathrm{AP_{75}}$&$\mathrm{AP_{90}}$&$\mathrm{oLRP}$ \\ \hline
    Focal Loss+SL1&ResNet-50&$38.3$&$57.8$&$40.7$&$?$&$?$\\ \hline
    AP Loss+SL1&ResNet-50&$36.5$&$58.3$&$38.2$&$11.3$&$69.8$\\ \hline
    aLRP Loss (Ours)&ResNet-50&$\mathbf{39.7}$&$\mathbf{58.8}$&$\mathbf{41.5}$&$\mathbf{18.2}$&$\mathbf{67.2}$\\ \hline
    \end{tabular}
\end{table}

\begin{table}[]
    \centering
    \caption{Comparison on Faster R-CNN using COCO \textit{minival}.}
    \label{tab:twostage}
    \begin{tabular}{|c|c|c|c|c|c|c|} \hline
    $\mathcal{L}$&Backbone&$\mathrm{AP}$&$\mathrm{AP_{50}}$&$\mathrm{AP_{75}}$&$\mathrm{AP_{90}}$&$\mathrm{oLRP}$ \\ \hline
    Cross Entropy+L1&ResNet-50&$37.8$&$58.1$&$41.0$&$?$&$?$\\ \hline
    Cross Entropy+GIoU&ResNet-50&$?$&$?$&$?$&$?$&$?$\\ \hline
    aLRP Loss (Ours)&ResNet-50&$40.4$&$60.1$&$42.6$&$?$&$?$\\ \hline
    \end{tabular}
\end{table}
}
\section{Conclusion}
\label{sec:Conclusion}

In this paper, we provided a general framework for the error-driven optimization of ranking-based functions. As a special case of this generalization, we introduced aLRP Loss, a  ranking-based, balanced loss function which handles the classification and localisation errors in a unified manner. aLRP Loss has only one hyperparameter which we did not need to tune, as opposed to around 6 in SOTA loss functions.  We showed that using aLRP improves its baselines significantly over different detectors by simplifying parameter tuning, and outperforms all one-stage detectors. 

\section*{Broader Impact}

We anticipate our work to significantly impact the following  domains: 
\begin{enumerate}
    \item \textbf{Object detection}: Our loss function is unique in many important aspects: It unifies localisation and classification in a single loss function. It uses ranking for both classification and localisation. It provides provable balance between negatives and positives, similar to AP Loss. 
    
    These unique merits will contribute to a paradigm shift in the object detection community towards more capable and sophisticated loss functions such as ours. 
    
    \item \textbf{Other computer vision problems with multiple objectives}: Problems including multiple objectives (such as instance segmentation, panoptic segmentation -- which actually has classification and regression objectives) will benefit significantly from our proposal of using ranking for both classification and localisation.     
    
    \item \textbf{Problems that can benefit from ranking}: Many vision problems can be easily converted into a ranking problem. They can then exploit our generalized framework to easily define a loss function and to determine the derivatives. 
    
    
\end{enumerate}

Our paper does not have direct social implications. However, it inherits the following implications of object detectors: Object detectors can be used for surveillance purposes for the betterness of society albeit privacy concerns. When used for detecting targets, an object detector's failure may have severe consequences depending on the application (e.g. self-driving cars). Moreover, such detectors are affected by the bias in data, although they will not try to exploit them for any purposes.


\begin{ack}
This work was partially supported by the Scientific and Technological Research Council of Turkey (T\"UB\.ITAK) through a project titled ``Object Detection in Videos with Deep Neural Networks'' (grant number 117E054). Kemal \"Oks\"uz is supported by the T\"UB\.ITAK 2211-A National Scholarship Programme for Ph.D. students. The numerical calculations reported in this paper were performed at TUBITAK ULAKBIM High Performance and Grid Computing Center (TRUBA), and Roketsan Missiles Inc. sources.
\end{ack}

\bibliographystyle{spbasic}
\bibliography{detectionbibliography}
\appendix
\section*{APPENDIX}


\addcontentsline{toc}{chapter}{\appendixname}
\renewcommand{\thefigure}{A.\arabic{figure}}
\renewcommand{\thetable}{A.\arabic{table}}
\renewcommand{\thetheorem}{A.\arabic{theorem}}
\renewcommand{\theequation}{A.\arabic{equation}}
\renewcommand{\thealgorithm}{A.\arabic{algorithm}}

\section{Details of Figure 1: Comparison of Loss Functions on a Toy Example}
\label{sec:Teaser}


This section aims to present the scenario considered in Figure 1 of the main paper. Section \ref{subsec:the_scenario} explains the scenario, Section \ref{subsec:performance_estimation} and Section \ref{subsec:loss_est} clarify how the performance measures (AP, $\mathrm{AP_{50}}$, etc.) and loss values (cross-entropy, AP Loss, aLRP Loss, etc.) are calculated.

\subsection{The Scenario}
\label{subsec:the_scenario}
We assume that the scenario in Figure 1(a) of the paper includes five ground truths of which four of them are detected as true positives with different Intersection-over-Union (IoU) overlaps by three different detectors (i.e. $\mathrm{C\&R_1}$, $\mathrm{C\&R_2}$, $\mathrm{C\&R_3}$). Each detector has a different ranking for these true positives with respect to their IoUs. In addition, the output of each detector contains the same six detections with different scores as false positives. Note that the IoUs of these false positives are marked with "--" in Figure 1(a) since they do not match with any ground truth and therefore their IoUs are not being considered neither by the performance measure (i.e. Average Precision) nor by loss computation.

\begin{figure}
\centering
\includegraphics[width=\linewidth]{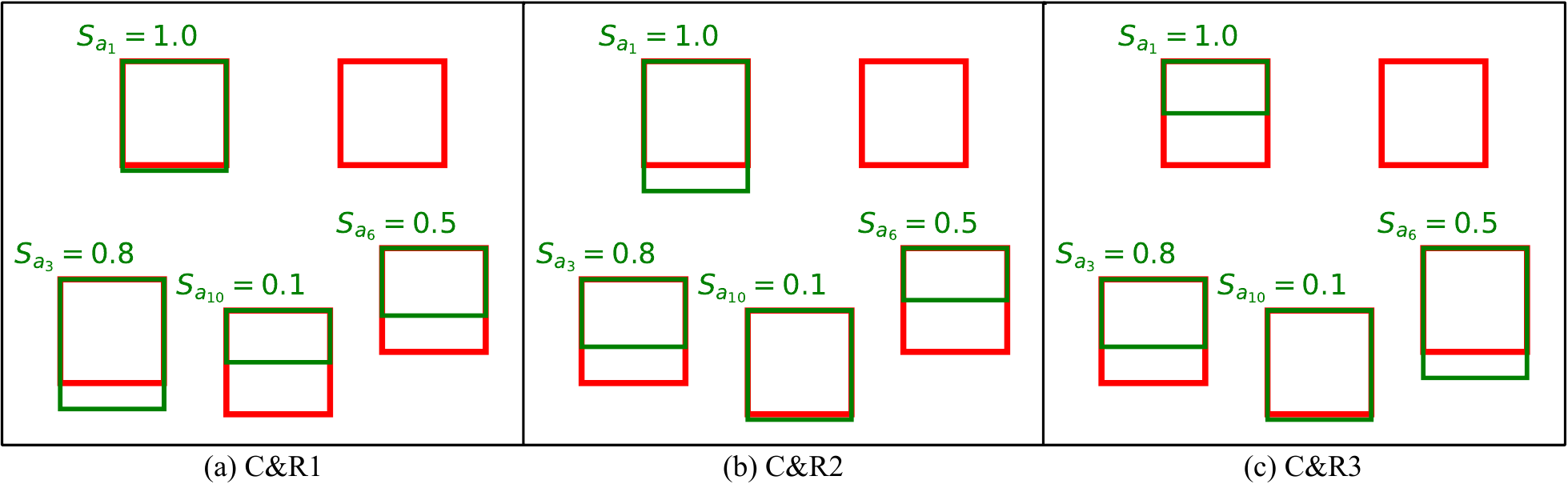}
\caption{Visualization of anchor boxes in the scenarios used in Figure 1. Green and red boxes are positive anchors and  ground truths respectively. $p_i$ refers to confidence score of the anchor $a_i$. Note that for all of the scenarios, there are additionally six false positives (see Figure 1), which are excluded in this figure for clarity.}
\label{fig:fig_regressors}
\end{figure}

\subsection{Performance Evaluation}
\label{subsec:performance_estimation}
There are different ways to calculate Average Precision (AP) and loss values. For example, in PASCAL \cite{PASCAL} and COCO  \cite{COCO} datasets, the recall domain is divided into 11 and 101 evenly spaced points, respectively, and the  precision values at these points are averaged to compute AP for a single IoU threshold.

Here, we present how Average Precision ($\mathrm{AP}$) is calculated in Figure 1(b). Similar to the widely adopted performance metric, COCO-style AP, we use $\mathrm{AP}_\mathrm{IoU}$ with different $\mathrm{IoU}$ thresholds. In order to keep things simple but provide the essence of the performance metric, we use four samples with $0.15$ increments (i.e. $\{0.50, 0.65, 0.80,0.95\}$) instead of ten samples with $0.05$ increments as done by original COCO-style AP. 

In order to compute a single average precision with an IoU as the conventional true positive labelling threshold, denoted by $\mathrm{AP}_\mathrm{IoU}$, the approaches use different methods for sampling/combining individual precision values on a PR curve. The PR curves corresponding to each detector-$\mathrm{AP}_\mathrm{IoU}$ pair are presented in Figure \ref{fig:pr_curves}. While drawing these curves, similar to Pascal-VOC and COCO, we also adopt interpolation on the PR curve, which requires keeping the larger precision value in the case that the larger one resides in a lower recall. Then, again similar to what these common methods do for a single AP threshold, we check the precision values on different recall values after splitting the recall axis equally. Here, we use 10 recall points between 0.1 and 1.0 in 0.1 increments. Then,  based on the PR curves in Figure \ref{fig:pr_curves}, we check the precision under these different recall values and present them in Table \ref{tab:perf_calculation}. Having generated these values in Table \ref{tab:perf_calculation} for each $\mathrm{AP}_\mathrm{IoU}$s, the computation is trivial: Just averaging over these precisions (i.e. row-wise average) yields $\mathrm{AP}_\mathrm{IoU}$s. Finally, averaging over these four $\mathrm{AP}_\mathrm{IoU}$s produces the final detection performance as $0.37$, $0.29$ and $0.20$ for $\mathrm{C\&R_1}$, $\mathrm{C\&R_2}$, $\mathrm{C\&R_3}$ respectively (see Table \ref{tab:perf_calculation}). 

\begin{figure}
\centering
\begin{tabular}{cccc}
\begin{subfigure}{0.20\textwidth}
\includegraphics[width=\linewidth]{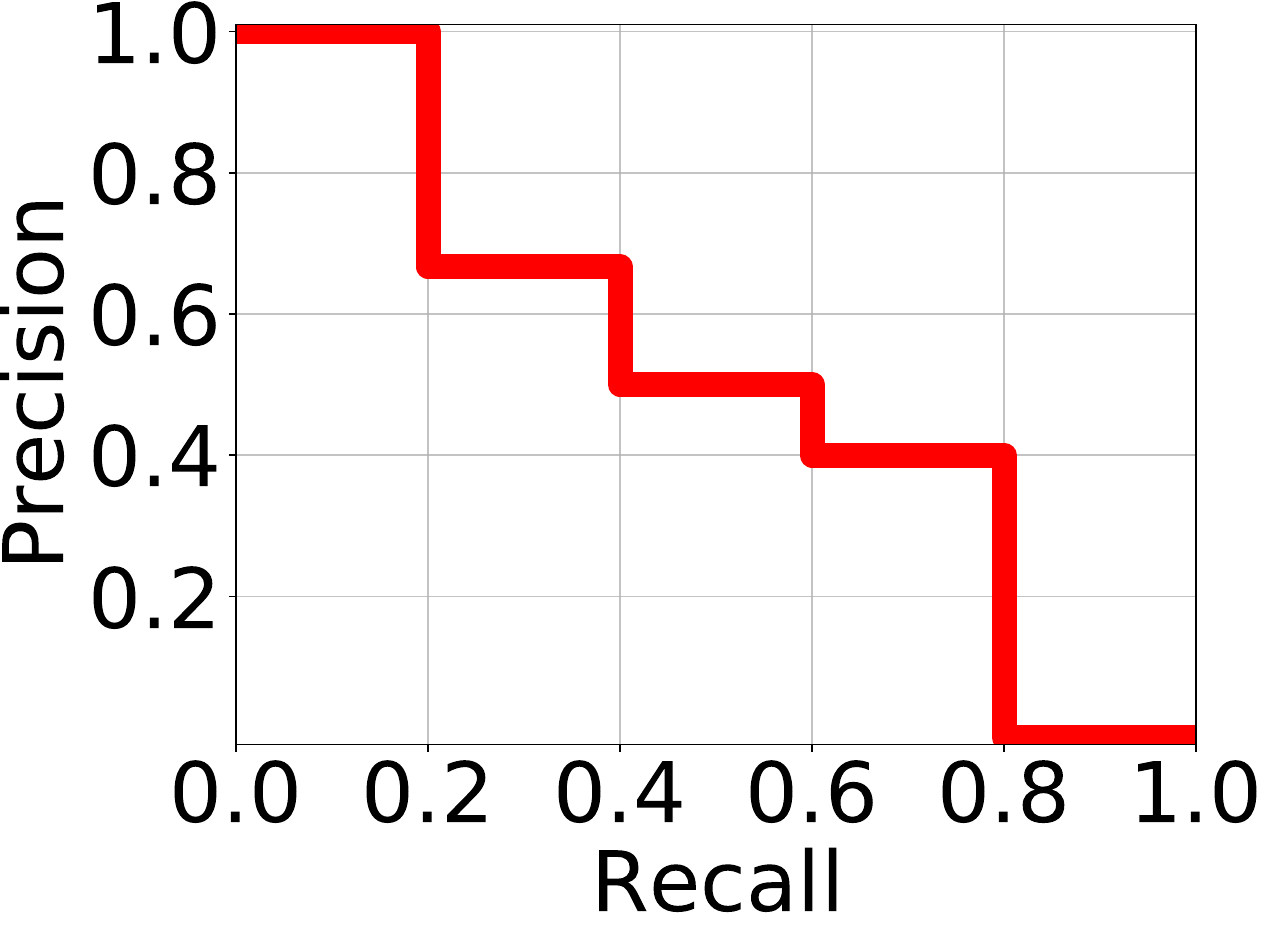}
\caption{C\&R1 $AP_{50}$}\label{fig:pr_1}
\end{subfigure}&
\begin{subfigure}{0.20\textwidth}
\includegraphics[width=\linewidth]{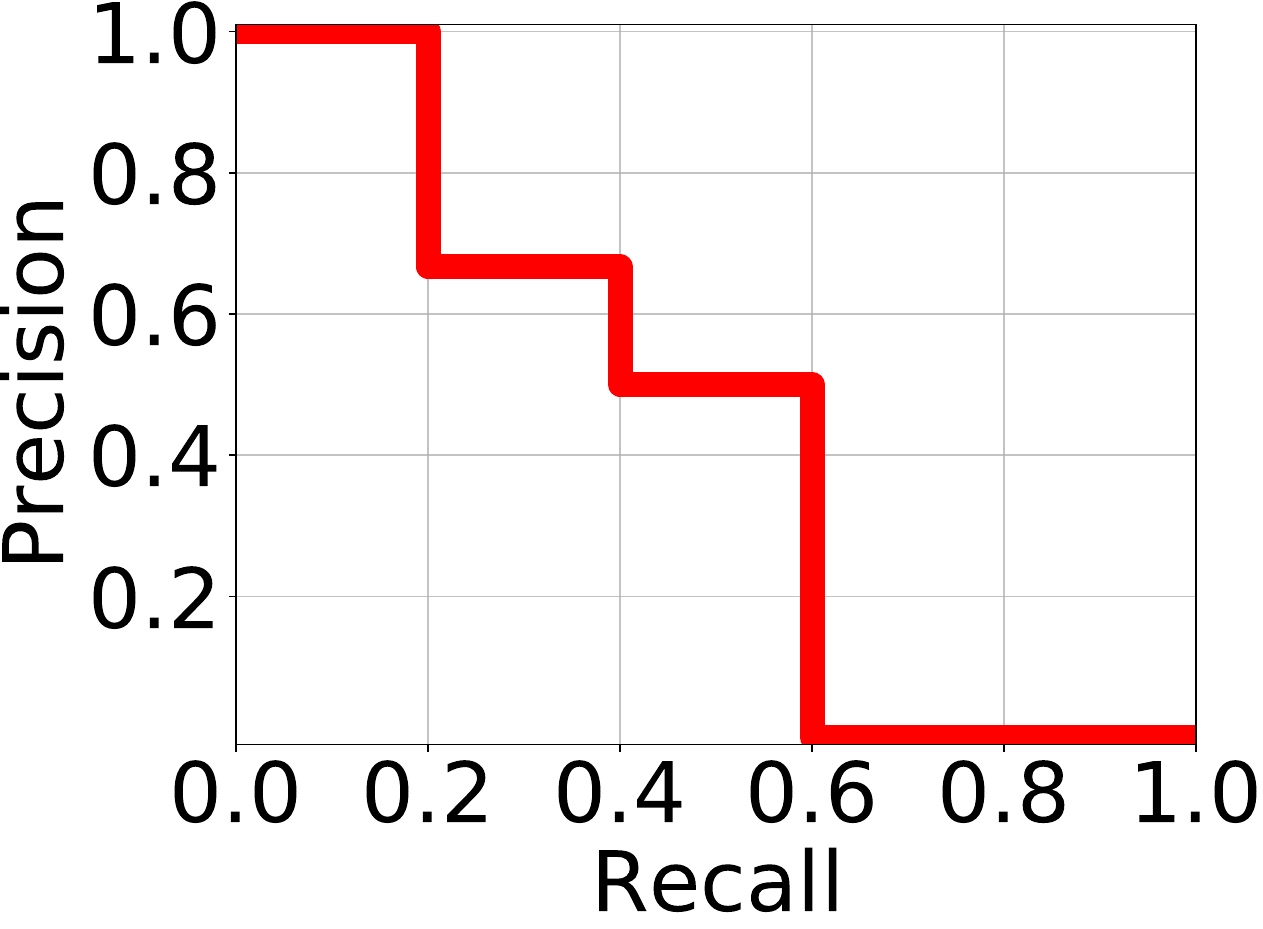}
\caption{C\&R1 $AP_{65}$}\label{fig:pr_2}
\end{subfigure}&
\begin{subfigure}{0.20\textwidth}
\includegraphics[width=\linewidth]{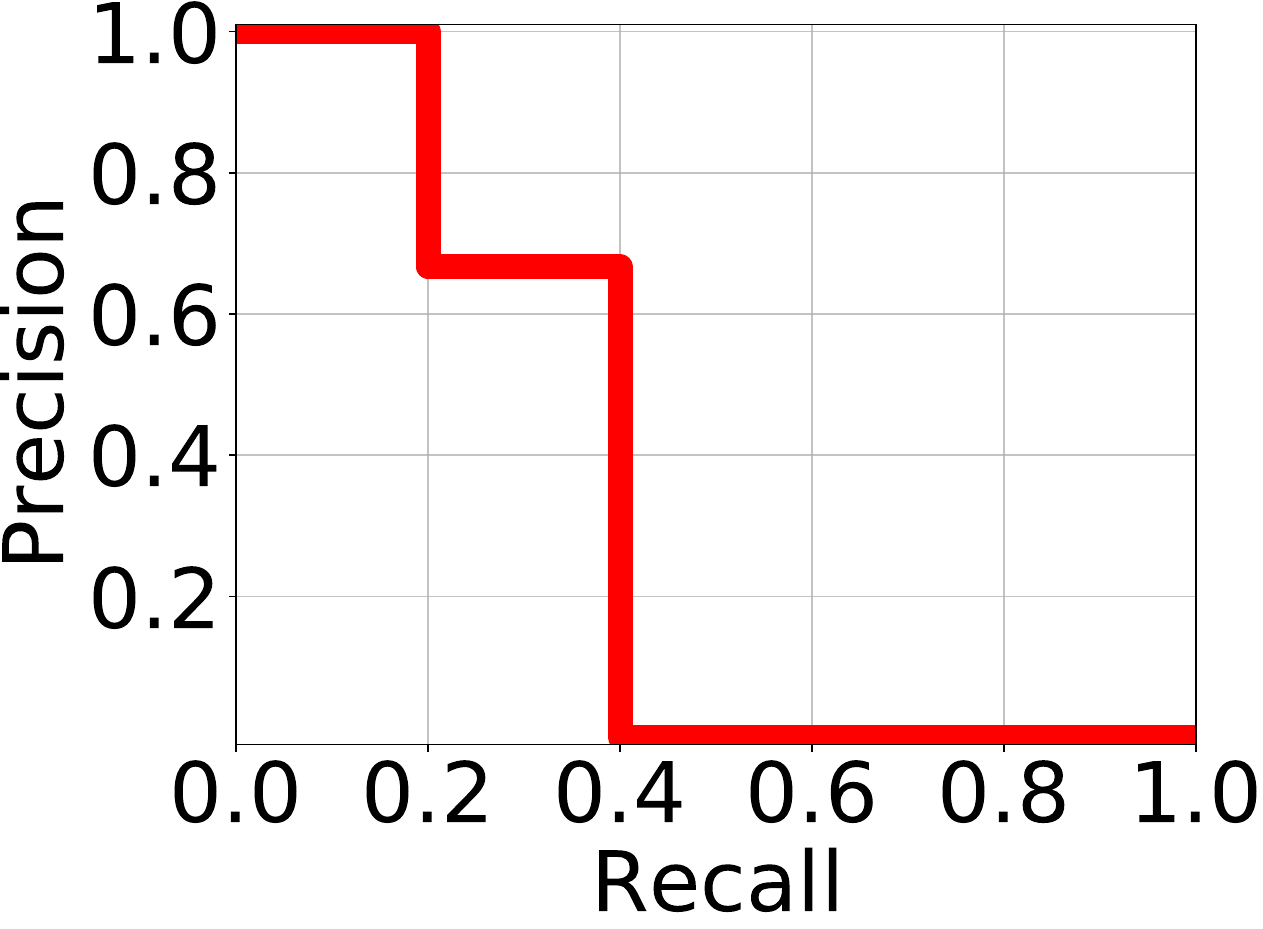}
\caption{C\&R1 $AP_{80}$}\label{fig:pr_3}
\end{subfigure}&
\begin{subfigure}{0.20\textwidth}
\includegraphics[width=\linewidth]{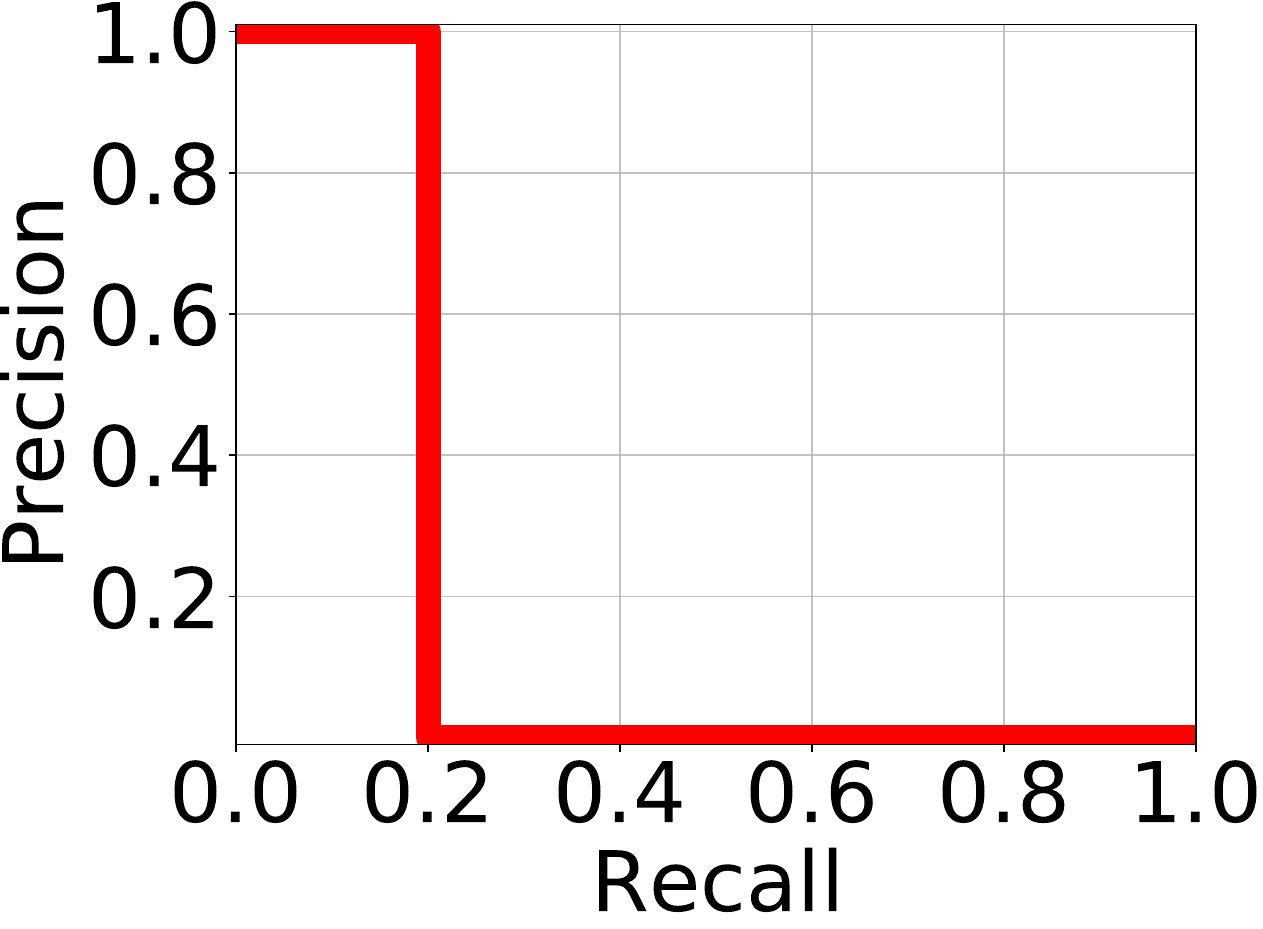}
\caption{C\&R1 $AP_{95}$}\label{fig:pr_4}
\end{subfigure}\\

\begin{subfigure}{0.20\textwidth}
\includegraphics[width=\linewidth]{figures_appendix/fig1_R1R2R3_AP50.pdf}
\caption{C\&R2 $AP_{50}$}\label{fig:pr_5}
\end{subfigure}&
\begin{subfigure}{0.20\textwidth}
\includegraphics[width=\linewidth]{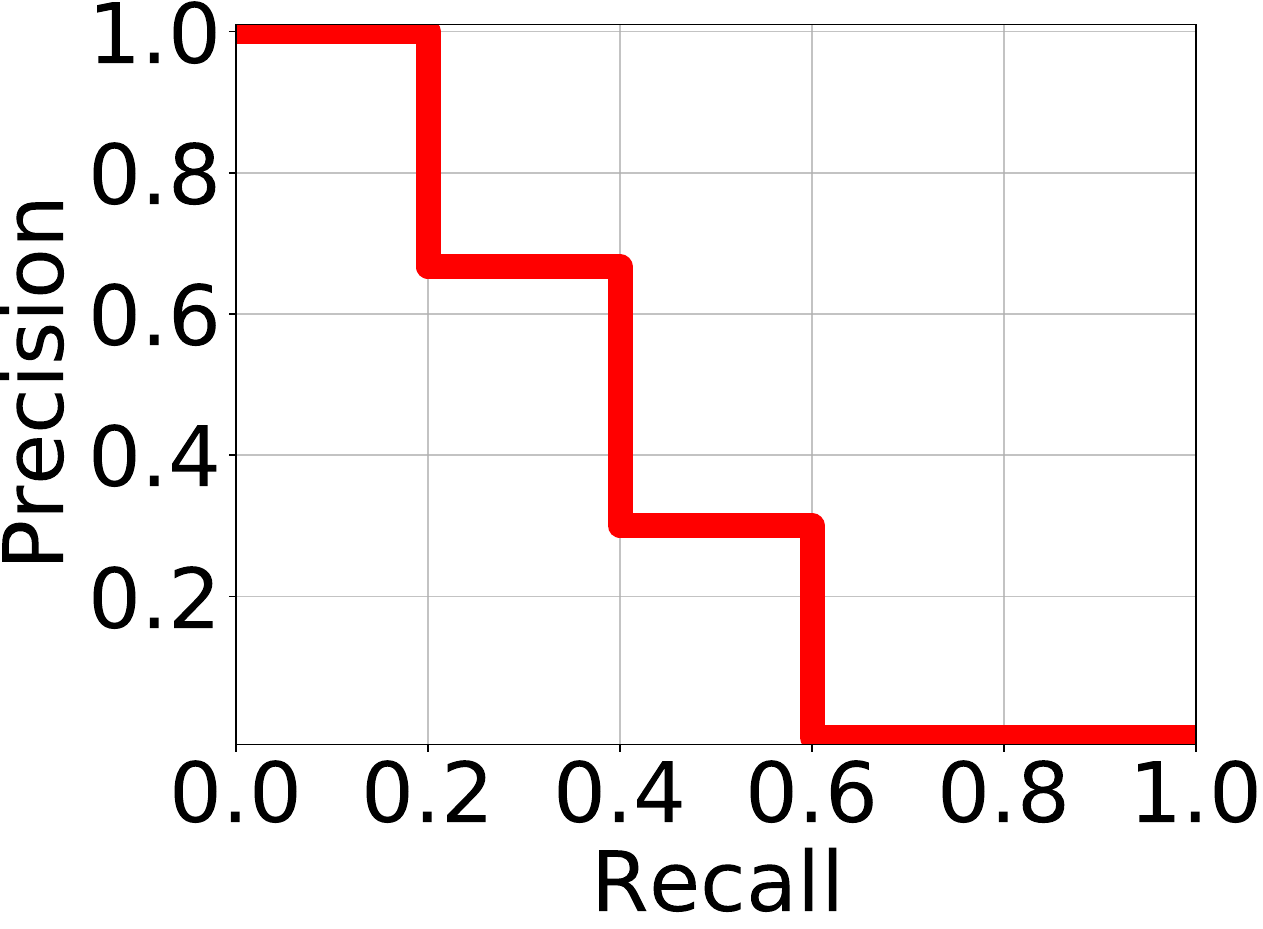}
\caption{C\&R2 $AP_{65}$}\label{fig:pr_6}
\end{subfigure}&
\begin{subfigure}{0.20\textwidth}
\includegraphics[width=\linewidth]{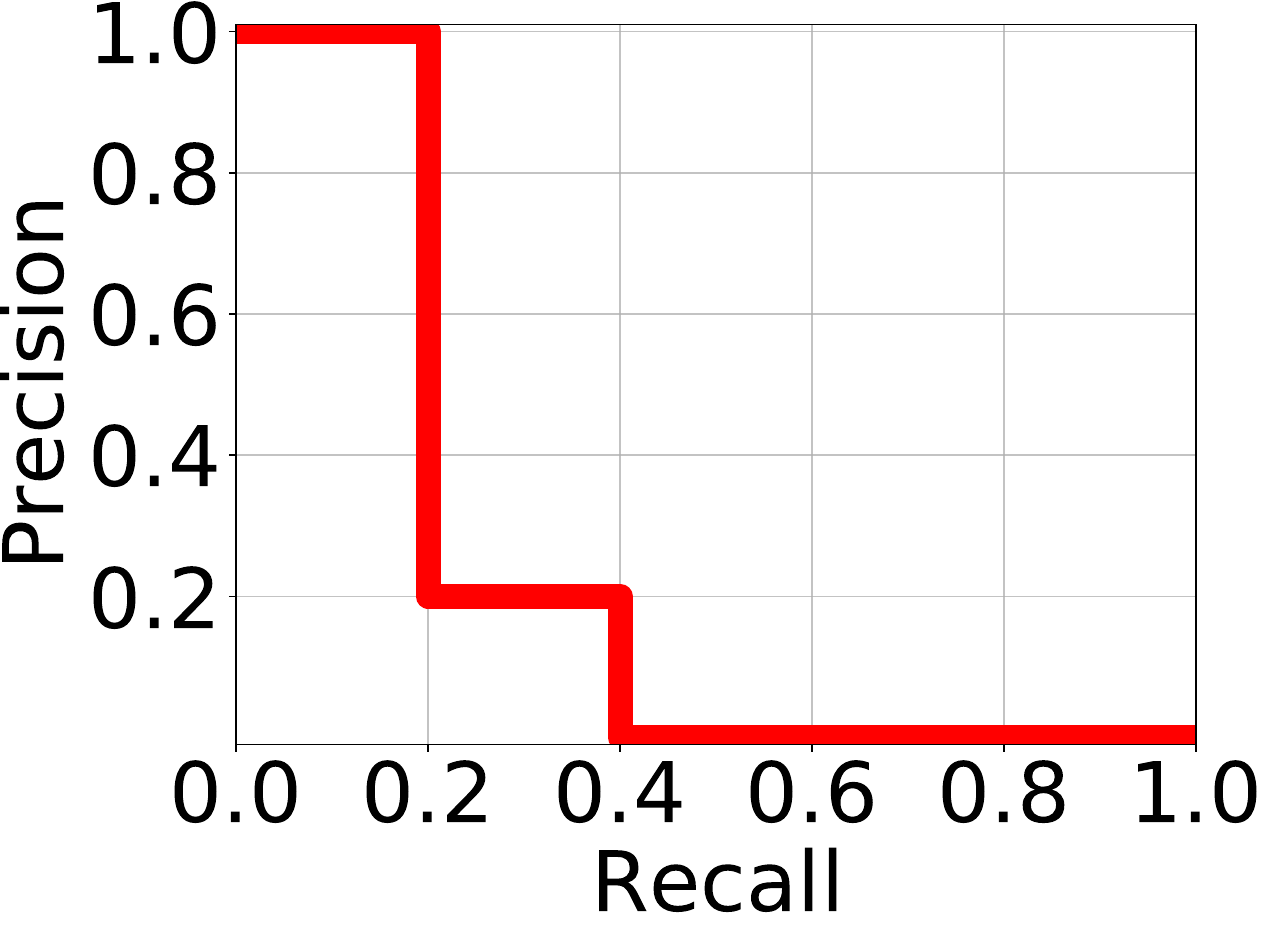}
\caption{C\&R2 $AP_{80}$}\label{fig:pr_7}
\end{subfigure}&
\begin{subfigure}{0.20\textwidth}
\includegraphics[width=\linewidth]{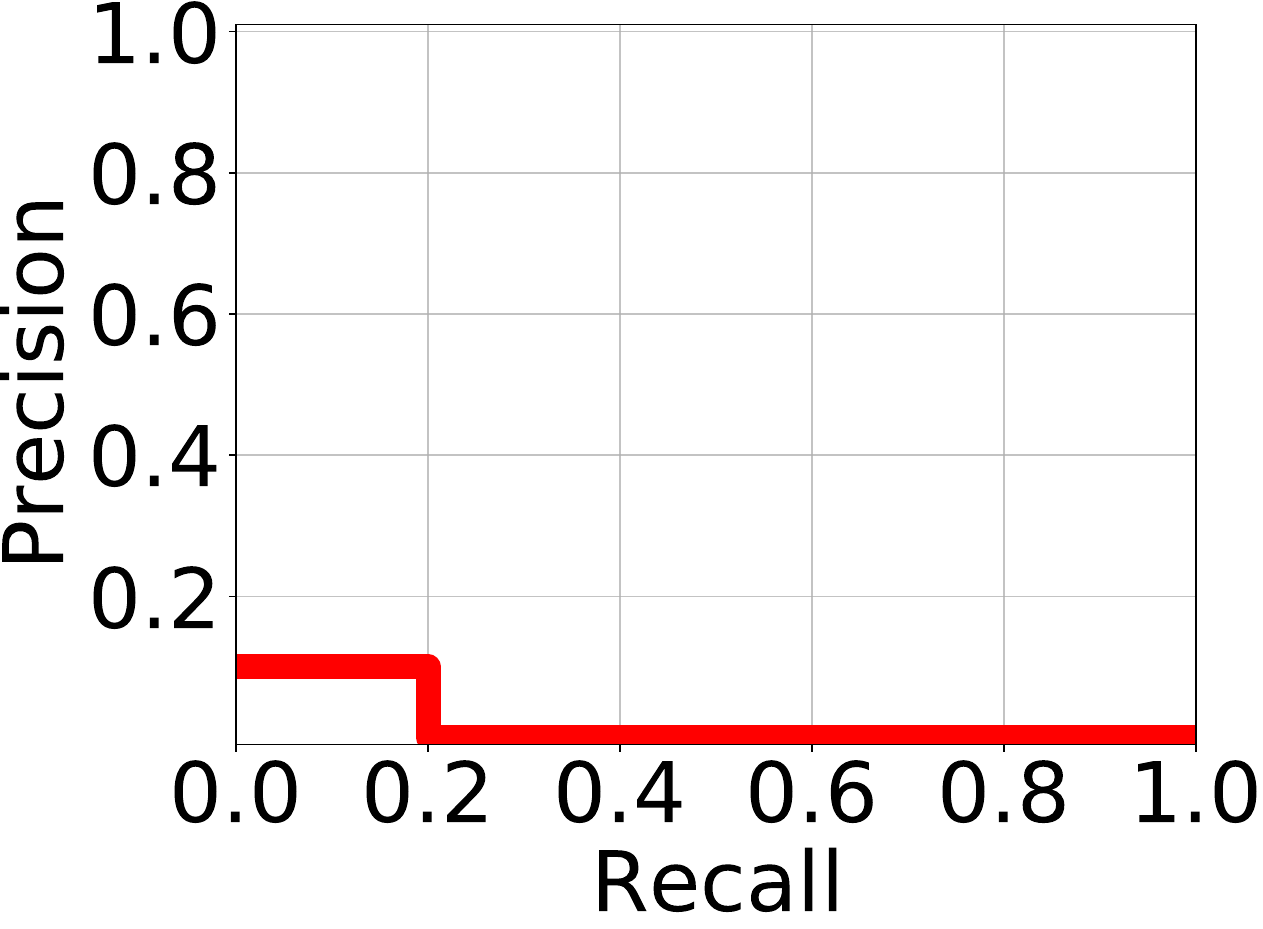}
\caption{C\&R2 $AP_{95}$}\label{fig:pr_8}
\end{subfigure}\\

\begin{subfigure}{0.20\textwidth}
\includegraphics[width=\linewidth]{figures_appendix/fig1_R1R2R3_AP50.pdf}
\caption{C\&R3 $AP_{50}$}\label{fig:pr_9}
\end{subfigure}&
\begin{subfigure}{0.20\textwidth}
\includegraphics[width=\linewidth]{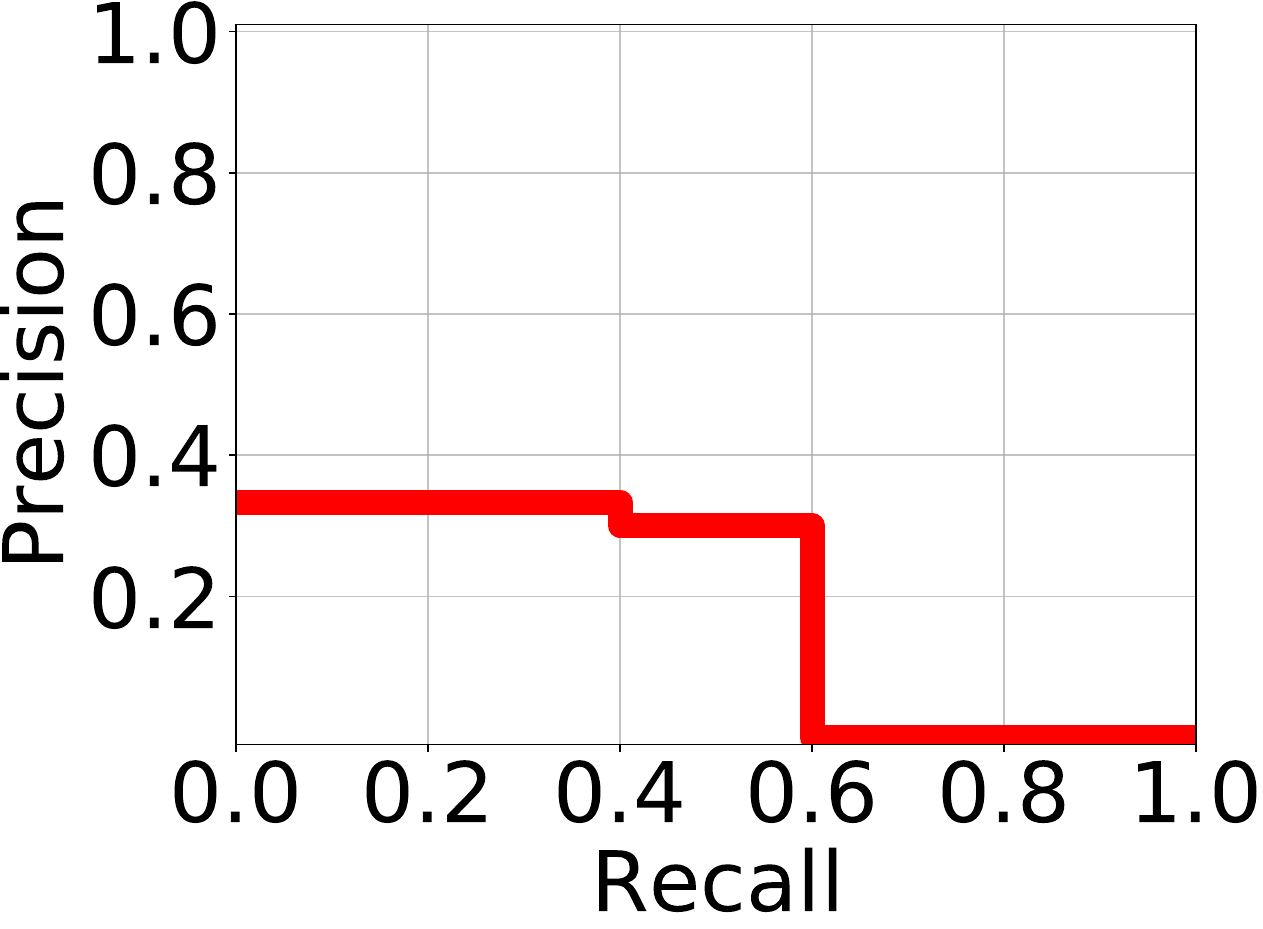}
\caption{C\&R3 $AP_{65}$}\label{fig:pr_10}
\end{subfigure}&
\begin{subfigure}{0.20\textwidth}
\includegraphics[width=\linewidth]{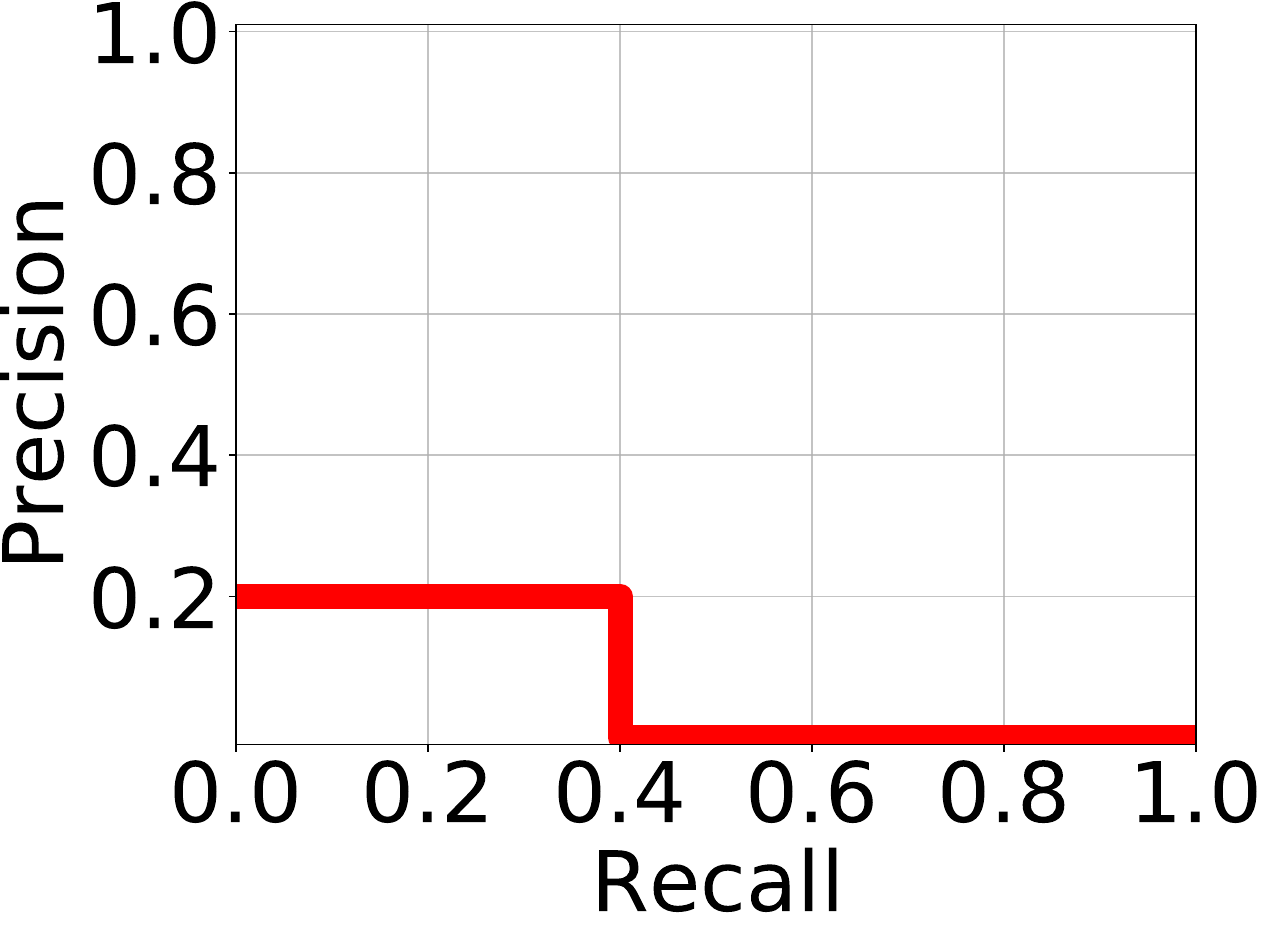}
\caption{C\&R3 $AP_{80}$}\label{fig:pr_11}
\end{subfigure}&
\begin{subfigure}{0.20\textwidth}
\includegraphics[width=\linewidth]{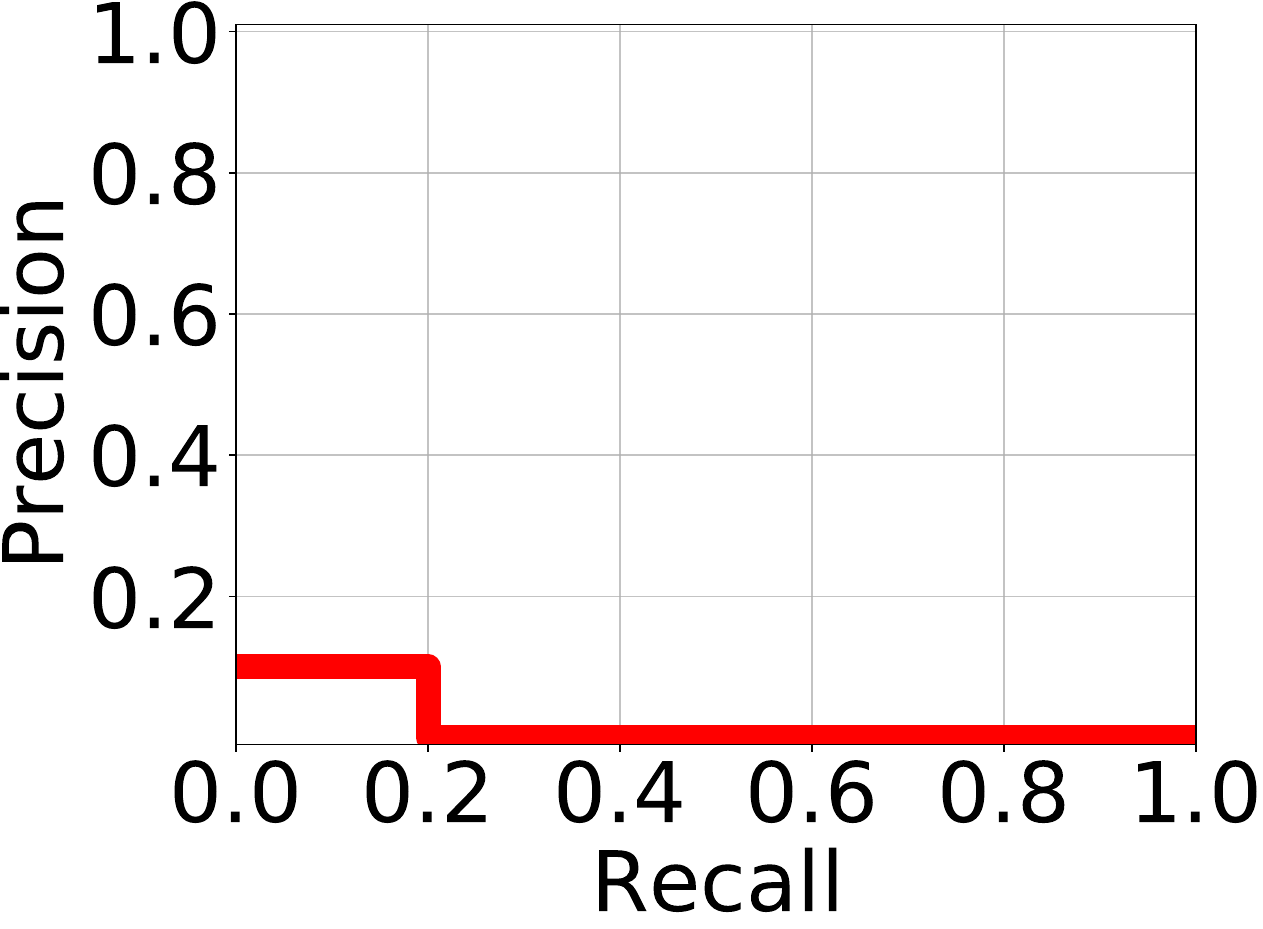}
\caption{C\&R3 $AP_{95}$}\label{fig:pr_12}
\end{subfigure}\\
\end{tabular}

\caption{PR curve of each detector output-$\mathrm{AP}_{IoU}$ pair. Rows and columns correspond to different $\mathrm{AP}_{IoU}$ and detector outputs respectively. PR curves are interpolated (see the text for more detail).}
\label{fig:pr_curves}
\end{figure}

\begin{table}[t]

    \centering
    \caption{Precision of each detector output-$\mathrm{AP}_{IoU}$ pair for evenly spaced recall values. This table is based on the PR curves presented in Fig. \ref{fig:pr_curves}.}
    \label{tab:perf_calculation}
    \setlength{\tabcolsep}{0.25em}
    \begin{tabular}{|c|c|c|c|c|c|c|c|c|c|c|c|c|} \hline
         \multirow{2}{*}{IoU}&\multirow{2}{*}{Output}&\multicolumn{10}{|c|}{Precisions for Different Recalls (R)}&\multirow{2}{*}{$\mathrm{AP_\mathrm{IoU}}$}\\ \cline{3-12}
         & &R=$0.1$&R=$0.2$&R=$0.3$&R=$0.4$&R=$0.5$&R=$0.6$&R=$0.7$&R=$0.8$&R=$0.9$&R=$1.0$& \\ \hline\hline
         \multirow{3}{*}{0.50}
         &$C\&R_1$&1.00&1.00&0.67&0.67&0.50&0.50&0.40&0.40&0.00&0.00&0.51\\
         &$C\&R_2$&1.00&1.00&0.67&0.67&0.50&0.50&0.40&0.40&0.00&0.00&0.51\\
         &$C\&R_3$&1.00&1.00&0.67&0.67&0.50&0.50&0.40&0.40&0.00&0.00&0.51\\
         \hline
         \multirow{3}{*}{0.65}
         &$C\&R_1$&1.00&1.00&0.67&0.67&0.50&0.50&0.00&0.00&0.00&0.00&0.43\\
         &$C\&R_2$&1.00&1.00&0.67&0.67&0.30&0.30&0.00&0.00&0.00&0.00&0.39\\
         &$C\&R_3$&0.33&0.33&0.33&0.33&0.30&0.30&0.00&0.00&0.00&0.00&0.19\\
         \hline
         \multirow{3}{*}{0.80}
         &$C\&R_1$&1.00&1.00&0.67&0.67&0.00&0.00&0.00&0.00&0.00&0.00&0.33\\
         &$C\&R_2$&1.00&1.00&0.20&0.20&0.00&0.00&0.00&0.00&0.00&0.00&0.24\\
         &$C\&R_3$&0.20&0.20&0.20&0.20&0.00&0.00&0.00&0.00&0.00&0.00&0.08\\
         \hline
         \multirow{3}{*}{0.95}
         &$C\&R_1$&1.00&1.00&0.00&0.00&0.00&0.00&0.00&0.00&0.00&0.00&0.20\\
         &$C\&R_2$&0.10&0.10&0.00&0.00&0.00&0.00&0.00&0.00&0.00&0.00&0.02\\
         &$C\&R_3$&0.10&0.10&0.00&0.00&0.00&0.00&0.00&0.00&0.00&0.00&0.02 \\
         \hline
    \end{tabular}
\end{table}

\subsection{Computing the Loss Values}
\label{subsec:loss_est}
In this section, computing the loss values in Figure 1(c) of the paper is presented in detail. Each section is devoted to a loss function presented in  Figure 1(c). To keep things simple, without loss of generality, we make the following assumptions in this section during the calculation of the classification and localisation losses:

\begin{enumerate}
    \item The classifier has sigmoid non-linearity at the top.
    \item There is only one foreground class.
    \item Similar to how localisation losses deal with scale- and translation-variance within an image, we assume that each ground truth box is normalized as $[0, 0, 1, 1]$.
    \item For each loss, the average of its contributors is reported.
\end{enumerate}

\subsubsection{Cross-entropy Loss}
Cross-entropy Loss of the $i$th example is defined as:
\begin{align}
    \label{eq:CrossEntropy}
    \mathcal{L}^{CE} (p_i) = - \mathcal{I}[i \in \mathcal{P}] \log (p_i) - \mathcal{I}[i \in \mathcal{N}] \log (1-p_i),
\end{align}
such that $p_i$ is the confidence score of the $i$th example obtained by applying the sigmoid activation to the classification logit $s_i$, and $\mathcal{I}[\mathrm{Q}]$ is the Iverson bracket which is $1$ if the predicate $\mathrm{Q}$ is true; or else it is $0$.

Seeing that all detector outputs, $C\&R1$, $C\&R2$ and $C\&R3$, involve the same classification output, we apply Eq. \eqref{eq:CrossEntropy} for each anchor on $C$, and then find their average as follows:
\begin{align}
    \mathcal{L}^{CE} &=\frac{1}{|\mathcal{P}|+|\mathcal{N}|} \sum_{p_i} \mathcal{L}^{CE}(p_i) ,\\
    &=-\frac{1}{10} \left( \log(1.00) + \log(1-0.90)+ \log(0.80) +\log(1-0.70) +\log(1-0.60)+ \log(0.50) \right. \\
    & \left. + \log(1-0.40)+\log(1-0.30)+\log(1-0.20) + \log(0.10) \right),\\
    &=0.87.
\end{align}

\subsubsection{Average precision (AP) Loss}
The computation of AP Loss is very similar to the $\mathrm{AP}_{\mathrm{50}}$ computation described in Section \ref{subsec:performance_estimation} except that precision is calculated on (and also averaged over) the positive examples instead of the recall values. With this intuition the precision values on the four positives are $1.00, 0.67, 0.50, 0.40$ respectively. Then, AP Loss for the output $C$ in Figure 1 regardless of the localisation output that it is combined with is:
\begin{align*}
    \mathcal{L}^{AP} = 1- \mathrm{\mathrm{AP}}_{50} &= 1 - \frac{1}{|\mathcal{P}|} \sum \limits_{i \in \mathcal{P}} \mathrm{precision}(i), \\
    &= 1 - \frac{1}{4}\times \left( 1.00+0.67+0.50+0.40 \right) = 0.36.
\end{align*}

\subsubsection{L1 Loss}
For a single ground truth, $\hat{B}_i=[\hat{x}_1, \hat{y}_1, \hat{x}_2, \hat{y}_2]$, and its corresponding detection, $B_i=[x_1, y_1, x_2, y_2]$, L1 Loss is defined simply by averaging over the L1 norm of the differences of the parameters of the detection boxes from their corresponding ground truths:
\begin{align}
    \label{eq:L1Loss}
    \mathcal{L}^{L1} (\hat{B}_i, B_i) =  \abs{\hat{x}_1-x_1}+\abs{\hat{y}_1-y_1}+\abs{\hat{x}_2-x_2}+\abs{\hat{y}_2-y_2},
\end{align}

Then, the average L1 Loss is:
\begin{align}
    \mathcal{L}^{L1} &=\frac{1}{|\mathcal{P}|} \sum_{i \in \mathcal{P}} \mathcal{L}^{L1} (\hat{B}_i, B_i), \\
    &=\frac{1}{4}\left( (0.00+0.00+0.00+0.05)+(0.00+0.00+0.00+0.25) \right.\\
    &\left. +(0.00+0.00+0.00+0.35)+(0.00+0.00+0.00+0.50)\right),\\
    &=0.29.
\end{align}

\subsubsection{IoU Loss}
For a single example, IoU Loss is simply $1-\mathrm{IoU}(\hat{B}_i, B_i)$. Then, for all three outputs in the scenario (also an instance is illustrated in Figure \ref{fig:fig_regressors}), seeing that the IoU distributions are all equal, the average IoU loss of this detection set is:
\begin{align*}
\mathcal{L}^{\mathrm{IoU}} &= \frac{1}{|\mathcal{P}|}\sum \limits_{i \in \mathcal{P}} 1-\mathrm{IoU}(\hat{B}_i, B_i), \\
&= \frac{1}{4} \left( (1-0.95)+ (1-0.80) + (1-0.65) + (1-0.50) \right) = 0.28.
\end{align*}

\subsubsection{aLRP Loss}
This section calculates aLRP Loss value in the scenario, and therefore we believe that at the same time it is also a toy example to present more insight on aLRP Loss.

 First, let us recall the definition of aLRP Loss from the paper to simplify tracking this section. aLRP Loss is defined as:
 \begin{align}
    \mathcal{L}^\mathrm{aLRP}:=\frac{1}{|\mathcal{P}|}\sum \limits_{i \in \mathcal{P}} \ell^{\mathrm{LRP}}(i),
\end{align}
such that
\begin{align}
\label{eq:LRPReformulation1}
    \ell^{\mathrm{LRP}}(i) 
    = \frac{1}{\mathrm{rank(i)}}
    \left(N_{FP}(i) + \mathcal{E}_{loc}(i) +  \sum \limits_{k \in \mathcal{P}, k \neq i}  \mathcal{E}_{loc}(k) H(x_{ik}) \right),
\end{align}
where $\mathcal{E}_{loc}(k) = (1-\mathrm{IoU}(k))/(1-\tau)$. Here, we take $H(x)$ as a step function instead of its approximation for simplicity.

Table \ref{tab:alrp-loss_calc} presents the computation of aLRP values including all by-products for each of the four positive anchors in $C \& R_1$, $C \& R_2$ and $C \& R_3$. Given the table presented in Figure 1(a) in the paper, we present how each column is derived in the following steps:
\begin{enumerate}
    \item $1-\mathrm{IoU(i)}$ is simply the IoU Loss of the positive anchors after prediction.
    \item $\mathcal{E}_{loc}(i) = (1-\mathrm{IoU}(i))/(1-\tau)$ such that $\tau=0.5$.
    \item Define a cumulative sum: $\mathrm{cumsum (\mathcal{E}_{loc})}(i) = \mathcal{E}_{loc}(i) +  \sum \limits_{k \in \mathcal{P}, k \neq i}  \mathcal{E}_{loc}(k) H(x_{ik})$ (see Eq. \ref{eq:LRPReformulation1}). Note that this simply corresponds to a cumulative sum on a positive example using the examples with larger scores and itself. Accordingly, in Table \ref{tab:alrp-loss_calc}, $\mathrm{cumsum (\mathcal{E}_{loc})}(i)$ is calculated by summing $\mathcal{E}_{loc})(i)$ column over anchors until (and including) $i$th example.
    \item $N_{FP}(i)$ is the number of negative examples with larger scores than the $i$th positive anchor. (See Section 3 for the formal definition.)
    \item $\mathrm{rank}(i)$ is the rank of an example within positives and negatives. (See Section 2 for the formal definition.)
    \item Then using $\mathrm{cumsum (\mathcal{E}_{loc})}(i)$, $N_{FP}(i)$ and $\mathrm{rank}(i)$, LRP error on a positive example can be computed as:
    \begin{align}
    \ell^{\mathrm{LRP}}(i) =  \frac{N_{FP}(i)+\mathrm{cumsum (\mathcal{E}_{loc})}(i)}{\mathrm{rank}(i)}.
    \end{align}
    \item In the rightmost column, aLRP Loss of a detector, $\mathcal{L}^{aLRP}$, is determined simply averaging over these single LRP values (i.e. $\ell^{\mathrm{LRP}}(i)$ ) on positives.
\end{enumerate}

\begin{table}[t]

    \centering
    \caption{Per-box calculation of $\mathcal{L}_{aLRP}$}
    \label{tab:alrp-loss_calc}
    \setlength{\tabcolsep}{0.25em}
    \begin{tabular}{|c|c|c|c|c|c|c|c|c|}
        \hline
         Output&Anchor&$\mathrm{1-IoU(i)}$&$\mathcal{E}_{loc}(i)$&  $\mathrm{cumsum (\mathcal{E}_{loc})}(i)$ &$N_{FP}(i)$&$\mathrm{rank(i)}$&$\ell^{\mathrm{LRP}}(i)$&$\mathcal{L}^{\mathrm{aLRP}}$ \\\hline
         \multirow{4}{*}{C\&R1}&$a_1$&$0.05$&$0.10$&$0.10$&$0.00$&$1.00$&$0.10$&\multirow{4}{*}{$0.53$}\\\cline{2-8}
         &$a_3$&$0.20$&$0.40$&$0.50$&$1.00$&$3.00$&$0.50$&\\\cline{2-8}
         &$a_6$&$0.35$&$0.70$&$1.20$&$3.00$&$6.00$&$0.70$&\\\cline{2-8}
         &$a_{10}$&$0.50$&$1.00$&$2.20$&$6.00$&$10.00$&$0.82$&\\\hhline{=========}
         \multirow{4}{*}{C\&R2}&$a_1$&$0.20$&$0.40$&$0.40$&$0.00$&$1.00$&$0.40$&\multirow{4}{*}{$0.69$}\\\cline{2-8}
         &$a_3$&$0.35$&$0.70$&$1.10$&$1.00$&$3.00$&$0.70$&\\\cline{2-8}
         &$a_6$&$0.50$&$1.00$&$2.10$&$3.00$&$6.00$&$0.85$&\\\cline{2-8}
         &$a_{10}$&$0.20$&$0.40$&$2.50$&$6.00$&$10.00$&$0.82$&\\\hhline{=========}
         \multirow{4}{*}{C\&R3}&$a_1$&$0.50$&$1.00$&$1.00$&$0.00$&$1.00$&$1.00$&\multirow{4}{*}{$0.89$}\\\cline{2-8}
         &$a_3$&$0.35$&$0.70$&$1.70$&$1.00$&$3.00$&$0.90$&\\\cline{2-8}
         &$a_6$&$0.20$&$0.40$&$2.10$&$3.00$&$6.00$&$0.85$&\\\cline{2-8}
         &$a_{10}$&$0.05$&$0.10$&$2.20$&$6.00$&$10.00$&$0.82$&\\\hline
    \end{tabular}
\end{table}

\section{Details of Table 1: Hyperparameters of the Loss Functions and Models}
\label{sec:Htperparameters}
This section presents the hyperparameters of the common loss functions in object detection and how they are combined by different models in Table 1.

\subsection{Hyperparameters of the Individual Loss Functions}
Table \ref{tab:Hyperparameters2} presents common loss functions and their hyperparameters. Note that since any change in these hyperparameter change the value of the loss function and affects its contribution to the multi-task learning nature of object detection, and, therefore $w_r$ also needs to be retuned. 

\begin{table}[t]
    \centering
    \caption{Common loss functions and the hyperparameters in their definitions.}
    \label{tab:Hyperparameters2}
    \begin{tabular}{|c|l|c|c|l|}\hline
         &Loss Function&Type & \multicolumn{2}{c|}{Number \& Usage of the Hyper-parameters\quad}\\ \hline \hline
         \multirow{9}{*}{\rotatebox[origin=c]{90}{$\mathcal{L}_c$ }}
         
         
         &Cross-entropy \cite{SSD,FasterRCNN}& Score-based&0& Sampling methods are required \\ \cline{2-5}
         
         &$\alpha$-bal. Cross-entropy\cite{FocalLoss}&Score-based&1& The weight of the foreground anchors\\ \cline{2-5}

        & \multirow{2}{*}{Focal Loss \cite{FocalLoss}}& \multirow{2}{*}{Score-based}& \multirow{2}{*}{2} & The weight of the foreground anchors\\ \cline{5-5}
         & & & & Modulating factor for hard examples \\ \cline{2-5}  

         &AP Loss \cite{APLoss}&Ranking-based&1&Smoothness of the step function   \\  \cline{2-5}

        & \multirow{3}{*}{DR Loss \cite{DRLoss}}& \multirow{3}{*}{Ranking-based}& \multirow{3}{*}{3} &Regularizer for foreground distribution\\ \cline{5-5}
         & & & &Regularizer for background distribution \\ \cline{5-5}  
         & & & &Smoothness of the loss \\ \hline  
         
         \multirow{4}{*}{\rotatebox[origin=c]{90}{$\mathcal{L}_r$ }}&Smooth $L_1$ \cite{FastRCNN}&$l_p$-based&1&Cut-off from $L_1$ loss to $L_2$ loss\\   \cline{2-5}

        & \multirow{2}{*}{Balanced $L_1$ \cite{LibraRCNN}}& \multirow{2}{*}{$l_p$-based}& \multirow{2}{*}{2} & The weight of the inlier anchors\\ \cline{5-5}
         & & & & Upper bound of the loss value \\ \cline{2-5} 
      
         &IoU Loss \cite{GIoULoss}&IoU-based&0&- \\  \hline
    \end{tabular}
\end{table}

\subsection{Hyperparameters of the Loss Functions of the Models}
This section discusses the loss functions of the methods discussed in Table 1 in the paper. Obviously, AP Loss \cite{APLoss}, Focal Loss \cite{FocalLoss} and DR Loss \cite{DRLoss} follow the formulation in Equation 1. Hence using Table \ref{tab:Hyperparameters2}, their total number of hyperparameters is easy to see. For example, DR Loss with three hyper-parameters is combined with Smooth L1, which has one hyperparameter. Including the weight of the localisation component, five hyper-parameters are required to be tuned.

Other architectures in Table 1 use more than two loss functions in order to learn different aspects to improve the performance:
\begin{itemize}
    \item FCOS \cite{FCOS} includes an additional centerness branch to predict the centerness of the pixels, which is trained by an additional cross entropy loss.
    \item FreeAnchor \cite{FreeAnchor} aims simultaneously to learn the assignment of the anchors to the ground truths by modeling the loss function based on maximum likelihood estimation. In Table 1, one can easily identify six hyper-parameters from the loss formulation of the Free Anchor and exploiting Table \ref{tab:Hyperparameters2}. Moreover, the inputs of the focal loss are subject to a saturated linear function with two hyperparameters, which makes eight in total.
    \item  A different set of approaches, an example of which is Faster R-CNN \cite{FasterRCNN}, uses directly cross entropy loss. However, cross entropy loss requires to be accompanied by a sampling method by which a set of positive and negative examples are sampled from the set of labelled anchors to alleviate the significant class imbalance. Even for random sampler, two of the following needs to be tuned in order to ensure stable training: (i) Number of positive examples (ii) Number of negative examples (iii) The rate between positives and negatives. Moreover, for a two-stage detector, these should be tuned for both stages, which brings about additional four hyper-parameters. That`s why Faster R-CNN \cite{FasterRCNN} in Table 1 requires nine hyperparameters.
    \item  Finally, CenterNet \cite{CenterNet}, as a state-of-the-art bottom-up method, has a loss function with several components while learning to predict the centers and the corners. It combines six individual losses, one of which is Hinge Loss with one hyperparameter. Considering the type of each, the loss function of CenterNet \cite{CenterNet} has 10 hyper-parameters in total.
\end{itemize}

\section{Proofs of Theorem 1 and Theorem 2}
\label{sec:Proofs}
This section presents the proofs for the theorems presented in our paper.

\begin{customthm}{1}
$\mathcal{L}= \frac{1}{Z}\sum \limits_{i \in \mathcal{P} } \ell(i) =
\frac{1}{Z}\sum \limits_{i \in \mathcal{P} }\sum \limits_{j \in \mathcal{N} }  L_{ij}$.
\end{customthm}
\begin{proof}
The ranking function is defined as:
\begin{align}
    \label{eq:RankingDefinition}
    \mathcal{L}=\frac{1}{Z}\sum_{i \in \mathcal{P}} \ell(i).
\end{align}
Since $\forall i \sum \limits_{j \in \mathcal{N}} p(j|i) = 1$, we can rewrite the definition as follows:
\begin{align}
    \frac{1}{Z} \sum \limits_{i \in \mathcal{P}} \ell(i) \left( \sum \limits_{j \in  \mathcal{N} } p(j|i)\right).
\end{align}
Reorganizing the terms concludes the proof as follows:
\begin{align}
    \frac{1}{Z} \sum \limits_{i \in \mathcal{P} }\sum \limits_{j \in \mathcal{N} } \ell(i)  p(j|i)    
    = \frac{1}{Z} \sum \limits_{i \in \mathcal{P} }\sum \limits_{j \in \mathcal{N} } L_{ij} .
\end{align}
\end{proof}

\begin{customthm}{2}
Training is balanced between positive and negative examples at each iteration; i.e. the summed gradient magnitudes of positives and negatives are equal:
\begin{align}
\sum \limits_{i \in \mathcal{P}} \abs{\frac{\partial \mathcal{L}}{\partial s_i}} = \sum \limits_{i \in \mathcal{N}} \abs{\frac{\partial \mathcal{L}}{\partial s_i}}.
\end{align}
\end{customthm}
\begin{proof}
The gradients of a ranking-based loss function are derived as (see Algorithm 1 and Equation 5 in the paper):
\begin{align}
    \label{eq:APGradients2}
    \frac{\partial \mathcal{L}}{\partial s_i} 
    = \frac{1}{Z} \left( \sum \limits_{j} \Delta x_{ij} - \sum \limits_{j} \Delta x_{ji} \right) = \frac{1}{Z} \sum \limits_{j} \Delta x_{ij} - \frac{1}{Z} \sum \limits_{j} \Delta x_{ji},
\end{align}
such that $\Delta x_{ij}$ is the update for $x_{ij}$s and defined as $\Delta x_{ij}=L^*_{ij}-L_{ij}$. Note that both $L_{ij}$ and $L^*_{ij}$ can be non-zero only if $i \in \mathcal{P}$ and $j \in \mathcal{N}$ following the definition of the primary term. Hence, the same applies to $\Delta x_{ij}$: if $i \notin \mathcal{P}$ or $j \notin \mathcal{N}$, then $\Delta x_{ij}=0$. Then using these facts, we can state in Eq. \eqref{eq:APGradients2} that if $i \in \mathcal{P}$, then $\sum \limits_{j} \Delta x_{ji}=0$; and if $i \in \mathcal{N}$, then $\sum \limits_{j} \Delta x_{ij}=0$. Then, we can say that, only one of the terms is active in Eq. \eqref{eq:APGradients2} for positives and negatives:
\begin{align}
    \label{eq:APGradientsActive}
    \frac{\partial \mathcal{L}}{\partial s_i} 
    = \underbrace{\frac{1}{Z} \sum \limits_{j} \Delta x_{ij}}_{\text{Active if $i\in \mathcal{P}$}}  - \underbrace{ \frac{1}{Z} \sum \limits_{j} \Delta x_{ji}}_{\text{Active if $i\in \mathcal{N}$}}. 
\end{align}

Considering that the value of a primary term cannot be less than its target, we have $\Delta x_{ij} \leq 0$, which implies  $\frac{\partial \mathcal{L}}{\partial s_i} \leq 0$. So, we can take the absolute value outside of summation:
\begin{align}
\label{eq:GradientsToUpdates2}
\sum \limits_{i \in \mathcal{P}} \abs{\frac{\partial \mathcal{L}}{\partial s_i}} =\abs{ \sum \limits_{i \in \mathcal{P}} \frac{\partial \mathcal{L}}{\partial s_i}},
\end{align}

and using the fact identified in Eq. \eqref{eq:APGradientsActive} (i.e. for $i \in \mathcal{P}$, $\frac{\partial \mathcal{L}}{\partial s_i} =  \frac{1}{Z} \sum \limits_{j \in \mathcal{N}} \Delta x_{ij}$):
\begin{align}
\label{eq:GradientsToUpdates}
\abs{ \sum \limits_{i \in \mathcal{P}} \frac{1}{Z} \sum \limits_{j \in \mathcal{N}} \Delta x_{ij}} =  \abs{\frac{1}{Z} \sum \limits_{i \in \mathcal{P}} \sum \limits_{j \in \mathcal{N}} \Delta x_{ij}} .
\end{align}

Simply interchanging the indices and the order of summations, and then reorganizing the constant $\frac{1}{Z}$ respectively yields:
\begin{align}
\abs{ \frac{1}{Z} \sum \limits_{j \in \mathcal{P}} \sum \limits_{i \in \mathcal{N}} \Delta x_{ji}}
= \abs{  \frac{1}{Z} \sum \limits_{i \in \mathcal{N}} \sum \limits_{j \in \mathcal{P}} \Delta x_{ji}} 
= \abs{  \sum \limits_{i \in \mathcal{N}} \frac{1}{Z} \sum \limits_{j \in \mathcal{P}} \Delta x_{ji}}.
\end{align}

Note that for $i \in \mathcal{N}$, $\frac{\partial \mathcal{L}}{\partial s_i} = - \frac{1}{Z} \sum \limits_{j \in \mathcal{P}} \Delta x_{ji}$, and hence $ \frac{1}{Z} \sum \limits_{j \in \mathcal{P}} \Delta x_{ji} = - \frac{\partial \mathcal{L}}{\partial s_i} $. Replacing $ \frac{1}{Z} \sum \limits_{j \in \mathcal{P}} \Delta x_{ji}$;
\begin{align}
\abs{  \sum \limits_{i \in \mathcal{N}} - \frac{\partial \mathcal{L}}{\partial s_i}}.
\end{align}

Since, for $i \in \mathcal{N}$, $\frac{\partial \mathcal{L}}{\partial s_i} = - \frac{1}{Z} \sum \limits_{j \in \mathcal{P}} \Delta x_{ji}$ is greater or equal to zero, the proof follows:
\begin{align}
\abs{  \sum \limits_{i \in \mathcal{N}} - \frac{\partial \mathcal{L}}{\partial s_i}}
=\abs{  \sum \limits_{i \in \mathcal{N}} \frac{\partial \mathcal{L}}{\partial s_i}}
=\sum \limits_{i \in \mathcal{N}}\abs{ \frac{\partial \mathcal{L}}{\partial s_i}}.
\end{align}

\end{proof}

\comment{
\section{Example Gradient Derivations}
\label{sec:GradDerivations}
Here, we first show that Average Precision (AP) Loss and Normalized Discounted Cumulative Gain (NDCG) Loss follow our ranking function, that is,
\begin{align}
    \label{eq:RankingDefinition}
    \mathcal{R}(\mathcal{D}) = \frac{1}{Z} \sum \limits_{i \in \mathcal{C}}  \mathcal{E}^\mathcal{R} (i),
\end{align}{}
and then derive their gradients. Compared to the derivation by Chen et al.\cite{APLoss}, we show that obtaining the gradients of AP Loss is very easy using our formulation, and on NDCG Loss we show that the update rules of loss functions from other domains can be derived following our well-defined steps. 

\subsection{Average Precision (AP) Loss and Its Gradients}
In the following we define and derive the gradients of AP Loss following our methodology:
\textbf{Definition:} AP Loss is defined as:
\begin{align}
\mathcal{L}^{AP} = 1- \mathrm{AP}_{50} = 1 - \frac{1}{|\mathcal{P}|} \sum \limits_{i \in \mathcal{P}} \mathrm{precision}(i) = \frac{1}{|\mathcal{P}|} \sum \limits_{i \in \mathcal{P}} (1-\mathrm{precision}(i)), 
\end{align}
Then, $\mathcal{A}=\mathcal{P}$, $\mathcal{B}=\mathcal{N}$, $Z=|\mathcal{P}|$ and $\mathcal{E}^{AP} (i) = (1-\mathrm{precision}(i))$ in Eq. \ref{eq:RankingDefinition}. Hence, it follows our ranking-based function formulation and we can use our two step formulation as follows:

\textbf{Step 1: Identifying Primary Terms:} Setting $p^{\mathcal{R}}(i, j) = \frac{H(x_{ij})}{N_{FP}(i)}$, the primary terms of AP Loss can be defined as:
\begin{align}
    \label{eq:APGeneralPrimaryTermDefinition}
    \Psi^{AP}_{ij} = (1-\mathrm{precision}(i)) \times \frac{H(x_{ij})}{N_{FP}(i)} = \frac{N_{FP}(i) }{\mathrm{rank}(i)} \times \frac{H(x_{ij})}{N_{FP}(i)}= \frac{H(x_{ij})}{\mathrm{rank}(i)}.
\end{align}

\textbf{Step 2: Identifying Gradients:} When $i$th positive is ranked accordingly, ${\Psi^{AP}_{ij}}^* = 0$ for AP loss. Hence, $G_{ij} ^{AP}$ is:
\begin{align}
    \label{eq:APClassificationGrads}
    G_{ij} ^{AP} = \Psi^{AP}_{ij} - {\Psi^{AP}_{ij}}^* 
    = \frac{H(x_{ij})}{\mathrm{rank}(i)} - 0 = \frac{H(x_{ij})}{\mathrm{rank}(i)}.
\end{align}

\subsection{Classification: Average F1 Loss and Its Gradients}
\textbf{Definition:} The AF1 loss is defined as 

\begin{align}
\frac{1}{|\mathcal{P}|} \sum \limits_{i \in \mathcal{P}} (1-\mathrm{F1}(i)), 
\end{align}
where $\mathcal{A}=\mathcal{P}$, $\mathcal{B}=\mathcal{N}$, $Z=|\mathcal{P}|$ and $\mathcal{E}^{AF1} (i) = (1-\mathrm{F1}(i))$. Hence, it follows our ranking-based function formulation. Note that f1 score of an example can be defined as:

\begin{align}
    \mathrm{F1}(i) = \frac{2 \mathrm{precision}(i) \mathrm{recall}(i)}{\mathrm{precision}(i) + \mathrm{recall}(i)},
\end{align}
such that
\begin{align}
    \mathrm{precision}(i) = \frac{N_{TP}(i) }{\mathrm{rank}(i)} \text{, and } \mathrm{recall}(i) = \frac{N_{TP}(i) }{|X|},
\end{align}
$|X|$ is the number of examples with the same class of $i$ in the batch.

\textbf{Step 1: Identifying Primary Terms:} Setting $p^{\mathcal{R}}(i, j) = \frac{H(x_{ij})}{N_{FP}(i)}$, the primary terms of the AF1 Loss is:
\begin{align}
    \label{eq:AF1GeneralPrimaryTermDefinition}
    \Psi^{AF1}_{ij} = (1-\mathrm{F1}(i)) \times \frac{H(x_{ij})}{N_{FP}(i)}.
\end{align}

\textbf{Step 2: Identifying Gradients:} When $i$th positive is ranked accordingly, $G(i) = \frac{1}{\log_2(1+1)}=1$ and the local error is $\mathcal{E}^{AF1} (i) =0 $. Then, ${\Psi^{AF1}_{ij}}^* =0$. Hence, $G_{ij} ^{AF1}$ is:
\begin{align}
    \label{eq:AF1ClassificationGrads}
    G_{ij} ^{AF1} &= \Psi^{AF1}_{ij} - {\Psi^{AF1}_{ij}}^* 
    = (1-\mathrm{F1}(i)) \times \frac{H(x_{ij})}{N_{FP}(i)}.
\end{align}
}

\section{Normalized Discounted Cumulative Gain (NDCG) Loss and Its Gradients: Another Case Example for our Generalized Framework}
In the following we define and derive the gradients of the NDCG Loss \cite{OptimizingUpperBound} following our generalized framework presented in Section 3 of our main paper.

The NDCG loss is defined as:
\begin{align}
\label{eq:NDCGDef}
    \mathcal{L}^{\mathrm{NDCG}} =1 - \frac{1}{G_{max}} \sum \limits_{i \in \mathcal{P}}G(i) = \frac{G_{max}-\sum \limits_{i \in \mathcal{P}}G(i)}{G_{max}} =  \sum \limits_{i \in \mathcal{P}} \frac{G_{max}/|\mathcal{P}| - G(i)}{G_{max}}.
\end{align}
Note that different from AP Loss and aLRP Loss, here $Z$ turns out to be $1$, which makes sense since NDCG is normalized by definition. Also, based on Eq. \ref{eq:NDCGDef}, one can identify NDCG Error on a positive as: $\ell^{\mathrm{NDCG}}(i) = \frac{G_{max}/|\mathcal{P}| - G(i)}{G_{max}}$ such that $G(i) = \frac{1}{\log_2(1+\mathrm{rank}(i))}$ and  $G_{max} = \sum \limits_{i =1}^{ |\mathcal{P}|}\log_2(1+i)$.  

Similar to AP and aLRP Loss, using $p(j|i) = \frac{H(x_{ij})}{N_{FP}(i)}$, the primary term of the NDCG Loss is $L^\mathrm{NDCG}_{ij} = \ell^\mathrm{NDCG} (i) p(j|i)$ (line 1 of Algorithm 1 in the paper). When the positive example $i$ is ranked properly, $G(i) = \frac{1}{\log_2(1+1)}=1$, and resulting desired NDCG Error is (line 2 of Algorithm 1):
\begin{align}
{\ell^\mathrm{NDCG}(i)}^* =\frac{G_{max}/|\mathcal{P}|-1}{G_{max}},
\end{align}
yielding a target primary term ${L^\mathrm{NDCG}_{ij}}^*={\ell^\mathrm{NDCG}_{i}}^* p(j|i)$. Using ${L^\mathrm{NDCG}_{ij}}$ and ${L^\mathrm{NDCG}_{ij}}^*$, the update can be calculated as follows (line 3 of Algorithm 1):
\begin{align}
    \label{eq:NDCGClassificationGrads}
     \Delta x_{ij}&={L_{ij}^{\mathrm{NDCG}}}^*-L^{\mathrm{NDCG}}_{ij}=\left({\ell^\mathrm{NDCG}(i)}^*- \ell^\mathrm{NDCG}(i) \right) p(j|i), \\
    &= \left( \frac{G_{max}/|\mathcal{P}| - G(i)}{G_{max}} - \frac{G_{max}/|\mathcal{P}|-1}{G_{max}} \right) \frac{H(x_{ij})}{N_{FP}(i)} ,  \\
    &= \frac{1- G(i)}{G_{max}} \frac{H(x_{ij})}{N_{FP}(i)},
\end{align}
and one can compute the gradients using Eq. 5 in the paper (line 4 of Algorithm 1).
\section{Computing aLRP Loss and its Gradients}
This section presents the algorithm to compute aLRP Loss in detail along with an analysis of space and time complexity. For better understanding, bold font denotes multi-dimensional data structures (which can be implemented by vectors, matrices or tensors). Algorithm \ref{alg:aLRPLoss} describes the steps to compute aLRP Loss along with the gradients for a given mini-batch. 

\textbf{Description of the inputs}: $\mathbf{S}$ is the raw output of the classification branch, namely logits. For localisation, as done by IoU-based localisation losses \cite{UnitBox,GIoULoss}, the raw localisation outputs need to be converted to the boxes, which are denoted by $\mathbf{B}$. We assume that $\mathbf{M}$ stores $-1$ for ignored anchors and $0$ for negative anchors. For positive anchors, $\mathbf{M}$ stores the index of the ground truth (i.e. $\{1,...,|\mathbf{\hat{B}}|\}$, where $\mathbf{\hat{B}}$ is a list of ground boxes for the mini-batch). Hence, we can find the corresponding ground truth for a positive anchor only by using $\mathbf{M}$. $\delta$ is the smoothness of the piecewise linear function defined in Eq. \ref{eq:PiecewiseLinear} and set to $1$ following AP Loss. We use the self-balance ratio, $\frac{\mathcal{L}^\mathrm{aLRP}}{\mathcal{L}^\mathrm{aLRP}_{cls}}$, by averaging over its values from the previous epoch. We initialize it as $50$ (i.e. see Table 4 in the paper).

\textbf{Part 1: Initializing Variables}: Lines 2-10  aim to initialize the necessary data from the inputs. While this part is obvious, please note that line 8 determines a threshold to select the relevant negative outputs. This is simply due to Eq. \ref{eq:PiecewiseLinear} and the gradients of these negative examples with scores under this threshold are zero. Therefore, for the sake of time and space efficiency, they are ignored.

\textbf{Part 2: Computing Unnormalized Localisation Errors}: Lines 12-14 compute unnormalized localisation error on each positive example. Line 12 simply finds the localisation error of each positive example and line 13 sorts these errors with respect to their scores in descending order, and Line 14 computes the cumulative sum of the sorted errors with $\mathrm{cumsum}$ function. In such a way, the example with the larger scores contributes to the error computed for each positive anchor with smaller scores. Note that while computing the nominator of the $\mathcal{L}^\mathrm{aLRP}_{loc}$, we employ the step function (not the piecewise linear function), since we can safely use backpropagation. 

\textbf{Part 3: Computing Gradient and Error Contribution from Each Positive}: Lines 16-32 compute the gradient and error contribution from each positive example. To do so, Line 16 initializes necessary data structures. Among these data structures, while  $\mathbf{\mathcal{L}^\mathrm{LRP}_{loc}}$, $\mathbf{\mathcal{L}^\mathrm{LRP}_{cls}}$ and $\frac{\partial \mathcal{L}^\mathrm{aLRP}}{ \partial \mathbf{S_+}}$ are all with size $|\mathcal{P}|$, $\frac{\partial \mathcal{L}^\mathrm{aLRP}}{ \partial \mathbf{S_-}}$ has size $|\hat{\mathcal{N}}|$, where $\hat{\mathcal{N}}$ is the number of negative examples after ignoring the ones with scores less than $\tau$ in Line 8, and obviously $|\hat{\mathcal{N}}| \leq |\mathcal{N}|$. The loop iterates over each positive example by computing LRP values and gradients since aLRP is defined as the average LRP values over positives (see Eq. 9 in the paper). Lines 18-22 computes the relation between the corresponding positive with positives and relevant negatives, each of which requires the difference transformation followed by piecewise linear function:
\begin{align}
    \label{eq:PiecewiseLinear}
    H(x) = \begin{cases} 
      0,  &  x < -\delta \\
      \frac{x}{2 \delta}+0.5, &  -\delta \leq x \leq \delta \\
      1, &  \delta < x.
      \end{cases}
\end{align}
Then, using these relations, lines 23-25 compute the rank of the $i$th examples within positive examples, number of negative examples with larger scores (i.e. false positives) and rank of the example. Lines 26 and 27 compute aLRP classification and localisation errors on the corresponding positive example. Note that to have a consistent denominator for total aLRP, we use $\mathrm{rank}$ to normalize both of the components. Lines 28-30 compute the gradients. While the local error is enough to determine the unnormalized gradient of a positive example, the gradient of a negative example is accumulated through the loop.

\textbf{Part 4: Computing aLRP Loss and Gradients}: Lines 34-40 simply derive the final aLRP value by averaging over LRP values (lines 34-36), normalize the gradients (lines 37-38) and compute gradients wrt the boxes (line 39) and applies self balancing (line 40).

\subsection{Time Complexity} 
\begin{itemize}
    \item First 16 lines of Algorithm \ref{alg:aLRPLoss} require time between $\mathcal{O}(|\mathcal{P}|)$ and $\mathcal{O}(|\mathcal{N}|)$. Since for the object detection problem, the number of negative examples is quite larger than number of positive anchors (i.e. $|\mathcal{P}| << |\mathcal{N}|$), we can conclude that the time complexity of first 13 lines is $\mathcal{O}(|\mathcal{N}|)$.
    \item  The bottleneck of the algorithm is the loop on lines 17-32. The loop iterates over each positive example, and in each iteration while lines 21, 24 and 30 are executed for relevant negative examples, the rest of the lines is executed for positive examples. Hence the number of operations for each iteration is $\max(|\mathcal{P}|, |\hat{\mathcal{N}}|)$ (i.e. number of relevant negatives, see lines 8-9), and overall these lines require $\mathcal{O}(|\mathcal{P}| \times \max(|\mathcal{P}|, |\hat{\mathcal{N}}|))$.  Note that, while in the early training epochs, $|\hat{\mathcal{N}}| \approx |\mathcal{N}|$, as the training proceeds, the classifier tends to distinguish positive examples from negative examples very well, and $|\hat{\mathcal{N}}|$ significantly decreases implying faster mini-batch iterations. 
    \item The remaining lines between 26-33 again require time between $\mathcal{O}(|\mathcal{P}|)$ and $\mathcal{O}(|\mathcal{N}|)$.
\end{itemize}
 Hence, we conclude that the time complexity of Algorithm \ref{alg:aLRPLoss} is $\mathcal{O}(|\mathcal{N}|+|\mathcal{P}| \times \max(|\mathcal{P}|, |\hat{\mathcal{N}}|))$.

Compared to AP Loss; 
\begin{itemize}
    \item aLRP Loss includes an extra computation of aLRP localisation component (i.e. lines 12-14, 27. Each of these lines requires $\mathcal{O}(|\mathcal{P}|)$).
    \item aLRP Loss includes an additional summation while computing the gradients with respect to the scores of the positive examples in line 29 requiring $\mathcal{O}(|\mathcal{P}|^2)$.
    \item aLRP Loss discards interpolation (i.e. using interpolated AP curve), which can take up to $\mathcal{O}(|\mathcal{P}| \times |\hat{\mathcal{N}}|)$.
\end{itemize}

\begin{algorithm}
\caption{The algorithm to compute aLRP Loss for a mini-batch. \label{alg:aLRPLoss}}
\begin{flushleft}
\hspace*{\algorithmicindent} \textbf{Input:} $\mathbf{S}$: Logit predictions of the classifier for each anchor, \\ 
\hspace*{15mm} $\mathbf{B}$: Box predictions of the localization branch from each anchor,\\
\hspace*{15mm} $\mathbf{\hat{B}}$: Ground truth (GT) boxes, \\
\hspace*{15mm} $\mathbf{M}$: Matching of the anchors with the GT boxes. \\
\hspace*{15mm} $\delta$: Smoothness of the piece-wise linear function ($\delta=1$ by default). \\
\hspace*{15mm} $w_{ASB}$: ASB weight, computed using $\frac{\mathcal{L}^\mathrm{aLRP}}{\mathcal{L}^\mathrm{aLRP}_{cls}}$ values from previous epoch. \\
 \hspace*{\algorithmicindent} \textbf{Output:} $\mathcal{L}^\mathrm{aLRP}$: aLRP loss, $\frac{\partial \mathcal{L}^\mathrm{aLRP}}{ \partial \mathbf{S}}$: Gradients wrt logits, $\frac{\partial \mathcal{L}^\mathrm{aLRP}}{ \partial \mathbf{B}}$: Gradients wrt boxes.
\end{flushleft}
\begin{algorithmic}[1]
\State // Part 1: Initializing Variables
\State $\mathbf{idx}_+ :=$ The indices of $\mathbf{M}$ where $\mathbf{M} > 0$.
\State $\mathbf{M}_+ :=$ The values of $\mathbf{M}$ where $\mathbf{M} > 0$.
\State $\mathbf{B_+}:=$ The values of $\mathbf{B}$ at indices $\mathbf{idx}_+$.
\State $\mathbf{S_+}:=$ The values of $\mathbf{S}$ at indices $\mathbf{idx}_+$.
\State $\mathbf{idx^{sorted}_+} :=$ The indices of $\mathbf{S_+}$ once it is sorted in descending order.
\State $\mathbf{S^{sorted}_+} :=$ The values of $\mathbf{S_+}$ when ordered according to $\mathbf{idx^{sorted}_+}$.
\State $\tau=\min (\mathbf{S_+})-\delta$.
\State $\mathbf{idx}_- :=$ The indices of $\mathbf{M}$ where $\mathbf{M} = 0$ and $s_j \geq \tau$ (i.e. relevant negatives only).
\State $\mathbf{S_-}:=$ The values of $\mathbf{S}$ at indices $\mathbf{idx}_-$.
\State // Part 2: Computing Unnormalized Localisation Errors
\State $\mathbf{\mathcal{E}_{Loc}} = \frac{1-\mathrm{IoU}(\mathbf{B_+}, \mathbf{\hat{B}_+})}{1-\tau}$. (or $\mathbf{\mathcal{E}_{Loc}} = \frac{(1-\mathrm{GIoU}(\mathbf{B_+}, \mathbf{\hat{B}_+}))/2}{1-\tau}$ for $\mathrm{GIoU}$ Loss \cite{GIoULoss}.)
\State $\mathbf{\mathcal{E}_{Loc}^{sorted}}:=$ The values of $\mathbf{\mathcal{E}_{Loc}}$ when ordered according to  $\mathbf{idx^{sorted}_+}$.
\State $\mathbf{\mathcal{E}_{Loc}^{cumsum}} =\mathrm{cumsum}(\mathbf{\mathcal{E}_{Loc}^{sorted}})$
\State // Part 3: Computing Gradient and Error Contribution from Each Positive
\State Initialize  , $\mathbf{\mathcal{L}^\mathrm{LRP}_{loc}}$, $\mathbf{\mathcal{L}^\mathrm{LRP}_{cls}}$, $\frac{\partial \mathcal{L}^\mathrm{aLRP}}{ \partial \mathbf{S_+}}$ and $\frac{\partial \mathcal{L}^\mathrm{aLRP}}{ \partial \mathbf{S_-}}$. 
\ForEach {$s_i \in \mathbf{S^{sorted}_+}$}
\State $\mathbf{X_+}:=$ Difference transform of $s_i$ with the logit of each positive example.
\State $\mathbf{R_+}:=$ The relation of $i \in \mathcal{P}$ with each $j \in \mathcal{P}$ using Eq. \ref{eq:PiecewiseLinear} with input $\mathbf{X_+}$.
\State $\mathbf{R_+}[i] = 0$
\State $\mathbf{X_-}:=$ Difference transform of $s_i$ with the logit of each negative example.
\State $\mathbf{R_-}:=$  The relation of $i \in \mathcal{P}$ with each $j \in \mathcal{N}$ using Eq. \ref{eq:PiecewiseLinear} with input $\mathbf{X_+}$.
\State $\mathrm{rank}_+=1+\mathrm{sum}(\mathbf{R_+})$
\State $\mathrm{FP}=\mathrm{sum}(\mathbf{R_-})$
\State $\mathrm{rank}=\mathrm{rank}_+ +\mathrm{FP}$
\State $\mathbf{\mathcal{L}^\mathrm{LRP}_{cls}}[i] =  \mathrm{FP}/\mathrm{rank}$
\State $\mathbf{\mathcal{L}^\mathrm{LRP}_{loc}}[i] = \mathbf{\mathcal{E}_{Loc}^{cumsum}}[i]/\mathrm{rank}$
\If{$\mathrm{FP} \geq \epsilon$} //For stability set $\epsilon$ to a small value (e.g. $1e-5$)
\State $\frac{\partial \mathcal{L}^\mathrm{aLRP}}{ \partial \mathbf{S_+}}[i] = - \left( \mathrm{FP} + \sum \limits_{i \in P} \mathbf{R_+}[i] \times \mathcal{E}_{Loc}^{cumsum}[i] \right) /\mathrm{rank}$
\State $\frac{\partial \mathcal{L}^\mathrm{aLRP}}{ \partial \mathbf{S_-}} +=\left(- \mathbf{ \frac{\partial \mathcal{L}^\mathrm{aLRP}}{ \partial S_+}}[i] \times \frac{\mathbf{R_-}}{\mathrm{FP}} \right)$
\EndIf
\EndFor
\State // Part 4: Computing the aLRP Loss and Gradients
\State $\mathcal{L}^\mathrm{aLRP}_{cls} = \mathrm{mean}(\mathbf{\mathcal{L}^\mathrm{LRP}_{cls}})$
\State $\mathcal{L}^\mathrm{aLRP}_{loc} = \mathrm{mean}(\mathbf{\mathcal{L}^\mathrm{LRP}_{loc}})$
\State $\mathcal{L}^\mathrm{aLRP} =\mathcal{L}^\mathrm{aLRP}_{cls} + \mathcal{L}^\mathrm{aLRP}_{loc}$
\State Place $\frac{\partial \mathcal{L}^\mathrm{aLRP}}{ \partial \mathbf{S_+}}$ and $\frac{\partial \mathcal{L}^\mathrm{aLRP}}{ \partial \mathbf{S_-}}$ into $\mathbf{\frac{\partial \mathcal{L}^\mathrm{aLRP}}{ \partial S}}$ also by setting the gradients of remaining examples to $0$.
\State $\mathbf{\frac{\partial \mathcal{L}^\mathrm{aLRP}}{ \partial S}}/=|\mathcal{P}|$
\State Compute $\frac{\partial \mathcal{L}^\mathrm{aLRP}}{ \partial \mathbf{B}}$ (possibly using autograd property of a deep learning library or refer to the supp. mat. of \cite{GIoULoss} for the gradients of GIoU and IoU Losses.
\State $\frac{\partial \mathcal{L}^\mathrm{aLRP}_{loc}}{ \partial \mathbf{B}} \times =w_{ASB}$
\State \textbf{return} $\frac{\partial \mathcal{L}^\mathrm{aLRP}}{ \partial \mathbf{S}}$, $\frac{\partial \mathcal{L}^\mathrm{aLRP}}{ \mathbf{\partial \mathbf{B}}}$ and  $\mathcal{L}^\mathrm{aLRP}$.
\end{algorithmic}
\end{algorithm}

\subsection{Space Complexity} 
Algorithm \ref{alg:aLRPLoss} does not require any data structure larger than network outputs (i.e. $\mathbf{B}$, $\mathbf{S}$). Then, we can safely conclude that the space complexity is similar to all of the common loss functions that is $\mathcal{O}(|\mathbf{S}|)$.

\section{Details of aLRP Loss}
\label{sec:aLRP}
This section provides details for aLRP Loss.

\subsection{A Soft Sampling Perspective for aLRP Localisation Component}
In sampling methods, the contribution ($w_i$) of the $i$th bounding box to the loss function is adjusted as follows:
\begin{equation}
    \label{eq:SamplingEq}
    \mathcal{L} = \sum \limits_{i \in \mathcal{P} \cup \mathcal{N}} w_i \mathcal{L}(i),
\end{equation}
where $\mathcal{L}(i)$ is the loss of the $i$th example. Hard and soft sampling approaches differ on the possible values of $w_i$. For the hard sampling approaches, $w_i \in \{0,1\}$, thus a BB is either selected or discarded. For soft sampling approaches, $w_i \in [0,1]$, i.e. the contribution of a sample is adjusted with a weight and each BB is somehow included in training. While this perspective is quite common to train the classification branch \cite{PrimeSample,FocalLoss}; the localisation branch is conventionally trained by hard sampling with some exceptions (e.g. CARL \cite{PrimeSample} sets $w_i = s_i$ where $s_i$ is the classification score).

Here, we show that, in fact, what aLRP localisation component does is soft sampling. To see this, first let us recall the definition of the localisation component: 
\begin{align}
    \label{eq:aLRPRegression}
     \mathcal{L}^{\mathrm{aLRP}}_{loc} &= \frac{1}{|\mathcal{P}|}\sum \limits_{i \in \mathcal{P}} \frac{1}{\mathrm{rank}(i) }\left( \mathcal{E}_{loc}(i)+\sum \limits_{k \in \mathcal{P}, k \neq i}  \mathcal{E}_{loc}(k) H(x_{ik})  \right),
\end{align}
which is differentiable with respect to the box parameters as discussed in the paper. With a ranking-based formulation, note that (i) the localisation error of a positive example $i$ (i.e. $\mathcal{E}_{loc}(i)$) contributes each LRP value computed on a positive example $j$ where $s_i \geq s_j$ (also see Fig. 2 in the paper), and (ii) each LRP value computed on a positive example $i$ is normalized by $\mathrm{rank}(i)$. Then, setting $\mathcal{L}(i) = \mathcal{E}_{loc}(i)$ in Eq. \ref{eq:SamplingEq} and accordingly taking Eq. \ref{eq:aLRPRegression} in $\mathcal{E}_{loc}(i)$ paranthesis, the weights of the positive examples (i.e. $w_i =0$ for negatives for the localisation component) are:
\begin{align}
    \label{eq:aLRPRegressionSoft}
    w_i = \frac{1}{|\mathcal{P}|} \left( \left( {\sum \limits_{k \in \mathcal{P}, k \neq i}}  \frac{H(x_{ki})}{\mathrm{rank}(k)}\right) + \frac{1}{\mathrm{rank}(i)} \right).
\end{align}
Note that $\mathcal{L}(i)$ is based on a differentiable IoU-based regression loss and $w_i$ is its weight, which is a scaler. As a result $H(x_{ki})$ in Eq. \ref{eq:aLRPRegressionSoft} does not need to be smoothed and we use a unit-step function (see line 14 in Algorithm \ref{alg:aLRPLoss}).

\comment{
\subsection{Convexity of aLRP Loss}
This section discusses the convexity of aLRP Loss, and concludes that it is not convex.

\begin{lemma}
\label{lemma:IoU}
IoU Loss (i.e. $1-\mathrm{IoU}(\cdot)$) is not convex.
\end{lemma}

\begin{proof}
This is proof by counterexample. If $\mathrm{IoU}(\cdot)$ were convex, then for any two sets of boxes, $A \in \mathcal{R}^8$ and $B \in \mathcal{R}^8$, Jensen`s inequality for $\alpha \in [0,1]$ for concave functions would hold for $IoU()$:
\begin{align}
    \mathrm{IoU}(\alpha A + (1-\alpha)B) \geq \alpha \mathrm{IoU}(A) + (1-\alpha) \mathrm{IoU}(B)
\end{align}
However, simply setting $A=[10,10,20,20,10,10,20,20]$ and $B = [5,5,10,10,10,10,20,20]$ (note that $\mathrm{IoU}(A)=1$, $\mathrm{IoU}(B)=0$) yields the following values for RHS and LHS of Jensen`s inequality at $\alpha=0.5$:
\begin{align}
    \mathrm{IoU}(\alpha A + (1-\alpha)B) &= 21/131.25 \\
    \alpha \mathrm{IoU}(A) + (1-\alpha) \mathrm{IoU}(B) &= 0.50
\end{align}{}
Hence IoU is not concave, and IoU Loss is not convex. 
\end{proof}{}

\begin{theorem}
\label{theorem:convex}
$\mathrm{aLRP}$ is not convex.
\end{theorem}

\begin{proof}
This is proof by contradiction. Assume that $\mathrm{aLRP}$ is convex. Following Theorem \ref{theorem:PrimaryTerms}, we can express aLRP as follows:
\begin{align}
    \mathcal{L}^{aLRP} = \frac{1}{Z} \sum \limits_{i \in \mathcal{P} } \sum \limits_{j \in \mathcal{N} }  L^{aLRP}_{ij}.
\end{align}{}
Since summation and division by a constant are convexity-preserving operations, if $\mathrm{aLRP}$ is convex, then $L^{aLRP}_{ij}$ is convex for any $i$ and $j$, that is,
\begin{align}
    \label{eq:aLRPGeneralPrimaryTermDefinition}
    L^\mathrm{aLRP}_{ij} =  \frac{1}{\mathrm{rank(i)}}
    \left(N_{FP}(i) + \mathcal{E}_{loc}(i) +  \sum \limits_{k \in \mathcal{P}, k \neq i}  \mathcal{E}_{loc}(k) H(x_{ik}) \right)  \frac{H(x_{ij})}{N_{FP}(i)}.
\end{align}
Without loss of generalization, for a single $i$ and $j$, note that $N_{FP}(i)$ and $\mathrm{rank}(i)$ are constants. Hence, if  $L^{aLRP}_{ij}$, then  $\sum \limits_{k \in \mathcal{P}, k \neq i}  \mathcal{E}_{loc}(k) H(x_{ik})$ is convex, and similarly $\mathcal{E}_{loc}(k)$ and IoU Loss is convex. However due to Lemma \ref{lemma:IoU}, this is a contradiction and hence we conclude that aLRP loss is not convex.
\end{proof}{}
}
\subsection{The Relation between aLRP Loss Value and Total Gradient Magnitudes}
\label{subsec:SB}
Here, we identify the relation between the loss value and the total magnitudes of the gradients following the generalized framework due to the fact that it is a basis for our self-balancing strategy introduced in Section 4.2 as follows:
\begin{align}
    \sum_{i \in \mathcal{P}} \abs{\frac{\partial \mathcal{L}}{\partial s_i}} = \sum_{i \in \mathcal{N}} \abs{\frac{\partial \mathcal{L}}{\partial s_i}} \approx \mathcal{L}^{\mathrm{aLRP}}.
\end{align}

Since we showed in Section \ref{sec:Proofs} that $\sum_{i \in \mathcal{P}} \abs{\frac{\partial \mathcal{L}}{\partial s_i}} = \sum_{i \in \mathcal{N}} \abs{\frac{\partial \mathcal{L}}{\partial s_i}}$, here we show that the loss value is approximated by the total magnitude of gradients. Recall from Eq. \eqref{eq:GradientsToUpdates} that total gradients of the positives can be expressed as:
\begin{align}
    \sum_{i \in \mathcal{P}} \abs{\frac{\partial \mathcal{L}}{\partial s_i}} =
    \abs{\frac{1}{|\mathcal{P}|} \sum \limits_{i \in \mathcal{P}} \sum \limits_{j \in \mathcal{N}} \Delta x_{ij}}.
\end{align}
Since $\Delta x_{ij} \leq 0$, we can discard the absolute value by multiplying it by $-1$:
\begin{align}
     - \frac{1}{|\mathcal{P}|} \sum \limits_{i \in \mathcal{P}} \sum \limits_{j \in \mathcal{N}} \Delta x_{ij}.
\end{align}
Replacing the definition of the $\Delta x_{ij}$ by $L_{ij}^*-L_{ij}$ yields:
\begin{align}
     &- \frac{1}{|\mathcal{P}|} \sum \limits_{i \in \mathcal{P}} \sum \limits_{j \in \mathcal{N}} (L_{ij}^*-L_{ij})
     = - \frac{1}{|\mathcal{P}|} \left( \sum \limits_{i \in \mathcal{P}} \sum \limits_{j \in \mathcal{N}} L_{ij}^* - \sum \limits_{i \in \mathcal{P}} \sum \limits_{j \in \mathcal{N}} L_{ij} \right) \\
     &= \frac{1}{|\mathcal{P}|}\sum \limits_{i \in \mathcal{P}} \sum \limits_{j \in \mathcal{N}} L_{ij} - \frac{1}{|\mathcal{P}|} \sum \limits_{i \in \mathcal{P}} \sum \limits_{j \in \mathcal{N}} L_{ij}^* .
\end{align}
Using Theorem \ref{theorem:PrimaryTerms}, the first part (i.e $\frac{1}{|\mathcal{P}|} \sum \limits_{i \in \mathcal{P}} \sum \limits_{j \in \mathcal{N}} L_{ij}$) yields the loss value, $\mathcal{L}$.  Hence:
\begin{align}
  \sum_{i \in \mathcal{P}} \abs{\frac{\partial \mathcal{L}}{\partial s_i}} = \mathcal{L} - \frac{1}{|\mathcal{P}|} \sum \limits_{i \in \mathcal{P}} \sum \limits_{j \in \mathcal{N}} L_{ij}^* .
\end{align}
Reorganizing the terms, the difference between the total gradients of positives (or negatives, since they are equal -- see Theorem \ref{theorem:BalancedTraining}) and the loss values itself is the sum of the targets normalized by number of positives:
\begin{align}
  \mathcal{L}- \sum_{i \in \mathcal{P}} \abs{\frac{\partial \mathcal{L}}{\partial s_i}} = \frac{1}{|\mathcal{P}|} \sum \limits_{i \in \mathcal{P}} \sum \limits_{j \in \mathcal{N}} L_{ij}^* .
\end{align}
Compared to the primary terms, the targets are very small values (if not $0$). For example, for AP Loss ${L_{ij}^\mathrm{AP}}^*=0$, and hence, loss is  equal to the sum of the gradients: $\mathcal{L} = \sum_{i \in \mathcal{P}} \abs{\frac{\partial \mathcal{L}}{\partial s_i}}$.

As for aLRP Loss, the target of a primary term is $\frac{ \mathcal{E}_{loc}(i)}{\mathrm{rank}(i)} \frac{H(x_{ij})}{N_{FP}(i)}$, hence if $H(x_{ij})=0$, then the target is also $0$. Else if $H(x_{ij})=1$, then it implies that there are some negative examples with larger scores, and $\mathrm{rank}(i)$ and $N_{FP}(i)$ are getting larger depending on these number of negative examples, which causes the denominator to grow, and hence  yielding a small target as well. Then ignoring this term, we conclude that:
\begin{align}
     \sum_{i \in \mathcal{P}} \abs{\frac{\partial \mathcal{L}}{\partial s_i}} =  \sum_{i \in \mathcal{N}} \abs{\frac{\partial \mathcal{L}}{\partial s_i}} \approx \mathcal{L}^{\mathrm{aLRP}} .
\end{align}

\subsection{Self-balancing the Gradients Instead of the Loss Value}
Instead of localisation the loss, $\mathcal{L}^{\mathrm{aLRP}}_{loc}$, we multiply ${\partial \mathcal{L}}/{\partial B}$ by the average $\mathcal{L}^{\mathrm{aLRP}}/\mathcal{L}^{\mathrm{aLRP}}_{loc}$ of the previous epoch. This is because formulating aLRP Loss as $\mathcal{L}^{\mathrm{aLRP}}_{loc}+ w_r \mathcal{L}^{\mathrm{aLRP}}_{loc}$ where $w_r$ is a weight to balance the tasks is different from weighing the gradients with respect to the localisation output, $B$, since weighting the loss value (i.e. $\mathcal{L}^{\mathrm{aLRP}}_{loc}+ w_r \mathcal{L}^{\mathrm{aLRP}}_{loc}$) changes the gradients of aLRP Loss with respect to the classification output as well since $\mathcal{L}^{\mathrm{aLRP}}_{loc}$, now weighed by $w_r$, is also ranking-based (has $\mathrm{rank}(i)$ term - see Eq. 11 in the paper). Therefore, we directly add the self balance term as a multiplier of ${\partial \mathcal{L}}/{\partial B}$ and backpropagate accordingly. On the other hand, from a practical perspective, this can simply be implemented by weighing the loss value, $\mathcal{L}^{\mathrm{aLRP}}_{loc}$ without modifying the gradient formulation for $\mathcal{L}^{\mathrm{aLRP}}_{cls}$.

\section{Additional Experiments}
This section presents more ablation experiments, the anchor configuration we use in our models and the effect of using a wrong target for the primary term in the error-driven update rule.

\subsection{More Ablation Experiments: Using Self Balance and GIoU with AP Loss}
We also test the effect of GIoU and our Self-balance approach on AP Loss, and present the results in Table \ref{tab:minival2}:
\begin{itemize}
    \item Using IoU-based losses with AP Loss improves the performance up to 1.0 AP as well and reaches 36.5 AP with GIoU loss.
    \item Our SB approach also improves AP Loss between 0.7 - 1.2 AP, resulting in 37.2AP as the best performing model without using $w_r$. However, it may not be inferred that SB performs better than constant weighting for AP Loss without a more thorough tuning of AP Loss since SB is devised to balance the gradients of localisation and classification outputs for aLRP Loss (see Section \ref{subsec:SB}).
    \item Comparing with the best performing model of AP Loss with 37.2AP, (i) aLRP Loss has a 1.7AP and 1.3oLRP points better performance, (ii) the gap is 4.0AP for $\mathrm{AP}_{90}$, and (iii) the correlation coeffient of aLRP Loss, preserves the same gap (0.48 vs 0.44 comparing the best models for AP and aLRP Losses), since applying these improvements (IoU-based losses and SB) to AP Loss does not have an effect on unifying branches.
\end{itemize}

\begin{table}[]
    \centering
    \footnotesize
    \caption{Using Self Balance and GIoU with AP Loss. For optimal LRP (oLRP), lower is better.}
    \label{tab:minival2}
    \begin{tabular}{|c|c|c|c|c|c|c|c||c|} \hline
       $\mathcal{L}_c$&$\mathcal{L}_r$&SB&$\mathrm{AP}$&$\mathrm{AP_{50}}$&$\mathrm{AP_{75}}$&$\mathrm{AP_{90}}$&$\mathrm{oLRP}$&$\rho$\\ \hline \hline
        \multirow{6}{*}{AP Loss \cite{APLoss}}&Smooth L1& & $35.5$&$58.0$&$37.0$&$9.0$&$71.0$ &$0.45$
        \\
        &Smooth L1&\checkmark&$36.7$&$58.2$&$39.0$&$10.8$&$70.2$&$0.44$\\
        &IoU Loss& &$36.3$&$57.9$&$37.9$&$11.8$&$70.4$ &$0.44$\\        
        &IoU Loss&\checkmark&$37.2$&$58.1$&$39.2$&$13.1$&$69.6$ &$0.44$\\
        &GIoU Loss& & $36.5$&$58.1$&$38.1$&$11.9$&$70.2$&$0.45$\\ 
        &GIoU Loss&\checkmark&$37.2$&$58.3$&$39.0$&$13.4$&$69.7$&$0.44$\\  
        \hline 
        \multirow{3}{*}{aLRP Loss}&with IoU& &$36.9$&$57.7$&$38.4$&$13.9$&$69.9$ &$0.49$
         \\ 
         &with IoU&\checkmark&$38.7$&$58.1$&$40.6$&$17.4$&$68.5$&$0.48$
         \\ 
         &with GIoU&\checkmark&$38.9$&$58.5$&$40.5$&$17.4$&$68.4$&$0.48$
         \\ 

        \hline
    \end{tabular}
\end{table}

\subsection{Anchor Configuration}
The number of anchors has a notable affect on the efficiency of training due to the time and space complexity of optimizing ranking-based loss functions by combining error-driven update and backpropagation. For this reason, different from original RetinaNet using three aspect ratios (i.e. $[0.5, 1, 2]$) and three scales (i.e. $[2^{0/2}, 2^{1/2}, 2^{2/2}]$) on each location, Chen et al. \cite{APLoss} preferred the same three aspect ratios, but reduced the scales to two as $[2^{0/2}, 2^{1/2}]$ to increase the efficiency of AP Loss. In our ablation experiments, except the one that we used ATSS \cite{ATSS}, we also followed the same anchor configuration of Chen et al. \cite{APLoss}.

One main contribution of ATSS is to simplify the anchor design by reducing the number of required anchors to a single scale and aspect ratio (i.e. ATSS uses 1/9 and 1/6 of the anchors of RetinaNet \cite{FocalLoss} and AP Loss \cite{APLoss} respectively), which is a perfect fit for our optimization strategy. For this reason, we used ATSS, however, we observed that the configuration in the original ATSS with a single aspect ratio and scale does not yield the best result for aLRP Loss, which may be related to the ranking nature of aLRP Loss which favors more examples to impose a more accurate ranking, loss and gradient computation. Therefore, different from ATSS configuration, we find it useful to set anchor scales $[2^{0/2}, 2^{1/2}]$ and $[2^{0/2}, 2^{1/2}, 2^{2/2}]$ for aLRPLoss500 and aLRPLoss800 respectively and use a single aspect ratio with $1$ following the original design of ATSS. 

\subsection{Using a Wrong Target for the Primary Term in the Error-driven Update Rule}
\begin{figure*}[t!]
    \begin{subfigure}[t]{0.5\textwidth}
        \centering
        \includegraphics[width=1\textwidth]{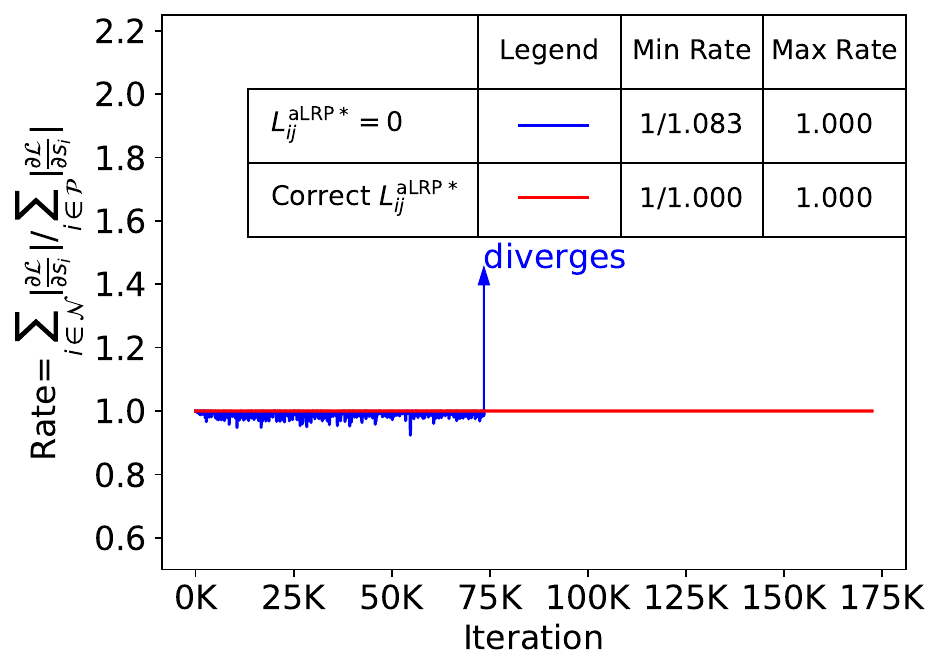}
    \end{subfigure}
    ~
    \begin{subfigure}[t]{0.47\textwidth}
        \centering
        \includegraphics[width=1\textwidth]{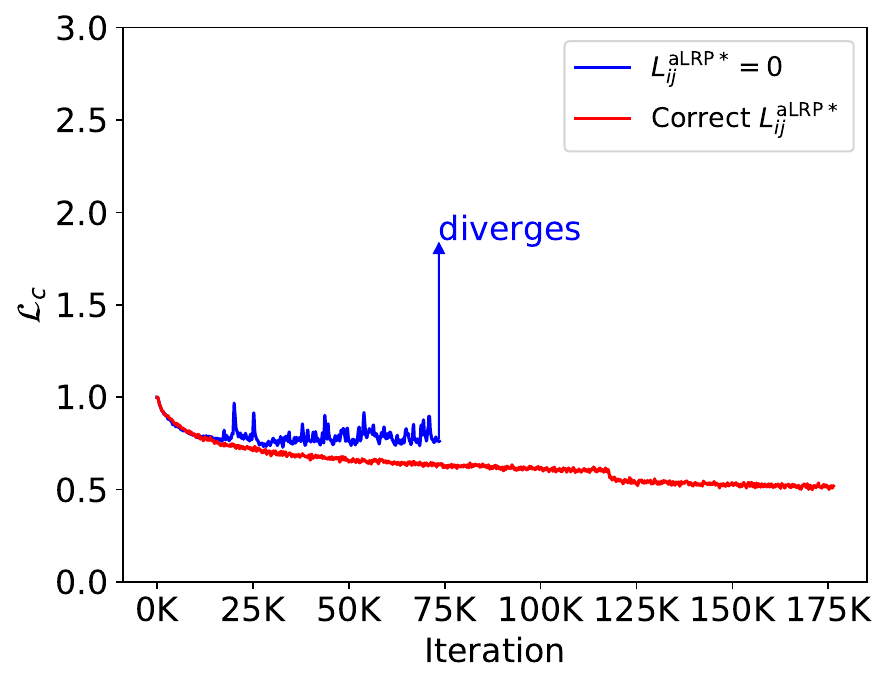}
    \end{subfigure}    
    \caption{\textbf{(left)} The rate of the total gradient magnitudes of negatives to positives. \textbf{(right)} Loss values. \label{fig:GradComp2}}
\end{figure*}

As discussed in our paper (Section 4.1, Equation 13), ${L_{ij}}^*$, the target value of the primary term ${L_{ij}}$ is non-zero due to the localisation error. It is easy to overlook this fact and assume that the target is zero. Fig. \ref{fig:GradComp2} presents this case where ${L_{ij}}^*$ is set to 0 (i.e. minimum value of aLRP). In such a case, the training continues properly, similar to that of the correct case, up to a point and then diverges. Note that this occurs when the positives start to be ranked properly but are still assigned gradients since ${L_{ij}}^*-{L_{ij}} \neq 0$ due to the nonzero localisation error. This causes $\sum \limits_{i \in \mathcal{P}} \abs{\frac{\partial \mathcal{L}}{\partial s_i}} > \sum \limits_{i \in \mathcal{N}} \abs{\frac{\partial \mathcal{L}}{\partial s_i}}$, violating Theorem \ref{theorem:BalancedTraining} (compare min-rate and max-rate in Fig. \ref{fig:GradComp2}). Therefore, assigning proper targets as indicated in Section 3 in the paper is crucial for balanced training. 

\subsection{Implementation Details for FoveaBox and Faster R-CNN}
In this section, we provide more implementation details on the FoveaBox and Faster R-CNN models that we trained with different loss functions. All the models in this section are tested on COCO \textit{minival}.

\textbf{Implementation Details of FoveaBox:} We train the models for 100 epochs with a learning rate decay at epochs 75 and 95. For aLRP Loss and AP Loss, we preserve the same learning rates used for RetinaNet (i.e. $0.008$ and $0.002$ for aLRP Loss and AP Loss respectively). As for the Focal Loss, we set the initial learning rate to $0.02$ following the linear scheduling hypothesis \cite{MegDet} (i.e. Kong et al. set learning rate to $0.01$ and use a batch size of $16$). Following AP Loss official implementation, the gradients of the regression loss (i.e. Smooth L1) are averaged over the output parameters of positive boxes for AP Loss. As for Focal Loss, we follow the mmdetection implementation which averages the total regression loss by the number of positive examples. The models are tested on COCO \textit{minival} by preserving the standard design by mmdetection framework.

\textbf{Implementation Details of Faster R-CNN:} To train Faster R-CNN, we first replace the softmax classifier of  Fast R-CNN by the class-wise sigmoid classifiers. Instead of heuristic sampling rules, we use all anchors to train RPN and top-1000 scoring proposals per image obtained from RPN to train Fast R-CNN (i.e. same with the default training except for discarding sampling). Note that, with aLRP Loss, the loss function consists of two independent losses instead of four in the original pipeline, hence instead of three scalar weights, aLRP Loss requires a single weight for RPN head, which we tuned as $0.20$. Following the positive-negative assignment rule of RPN, different from all the experiments, which use $\tau=0.50$, $\tau=0.70$ for aLRP Loss of RPN. We set the initial learning rate to $0.04$ following the linear scheduling hypothesis \cite{MegDet} for the baselines, and decreased by a factor of $0.10$ at epochs 75 and 95. Localisation loss weight is kept as 1 for L1 Loss and to 10 for GIoU Loss \cite{mmdetection,GIoULoss}. The models are tested on COCO \textit{minival} by preserving the standard design by mmdetection framework. We do not train Faster R-CNN with AP Loss due to the difficulty to tune Faster R-CNN for a different loss function.




\end{document}